\documentclass[runningheads,orivec]{llncs}
\pdfoutput=1
\usepackage[T1]{fontenc}

\usepackage{multirow}
\usepackage{graphicx}

\usepackage{cite}

\usepackage{acro}

\usepackage{subcaption}

\usepackage{amsthm}
\usepackage{amsmath}
\usepackage{amssymb}
\usepackage{amsfonts}
\usepackage{tipa}

\usepackage{multirow}

\usepackage{graphicx}

\usepackage{mathtools}
\usepackage[inline]{enumitem}
\usepackage{hyperref}
\usepackage{cleveref}
\usepackage{wrapfig}

\usepackage{tabularx}

\usepackage{tablefootnote}

\usepackage{tikz}
\usepackage{pgfplots}
\pgfplotsset{compat=1.11}
\usepgfplotslibrary{fillbetween}
\usetikzlibrary{intersections}
\pgfdeclarelayer{bg}
\pgfsetlayers{bg,main}

\usepackage{algorithm}
\usepackage[noEnd=true,indLines=true,italicComments=false,spaceRequire=false]{algpseudocodex}

\usepackage{placeins}

\newcommand{\reluSym}[0]{\ensuremath{\mathrm{ReLU}}}
\newcommand{\relu}[1]{\ensuremath{\reluSym\left(#1\right)}}

\newcommand{\zonoLower}[1]{\ensuremath{\underline{#1}}}
\newcommand{\zonoUpper}[1]{\ensuremath{\overline{#1}}}

\newcommand{\affineForm}{\text{\textctyogh}}%
\newcommand{\zonotope}{\ensuremath{\mathcal{Z}}}

\newcommand{\errorSoftmaxSigmoid}[2]{\ensuremath{\mathrm{err}_{\sigma}\left(#1,#2\right)}}
\newcommand{\errorSoftmaxPolytope}[2]{\ensuremath{\mathrm{err}_{\text{poly}}\left(#1,#2\right)}}

\DeclareAcronym{OLNNV}{
    short = open-loop NNV,
    long = open-loop neural network verification,
    first-style=short
}
\DeclareAcronym{CLNNV}{
    short = closed-loop NNV,
    long = closed-loop neural network verification,
    first-style=short
}
\DeclareAcronym{NN}{
    short = NN,
    long = neural network
}
\DeclareAcronym{FNN}{
    short = FNN,
    long = feed forward neural network
}
\DeclareAcronym{NNCS}{
    short = NNCS,
    long = neural network based control system
}
\DeclareAcronym{DNF}{
    short = DNF,
    long = disjunctive normal form
}
\DeclareAcronym{ACASX}{
    short = ACAS X,
    long = Airborne Collision Avoidance System X,
    cite={olson2015airborne}
}
\DeclareAcronym{ACASXu}{
    short = ACAS Xu,
    long = Airborne Collision Avoidance System X unmanned
}
\newcommand{\keymaeraxtext}{Ke{\kern-0.75ptY}maera X}
\DeclareAcronym{keymaerax}{
    short = \keymaeraxtext{},
    long = \keymaeraxtext{},
    first-style = long,
    tag = noindexplease
}
\DeclareAcronym{modelplex}{
    short = ModelPlex,
    long = ModelPlex,
    first-style = long,
    tag = noindexplease
}
\DeclareAcronym{DNNV}{
    short = DNNV,
    long = DNNV,
    cite={Shriver2021},
    first-style = long,
    tag = noindexplease
}
\DeclareAcronym{OVERT}{
    short = OVERT,
    long = OVERT,
    first-style = long,
    tag = noindexplease
}
\DeclareAcronym{dL}{
    short=$d\mathcal{L}$,
    long=differential dynamic logic
}
\DeclareAcronym{NMAC}{
    short=NMAC,
    long=Near Mid-Air Collision
}
\DeclareAcronym{CPS}{
    short=CPS,
    long=Cyber-Physical System
}
\DeclareAcronym{SNNT}{
    short=\textsc{N$^3$V},
    long=\textbf{N}on-linear \textbf{N}eural \textbf{N}etwork \textbf{V}erifier,
    first-style=short
}
\DeclareAcronym{SMT}{
    short=SMT,
    long=Satisfiability Modulo Theories
}
\DeclareAcronym{RL}{
    short=RL,
    long=reinforcement learning
}
\DeclareAcronym{MILP}{
    short=MILP,
    long=Mixed Integer Linear Programming
}
\DeclareAcronym{ACC}{
    short=ACC,
    long=Adaptive Cruise Control
}
\DeclareAcronym{TCAS}{
    short=TCAS,
    long=Traffic Alert and Collision Avoidance System
}
\DeclareAcronym{FAA}{
    short=FAA,
    long=Federal Aviation Administration
}
\newcommand{\bad}[1]{{\color{red} #1}}
\newcommand{\good}[1]{{\color{green!30!black} #1}}

\usepackage[createShortEnv,conf={restate,text proof translated={Proof of }}]{proof-at-the-end}
\newEndThm[normal]{definitionE}{definition}

\begin{document}
\title{%
Revisiting Differential Verification:\\
Equivalence Verification with Confidence
}
\titlerunning{Revisiting Differential Verification: Equivalence Verification with Confidence}

\author{Samuel Teuber\orcidID{0000-0001-7945-9110} \and
Philipp Kern\orcidID{0000-0002-7618-7401} \and\\
Marvin Janzen \and
Bernhard Beckert\orcidID{0000-0002-9672-3291}}
\institute{Karlsruhe Institute of Technology, Karlsruhe, Germany}

\maketitle              %
\begin{abstract}
\looseness=-1
When validated \acp{NN} are pruned (and retrained) before deployment, it is desirable to prove that the new \ac{NN} behaves \emph{equivalently} to the (original) reference \ac{NN}.
To this end, our paper revisits the idea of \emph{differential verification} which performs reasoning on differences between \acp{NN}:
On the one hand, our paper proposes a novel abstract domain for differential verification admitting more efficient reasoning about equivalence.
On the other hand, we investigate empirically and theoretically which equivalence properties are (not) efficiently solved using differential reasoning.
Based on the gained insights, and following a recent line of work on confidence-based verification, we propose a novel equivalence property that is amenable to Differential Verification while providing guarantees for \emph{large parts of the input space} instead of small-scale guarantees constructed w.r.t. predetermined input points.
We implement our approach in a new tool called \emph{VeryDiff} and perform an extensive evaluation on numerous old and new benchmark families, including new pruned \acp{NN} for particle jet classification in the context of CERN's LHC where we observe median speedups $>300\times$ over the State-of-the-Art verifier $\alpha,\beta$-CROWN.

\end{abstract}

\section{Introduction}
\looseness=-1
Specifying what an \ac{NN} is supposed to do is a difficult problem, that is at most partially solved.
One class of specifications that is comparatively easy to formalize are equivalence properties:
Given an ``old'' reference \ac{NN} $f_1$, we aim to prove that a ``new'' \ac{NN} $f_2$ behaves in some way equivalently.
For example $\varepsilon$ equivalence~\cite{kleine_buning_verifying_2020,paulsen_reludiff_2020} requires that the numerical outputs of $f_1$ and $f_2$ for the same input point differ by at most $\varepsilon$ or Top-1 equivalence~\cite{kleine_buning_verifying_2020} requires that the two \acp{NN}' classifications match.
Known applications of equivalence verification are verification after retraining or pruning~\cite{WANG2024127347}, student-teacher training~\cite{kleine_buning_verifying_2020,Teuber2021a}, analysis of sensitivity to \ac{NN}-based preprocessing steps~\cite{narodytska_verifying_2018} and construction of quantized \acp{NN}~\cite{CEG4N}.
Several publications~\cite{paulsen_reludiff_2020,paulsen_neurodiff_2020,kleine_buning_verifying_2020,Teuber2021a,Eleftheriadis2022,WANG2024127347,DBLP:conf/nfm/PP24} have proposed methods for the verification of equivalence properties (sometimes
calling it ``approximate conformance'').
While it is known that equivalence verification w.r.t.\ the $\varepsilon$ equivalence (\Cref{def:epsilon_equivalence}) property is coNP-complete~\cite{Teuber2021a}, the complexity-theoretic status of Top-1 equivalence verification (\Cref{def:top1_equivalence}) was to date unclear.

\paragraph{Contribution.}
\looseness=-1
This work encompasses multiple theoretical and practical contributions to the field of equivalence verification:
\begin{itemize}
    \item[(C1)] We prove that deciding if two $\reluSym{}$ \acp{NN} are Top-1 equivalent is a coNP-complete decision problem, i.e. it is as hard as $\varepsilon$-equivalence verification~\cite{Teuber2021a} or the classic \ac{NN} verification problem~\cite{katz_reluplex_2017,DBLP:conf/rp/SalzerL21}.
    \item[(C2)] We propose \emph{Differential Zonotopes}: An abstract domain that allows the usage of the differential verification methodology w.r.t. the Zonotope abstract domain by propagating a Zonotope bounding the \emph{difference} between two \acp{NN} in lock-step with a reachability analysis for the individual \acp{NN}.
    \item[(C3)] We implement the proposed approach in a new tool and evaluate its efficiency.
    For $\varepsilon$ equivalence we achieve median speedups >10 for 8 of 9 comparisons (4.5 in the other case).
    \item[(C4)]
    For Top-1 equivalence we demonstrate empirically that Differential Zonotopes do not aid verification.
    We provide a theoretical intuition for this observation and demonstrate this is a fundamental limitation of Differential Verification in general -- independently of the chosen abstract domain. %
    \item[(C5)] Based on these insights, we propose a new confidence-based equivalence property for classification \ac{NN} which is
    \begin{enumerate*}
        \item verifiable on larger parts of the input space of \acp{NN};
        \item amenable to differential verification.
    \end{enumerate*}
    Furthermore, we propose a simpler \emph{and} more precise linear approximation of the $\mathrm{softmax}$ function in comparison to prior work~\cite{DBLP:conf/cav/AthavaleBCMNW24}.
    In additional experiments, we demonstrate that our tool can certify 327\% more benchmark queries than $\alpha,\beta$-CROWN for confidence-based equivalence.
\end{itemize}

\paragraph{Related Work.}
\looseness=-1
Prior work on Zonotope-based \ac{NN} verification~\cite{GehrZonotope,Singh18} verified non-relational properties w.r.t. a single \ac{NN}.
We extend this work by providing a methodology that allows reasoning about \emph{differences} between \acp{NN}.
Equivalence properties, can in principle, be analyzed using classical \ac{NN} verification techniques such as $\alpha,\beta$-CROWN~\cite{zhang22babattack,xu2020automatic,zhang2018efficient,zhang2022general,shi2024genbab}
for a single \ac{NN} by building ``product-networks'' (similar to product-programs in classical program verification~\cite{DBLP:conf/fm/BartheCK11}), but early work on \ac{NN} equivalence verification demonstrated that this approach is inefficient due to the accumulation of overapproximation errors in the two independent \acp{NN}~\cite{paulsen_reludiff_2020}.
While this view was recently challenged by \cite{DBLP:conf/nfm/PP24}, \Cref{sec:evaluation} conclusively demonstrates that tailored verification tools still outperform State-of-the-Art ``classical'' \ac{NN} verification tools for \acp{NN} with similar weight structures.

Prior work suggested using Star-Sets for equivalence verification without analyzing weight differences and heavily relied on LP solving~\cite{Teuber2021a}.
Prior work on differential verification~\cite{paulsen_reludiff_2020,paulsen_neurodiff_2020} did not verify the equivalence of classifications and also fell short of using the Zonotope abstract domain.
\Cref{sec:evaluation} compares to equivalence verifiers~\cite{paulsen_neurodiff_2020,kleine_buning_verifying_2020,Teuber2021a}.
In another line of work, QEBVerif~\cite{DBLP:conf/cav/ZhangSS23} proposes a sound and complete analysis technique tailored to quantized \acp{NN} which is not directly applicable to other kinds of \acp{NN} studied in our evaluation.

Another line of research analyzes relational properties w.r.t. multiple runs of a \emph{single} \ac{NN}.
All listed works do not verify equivalence.
For example, Banerjee \emph{et al.}~\cite{DBLP:journals/pacmpl/BanerjeeXS24} propose an abstract domain for relational properties, but assume that all executions happen on the same \ac{NN}.
This makes their approach incompatible with our benchmarks which require the analysis of multiple \emph{different} \ac{NN}.
Another incomplete relational verifier~\cite{banerjee2024relational} also assumes executions on a single \ac{NN} and requires tailored relaxations not available for equivalence properties.
Encoding relational properties via product \acp{NN} has also been explored by Athavale \emph{et al.}~\cite{DBLP:conf/cav/AthavaleBCMNW24}.
We compare against Marabou (which they used) and we prove that our approximation of softmax, though simpler, is always more precise.

\paragraph{Overview.}
\Cref{sec:background} introduces the necessary background on \ac{NN} verification via Zonotopes, equivalence verification and confidence based \ac{NN} verification.
\Cref{sec:complexity} proves the coNP-completeness of Top-1 equivalence.
Subsequently, we introduce \emph{Differential Zonotopes} as an abstract domain for differential reasoning via Zonotopes (\Cref{sec:diff_zono}) and explain how Differential Zonotopes can be used to perform equivalence verification (\Cref{sec:top1}).
\Cref{sec:classification} explains why Top-1 equivalence does not benefit from differential reasoning in general and derives a new confidence-based equivalence property that may hold on large parts of the input space and can be verified more efficiently using differential verification.
Finally, \Cref{sec:evaluation} provides an evaluation of our approach.

\section{Background}
\label{sec:background}
\looseness=-1
We deal with the verification of piece-wise linear, feed-forward \acfp{NN}.
A \ac{NN} with input dimension $I \in \mathbb{N}$ and output dimension $O \in \mathbb{N}$ and $L\in\mathbb{N}$ layers can be summarized as a function $f:\mathbb{R}^I \to \mathbb{R}^O$ which maps input vectors $\mathbf{x}^{(0)} \in \mathbb{R}^I$ to output vectors $\mathbf{x}^{(L)} \in \mathbb{R}^O$.
In more detail, each layer of the \ac{NN} consists of an affine transformation $\tilde{\mathbf{x}}^{(i)} = W^{(i)} \mathbf{x}^{(i-1)} + b^{(i)}$ (for a matrix $W^{(i)}$ and a vector $b^{(i)}$) followed by the application of a non-linear function $\mathbf{x}^{(i)} = h^{(i)}\left(\tilde{\mathbf{x}}^{(i)}\right)$.
Many feed-forward architectures can be compiled into this format~\cite{shriver_dnnv_2021}.
We focus on the case of \acp{NN} with \reluSym{} activations, i.e. $h^{(i)}\left(\tilde{\mathbf{x}}^{(i)}\right) = \relu{\tilde{\mathbf{x}}^{(i)}}=\max\left(0,\tilde{\mathbf{x}}^{(i)}\right)$ for all 
$1 \leq i \leq L$.
$f^{(i)}\left(\mathbf{x}\right)$ is the computation of the \ac{NN}'s first $i$ layers.
We uniformly denote vectors in bold ($\mathbf{v}$) and matrices in capital letters ($M$) and Affine Forms/Zonotopes are denoted as $\affineForm$/$\zonotope$.

\paragraph{\ac{NN} Verification.}
A well-known primitive in the literature on \ac{NN} verification are \emph{Zonotopes}~\cite{GehrZonotope,Singh18}: An abstract domain that allows the efficient propagation of an interval box over the input space through (piece-wise) affine systems:
\begin{definition}[Zonotope]
A \emph{Zonotope} with input dimension $n$ and output dimension $m$ is 
a collection of $m$ \emph{Affine Forms} of the structure
\begin{align*}
\mathbf{g} \mathbf{\epsilon} + c &&\left(\mathbf{\epsilon}\in\left[-1,1\right]^n, \mathbf{g}\in\mathbb{R}^n,c\in\mathbb{R}\right)
\end{align*}
We denote a single Affine Form as a tuple $\left(\mathbf{g},c\right)$.
Given an Affine Form $\affineForm{}=\left(\mathbf{g},c\right)$ and a vector $\mathbf{v}\in\left[-1,1\right]^d$ ($d \leq n$) we denote by $\affineForm{}\left(\mathbf{v}\right)$ the Affine Form (or value for $d=n$) where the first $d$ values of $\mathbf{\epsilon}$ are fixed to $\mathbf{v}$, i.e. to $\left(\mathbf{g}_{d+1:n},\mathbf{g}_{1:d}\mathbf{v}+c\right)$.
For $\affineForm{}=\left(\mathbf{g},c\right)$, we denote the set of points described by $\affineForm{}$ as $\langle \affineForm{} \rangle = \left\{ \affineForm{}\left(\mathbf{\epsilon}\right) ~\middle|~ \mathbf{\epsilon}\in\left[-1,1\right]^n\right\}$.
\end{definition}
\noindent
Via $\affineForm{}\left(\mathbf{v}\right)$ we denote the values reachable given the (input) vector $\mathbf{v}$.
To improve clarity, some transformations applied to Zonotopes will be described for the 1-dimensional case, i.e. to a single Affine Form.
Nonetheless, a major advantage of Zonotopes lies in their Matrix representation: $m$ Affine Forms are then a matrix $G\in\mathbb{R}^{n \times m}$ and a vector $\mathbf{c} \in \mathbb{R}^m$ (then denoted as $\zonotope{}=\left(G,\mathbf{c}\right)$).
A component of $\mathbf{g}$/a column of $G$ is called a \emph{generator}.
Given a Zonotope $\zonotope{}$ we denote its $i$-th Affine Form (represented by $G$'s $i$-th row and $\mathbf{c}$'s $i$-th component) as $\zonotope{}_i$.
Similar to the affine forms, we define
$
\langle \zonotope \rangle = \left\{
\mathbf{x} \in \mathbb{R}^m \middle|
\mathbf{\epsilon}\in\left[-1,1\right]^n,
\mathbf{x}_i = \zonotope_i\left(\mathbf{\epsilon}\right)
\right\}
$.
Zonotopes are a good fit for analyzing (piece-wise) linear \ac{NN} as they are closed under affine transformations \cite{bak_improved_2020}:
\begin{proposition}[Affine Zonotope Transformation]
\label{prop:background:zono_affine}
    For some Zonotope $\zonotope=\left(G,\mathbf{c}\right)$ and an affine transformation $h\left(\mathbf{x}\right)=W\mathbf{x}+\mathbf{b}$,
    the Affine Form $\hat{\zonotope}=\left(WG,W\mathbf{c}+\mathbf{b}\right)$ exactly describes the affine transformation applied to the points $x \in \langle \zonotope \rangle$, i.e.
    for all $d \leq n$ and $\mathbf{v} \in \left[-1,1\right]^d$:
    $
    \left\{ W\mathbf{x}+\mathbf{b} ~\middle|~ \mathbf{x} \in \langle \zonotope\left(\mathbf{v}\right) \rangle\right\} =
    \langle \hat{\zonotope}\left(\mathbf{v}\right) \rangle
    $
\end{proposition}
\noindent
This proposition implies $\left\{ W\mathbf{x}+\mathbf{b} ~\middle|~ \mathbf{x} \in \langle \zonotope{}\rangle\right\} = \langle \hat{\zonotope{}} \rangle$, but it is even stronger
as it also guarantees a linear map from an input ($\mathbf{v}$) to all reachable outputs ($\langle \hat{\zonotope{}}\left(\mathbf{v}\right)\rangle$).
Zonotopes also admit efficient computation of interval bounds for their outputs:
\begin{proposition}[Zonotope Output Bounds]
\label{prop:background:zono_bounds}
    Consider some Affine Form $\affineForm{}=\left(\mathbf{g},c\right)$
    it holds for all $\mathbf{x} \in \left[-1,1\right]^n$ that:
    \[
    \mathbf{g}\mathbf{x}+c \in
    \left[\zonoLower{\affineForm{}},\zonoUpper{\affineForm{}}\right] \coloneqq
    \left[ c-\sum_{i=0}^n \left|\mathbf{g}_{i}\right|,c+\sum_{i=0}^n \left|\mathbf{g}_{i}\right|  \right]
    \]
\end{proposition}%
\noindent

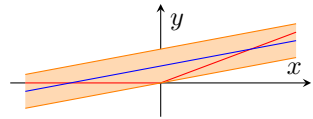
\begin{wrapfigure}[7]{r}{0.35\textwidth}
\centering
\vspace*{-1.2em}
\begin{tikzpicture}
  \begin{axis}[
    axis lines = middle,
    xlabel = {$x$},
    ylabel = {$y$},
    y=3cm,
    x=2cm,
    xmin=-1.0, xmax=1.0,
    ymin=-0.15, ymax=0.35,
    xtick distance=5.0,
    ytick distance=5.0]
  \addplot [name path = relu1,
    domain = -0.9:0,
    samples = 20,red] {0};
  \addplot [name path = relu2,
    domain = 0:0.9,
    samples = 20,red] {0.25*x};
  \addplot [name path = relu_interpolation,
    domain = -0.9:0.9,
    samples = 20, blue] {0.25*1.2/(1.2+1.2)*x+0.25*0.5*(1.2*1.2/(1.2+1.2))};
  \addplot [name path = zonoUp,
    domain = -0.9:0.9,
    samples = 20, orange] {0.25*(1.2/(1.2+1.2)*x+(1.2*1.2/(1.2+1.2)))};
  \addplot [name path = zonoDown,
    domain = -0.9:0.9,
    samples = 20, orange] {0.25*(1.2/(1.2+1.2)*x)};
  \addplot [orange!30] fill between [of = zonoDown and zonoUp, soft clip={domain=-0.9:0.9}];
  \end{axis}
\end{tikzpicture}
\caption{$\reluSym$ approximation}
\label{fig:zono_relu_approx}
\end{wrapfigure}%
\noindent
Zonotopes cannot exactly represent the application of $\reluSym$, but we can approximate the effect by distinguishing three cases:
\begin{enumerate*}
    \item The upper-bound $\zonoUpper{\affineForm{}}$ is negative and thus $\relu{x}=0$ for $x \in \langle \affineForm{}\rangle$;
    \item The lower-bound $\zonoLower{\affineForm{}}$ is positive and thus $\relu{x}=x$ for $x \in \langle \affineForm{}\rangle$;
    \item The $\reluSym{}$-node is \emph{instable} and its output is thus piece-wise linear.
\end{enumerate*}
The first and second case can be represented as an affine transformation and the third case requires an approximation (see \Cref{fig:zono_relu_approx} for intuition):
\emph{Interpolation} between $\relu{x}=0$ and $\relu{x}=x$ (see the blue line in \Cref{fig:zono_relu_approx} and $\lambda \mathbf{g}$ in \Cref{prop:background:zono_relu}) yields a representation, that we can turn into a sound overapproximation (meaning all possible output values of $\relu{x}$ are contained).
This is achieved by adding a new generator that appropriately bounds the error of the interpolated function (see orange lines in \Cref{fig:zono_relu_approx} and $\tfrac{1}{2}\lambda\zonoLower{\affineForm{}}$ in \Cref{prop:background:zono_relu}).
This result is summarized as follows: %
\begin{proposition}[\reluSym{} Zonotope Transformation~\cite{Singh18}]
\label{prop:background:zono_relu}
Consider some Affine Form $\affineForm{}=\left(\mathbf{g},c\right)$.
Define a new Affine Form $\hat{\affineForm{}}=\left(\hat{\mathbf{g}},\hat{c}\right)$ such that:
\begin{align*}
    \hat{\mathbf{g}}&=0 \in \mathbb{R}^{n}&\hat{c}&=0 && \text{if }\zonoUpper{z}<0\\
    \hat{\mathbf{g}}&=\mathbf{g} \in \mathbb{R}^{n}&\hat{c}&=c && \text{if }\zonoLower{z}>0\\
    \hat{\mathbf{g}}&=\left(
    \begin{array}{c}
        \lambda \mathbf{g}\\
        \tfrac{1}{2}\lambda\zonoLower{\affineForm{}}
    \end{array}
    \right) \in \mathbb{R}^{n+1} &\hat{c}&=c-\tfrac{1}{2}\lambda\zonoLower{\affineForm{}}&&\text{else}
\end{align*}
for $\lambda = \frac{\zonoUpper{\affineForm{}}}{\left(\zonoUpper{\affineForm{}}-\zonoLower{\affineForm{}}\right)}$.
Then $\hat{\affineForm{}}$ guarantees for all $d \leq n$ and $\mathbf{v} \in \left[-1,1\right]^d$ that:
\[
\left\{\reluSym\left(x\right) \middle| x \in \langle \affineForm{}\left(\mathbf{v}\right) \rangle\right\} \subseteq \langle \hat{\affineForm{}}\left(\mathbf{v}\right)\rangle
\]
\end{proposition}

\noindent
\ac{NN} verification via Zonotopes typically proceeds as follows:
An input set described as Zonotope is propagated through the \ac{NN} using the transformers from \Cref{prop:background:zono_affine,prop:background:zono_relu}.
This yields an overapproximation of the \ac{NN}'s behavior.
Depending on the verification property, one can either check the property by computing the Zonotope's bounds (\Cref{prop:background:zono_bounds}) or by solving a linear program.
If a property cannot be established, the problem is refined by either splitting the input space, w.r.t. its dimensions (input-splitting, e.g.~\cite{Singh18}) or w.r.t. a particular neuron to eliminate the $\reluSym{}$'s nonlinearity (neuron-splitting, e.g.~\cite{bak_improved_2020}).

\paragraph{Equivalence Verification.}
To show that two \acp{NN} $f_1,f_2$ behave equivalently, we can, for example, verify that the \acp{NN}' outputs are equal up to some $\varepsilon$:
\begin{definition}[$\varepsilon$ Equivalence~\cite{kleine_buning_verifying_2020,paulsen_reludiff_2020}]
\label{def:epsilon_equivalence}
    Given two \acp{NN} $f_1,f_2:\mathbb{R}^I \to \mathbb{R}^O$ and an input set $X \subseteq \mathbb{R}^I$ we say $g_1$ and $g_2$ are \emph{$\varepsilon$ equivalent} w.r.t. a $p$-norm iff for all $\mathbf{x} \in X$ it holds that $\left\lVert f_1\left(\mathbf{x}\right) - f_2\left(\mathbf{x}\right)\right\rVert_p < \varepsilon$
\end{definition}
\noindent
Deciding $\varepsilon$ equivalence is coNP-complete~\cite{Teuber2021a}.
Another line of work proposes verification of Top-1 equivalence which is important for classification \acp{NN}~\cite{kleine_buning_verifying_2020}:
\begin{definition}[Top-1 Equivalence~\cite{kleine_buning_verifying_2020}]
\label{def:top1_equivalence}
    Given two \acp{NN} $f_1,f_2:\mathbb{R}^I \to \mathbb{R}^O$ and an input set $X \subseteq \mathbb{R}^I$, $f_1$ and $f_2$ are \emph{Top-1 equivalent} iff for all $\mathbf{x} \in X$ we have $\mathrm{argmax}_i~f_{1, i}\left(\mathbf{x}\right) = \mathrm{argmax}_i~f_{2, i}\left(\mathbf{x}\right)$, i.e. for every $k\in\left[1,O\right],\mathbf{x}\in X$ it holds   $f_{1,k}\left(\mathbf{x}\right)\geq f_{1,j}\left(\mathbf{x}\right)$ (for all $j\neq k$) implies that $f_{2,k}\left(\mathbf{x}\right)\geq f_{2,j}\left(\mathbf{x}\right)$ (for all $j\neq k$).
\end{definition}
\noindent
For $\varepsilon$ equivalence, prior work by Paulsen \emph{et al.}~\cite{paulsen_reludiff_2020,paulsen_neurodiff_2020} introduced \emph{differential verification}: For two \acp{NN} of equal depth (i.e. $L_1=L_2$),
equivalence can be verified more effectively by reasoning about weight differences.
To this end, Paulsen \emph{et al.} used symbolic intervals~\cite{wang_efficient_2018,wang_formal_2018} not only for bounding values of a single \ac{NN}, but to bound the \emph{difference} of values between two \acp{NN} $f_1,f_2$, i.e. at any layer $1 \leq i \leq L$ we compute two linear symbolic bound functions $l^{(i)}_\Delta\left(\mathbf{x}\right),u^{(i)}_\Delta\left(\mathbf{x}\right)$ such that $l^{(i)}_\Delta\left(\mathbf{x}\right) \leq f_1^{(i)}\left(\mathbf{x}\right) - f_2^{(i)}\left(\mathbf{x}\right) \leq u^{(i)}_\Delta\left(\mathbf{x}\right)$.
Differential bounds are computed by propagating symbolic intervals through two \acp{NN} in lock-step.
This enables the computation of bounds on the difference at every layer.

\paragraph{Confidence Based Verification}
Many classification \acp{NN} provide a confidence for their classification by using the Softmax function $\mathrm{softmax}_i\left(x\right) = \frac{e^{x_i}}{\sum_{i=1}^n x_i}$.
A recent line of work~\cite{DBLP:conf/cav/AthavaleBCMNW24} proposes to use the confidence values classically provided by classification \acp{NN} as a starting point for verification.
For example, Athavale et al.~\cite{DBLP:conf/cav/AthavaleBCMNW24} propose a global, confidence-based robustness property which stipulates that inputs classified with high confidence must be robust to noise (i.e. we require the same classification for some bounded perturbation of the input).
While this introduces reliance on the \ac{NN}'s confidence, it enables \emph{global} specifications verified on the full input space and this limitation can be addressed by an orthogonal direction of research that aims at training calibrated \acp{NN}~\cite{DBLP:conf/icml/GuoPSW17,DBLP:conf/uai/AoRS23}, i.e. \acp{NN} that correctly estimate the confidence of their predictions.

\section{Complexity of Top-1 Verification}
\label{sec:complexity}
\looseness=-1
Prior work showed that the classic \ac{NN} verification problem 
for a single \reluSym{}-\ac{NN} is an NP-complete problem~\cite{katz_reluplex_2017,DBLP:conf/rp/SalzerL21}.
\begin{textAtEnd}
\subsection{Proofs on NP-completeness}
Prior work demonstrated that finding counterexamples for the specification of a $\reluSym{}$ \ac{NN} is an NP-complete decision problem~\cite{katz_reluplex_2017}:
\begin{definition}[\textsc{NetVerify}]
The problem \textsc{NetVerify} is concerned with the following task:
Given linear constraints $\psi_1\left(\mathbf{x}\right) \equiv C_1 \mathbf{x} \leq \mathbf{b_1}$, $\psi_2\left(\mathbf{y}\right) \equiv C_2 \mathbf{y} \leq \mathbf{b_2}$ and a \reluSym{}-\ac{NN} $f$, find $\mathbf{\hat{x}} \in \mathbb{R}^I, \mathbf{\hat{y}}\in\mathbb{R}^O$ such that $\psi_1\left(\mathbf{\hat{x}}\right),\psi_2\left(\mathbf{\hat{y}}\right)$ and furthermore $\mathbf{\hat{y}} = f\left(\mathbf{\hat{x}}\right)$.
\end{definition}
\noindent
This problem is NP-complete. While the original proof contained a subtle flaw, this was later fixed~\cite{DBLP:conf/rp/SalzerL21}.
The proof of NP-completeness for finding $\varepsilon$ equivalence violation reduces the \textsc{NetVerify} problem to the problem of finding an $\varepsilon$-violation.
While the proof on membership in NP for $\varepsilon$ equivalence~\cite{Teuber2021a} suffers from the same problem as the proof by Katz et al.~\cite{katz_reluplex_2017}, it can be repaired in the manner suggested by Sälzer and Lange~\cite{DBLP:conf/rp/SalzerL21}: Instead of guessing an input violating $\varepsilon$-equivalence directly, we guess a configuration of \reluSym{}-phases and solve the corresponding LP problem.
Unfortunately, Top-1 equivalence is subtly different: While a counterexample for $\varepsilon$ equivalence has a difference $\geq \varepsilon$, a counterexample for Top-1 equivalence requires another node which is \emph{strictly} larger.
Thus, the same reduction is not readily applicable to the case of Top-1 equivalence.
To prove the NP-completeness of finding violations for Top-1 equivalence, we thus begin by showing the NP-completeness of the following, modified, \ac{NN}-verification problem which considers \emph{strict inequalities}:
\begin{definition}[\textsc{StrictNetVerify}]
\label{def:strict_net_verify}
The problem \textsc{StrictNetVerify} is concerned with the following task:
Given linear constraints $\psi_1\left(\mathbf{x}\right) \equiv C_1 \mathbf{x} < \mathbf{b_1}$, $\psi_2\left(\mathbf{y}\right) \equiv C_2 \mathbf{y} < \mathbf{b_2}$ and a \reluSym{}-\ac{NN} $f$, find $\mathbf{\hat{x}} \in \mathbb{R}^I, \mathbf{\hat{y}}\in\mathbb{R}^O$ such that $\psi_1\left(\mathbf{\hat{x}}\right)$ and $\psi_2\left(\mathbf{\hat{y}}\right)$ are satisfied and furthermore $\mathbf{\hat{y}} = g\left(\mathbf{\hat{x}}\right)$.
\end{definition}
\begin{lemma}[\textsc{StrictNetVerify} is NP-complete]
\label{lem:strict_net_verify}
The problem \textsc{StrictNetVerify} is NP-complete.
\end{lemma}
\begin{proof}
For the proof of membership in NP we follow the NP-completeness proof for \textsc{NetVerify}~\cite{DBLP:conf/rp/SalzerL21}:
We guess a phase configuration for all \reluSym{} nodes of our problem and then encode the constraints as a linear program.
We add an additional variable $\delta\geq0$ to our linear program and add $+\delta$ to all strict inequalities.
We can then maximize $\delta$ w.r.t. the linear program which is a problem solvable in polynomial time~\cite{DBLP:conf/stoc/Karmarkar84}.
The problem is feasible iff the maximal value of $\delta$ is larger than $0$.
For the hardness result, we follow the proof strategy by Katz \emph{et al.}~\cite{katz_reluplex_2017}.
However, instead of using the fix from Sälzer and Lange~\cite{DBLP:conf/rp/SalzerL21}, we propose an alternative fix for the \textsc{Bool}-gadget.
The proof works by reducing the problem \textsc{3Sat} to (\textsc{Strict})\textsc{NetVerify}.
Consider a \textsc{3Sat} formula $\phi \equiv \phi_1 \land \dots \land \phi_n$ over $k$ variables where each $\phi_i$ is a disjunction over 3 literals $q_{i1} \lor q_{i2} \lor q_{i3}$.
The problem of determining whether there exists an assignment of the $k$ variables such that $\phi$ is satisfied is a well-known NP-complete problem~\cite{DBLP:conf/stoc/Cook71}.
Katz \emph{et al.}~\cite{katz_reluplex_2017} propose three \emph{gadgets} that allow us to encode negation, disjunction of three literals, and conjunction as \reluSym{}-\acp{NN}:
\begin{align*}
    \mathrm{Neg}\left(q_1\right) &= 1 - q_1\\
    \mathrm{Or}\left(q_1,q_2,q_3\right) &= 1-\relu{1-q_1-q_2-q_3}\\
    \mathrm{And}\left(q_1,\dots,q_n\right) &= \sum_{i=1}^n q_i
\end{align*}
Clearly, if all $q_i$ are in $\left\{0,1\right\}$ then $\mathrm{Neg}$ returns the negation (i.e. $0$ if originally $1$ and $1$ otherwise), $\mathrm{Or}$ returns the disjunction over the three $q_i$ (i.e. it is $1$ if at least one of the $q_i$ is one) and $\mathrm{And}$ returns the conjunction over all $q_i$: It is $n$ iff all $q_i$ are $1$.
Thus, using these gadgets and assuming $k$ inputs from $\left\{0,1\right\}$ it is possible to construct an \ac{NN} of which the output is $n$ iff the given \textsc{3Sat} formula $\phi$ is satisfied.
However, the question now is how to extend this to a convex input set described by strict linear inequalities and how to phrase the output constraint (currently $=n$, i.e. the opposite of a strict inequality).
We choose some $\epsilon < \frac{1}{n+3} < \frac{1}{2}$ (where $n$ is the number of clauses in the \textsc{3Sat} problem).
Here, we diverge from the proof of Katz \emph{et al.}~\cite{katz_reluplex_2017} and Sälzer and Lange~\cite{DBLP:conf/rp/SalzerL21}.
Instead, we fix the input constraint as follows: $\bigwedge_{i=1}^k 0-\epsilon < x_i < 1+\epsilon$.
We then use a new gadget to clamp the values of the inputs to the interval $\left[0,1\right]$ again.
To this end, each input individually is passed through the following function:
\begin{equation*}
    \mathrm{Bound}\left(x\right) = \relu{x}-\relu{x-1}
\end{equation*}
We then apply the output of $\mathrm{Bound}\left(x_i\right)$ to the \reluSym{}-\ac{NN} which computes the output of $\phi$ (see above).
However, this still allows arbitrary values in the interval $\left[0,1\right]$ as input for the NN computing $\phi$.
To mitigate this, we use the following, fixed, $\mathrm{Bool^\circ}$ gadget:
\begin{equation*}
    \mathrm{Bool^\circ}\left(x\right) = \epsilon-x + \relu{2x-1}
\end{equation*}
For $x=0$ and $x=1$ this yields $\mathrm{Bool^\circ}\left(x\right)=\epsilon$.
For $x \in \left(0,\epsilon\right)$ this yields $\epsilon-x \in \left(0,\epsilon\right)$.
For $x \in \left(1-\epsilon,1\right)$ this yields $\epsilon-x+2x-1=\epsilon+x-1 \in \left(0,\epsilon\right)$.
For values in between, i.e. for $x \in \left[\epsilon,1-\epsilon\right]$ the gadget yields either $\epsilon-x \leq 0$ (for $x\geq\epsilon$ and $2x-1\leq0$) or $\epsilon-x+2x-1=\epsilon+x-1 \leq 0$ (for $2x-1> 0$ which implies $1-\epsilon \geq x>0.5$).
We can thus guarantee that all $x$ are in $\left(0,\epsilon\right) \cup \left(1-\epsilon,1\right)$ by making $0 < \mathrm{Bool}^\circ\left(x\right) < \epsilon$ an output constraint for the \ac{NN}.
Given a ``truth-value'' of this kind, the output of $\mathrm{Neg}$ is again in $\left(0,\epsilon\right) \cup \left(1-\epsilon,1\right)$ and the output of $\mathrm{Or}$ is in $\left(0,3\epsilon\right) \cup \left(1-\epsilon,1\right)$.
Thus, all conjunctions are satisfied iff $\mathrm{And}$'s output is in $\left(n\left(1-\epsilon\right),n\right)$.
On the contrary, if at least one clause is not satisfied, the output is in the range $\left(0,(n-1)+3\epsilon\right)$.
To ensure that this interval's upper bound is strictly smaller than the lower bound of ``true'' we must ensure that $n-1+3\epsilon < n-n\epsilon \leftrightarrow \epsilon < \frac{1}{n+3}$, which is already satisfied by our choice of $\epsilon$.
Thus, the output of $\mathrm{And}$ can only be larger $n\left(1-\epsilon\right)$ iff the inputs are assigned with a satisfying solution.
Let $x_1,\dots,x_k$ be the input variables and $h\left(x_1,\dots,x_k\right)$ be the \ac{NN} encoding the formula $\phi$ using the gadgets $\mathrm{Neg},\mathrm{Or},\mathrm{And}$ then we get the following encoding as \textsc{StrictNetVerify}:
\begin{align*}
    \psi_1 \equiv &\bigwedge_{i=0}^{i=k} 0-\epsilon < x_i < 1+\epsilon\\
    \psi_2 \equiv &\bigwedge_{i=0}^{i=k} 0 < \textsc{Bool}^\circ\left(\textsc{Bound}\left(x_i\right)\right) < \epsilon \land\\
    & n\left(1-\epsilon\right) < h\left(\textsc{Bound}\left(x_1\right),\dots,\textsc{Bound}\left(x_k\right)\right)<n
\end{align*}
As argued above, this encoding is satisfiable iff there exists a satisfying assignment for the underlying \textsc{3Sat} problem.
Observe, that this encoding only uses strict inequalities and is in polynomial size of the \textsc{3Sat} instance.
\end{proof}
Based on this modified version of the \ac{NN} verification problem we can now prove the NP-completeness of finding Top-1 equivalence violations:
\end{textAtEnd}
Similarly, the problem of finding a violation for $\varepsilon$ equivalence is NP-complete for \reluSym{}-\acp{NN} 
as the single-\ac{NN} verification problem can be reduced to this setting~\cite{Teuber2021a}.
We now show, that finding a violation of Top-1 equivalence (\textsc{Top-1-Net-Equiv}) is also NP-complete implying that proving absence of counterexamples is coNP-complete (see proof on \cpageref{proof:nn_program_existence}):
\begin{theoremE}[\textsc{Top-1-Net-Equiv} is coNP-complete][end,restate,text link=]
\label{thm:top_1_comp}
    Let $X\subseteq\mathbb{R}^I$ be some polytope over the input space of two $\reluSym{}$-\acp{NN} $f_1,f_2$.
    Deciding whether there exists $\mathbf{x} \in X$ and a $k \in \left[1,O\right]$ s.t. $\left(f_1\left(\mathbf{x}\right)\right)_k \geq \left(f_1\left(\mathbf{x}\right)\right)_i$ for all $i \in \left[1,O\right]$ but for some $j\in\left[1,O\right]$ it holds that $\left(f_2\left(\mathbf{x}\right)\right)_k < \left(f_2\left(\mathbf{x}\right)\right)_j$ is NP-complete.
\end{theoremE}%
{
\renewcommand\proofname{Proof Sketch}
\begin{proof}
Our proof differs from the proof for $\varepsilon$ equivalence in the reduced problem which is the ``classical'' \ac{NN} verification problem~\cite{katz_reluplex_2017,DBLP:conf/rp/SalzerL21} for $\varepsilon$ equivalence verification.
To apply a similar proof technique to Top-1 equivalence, we require an \ac{NN} verification instance with only \emph{strict} inequality constraints in the input and output.
Therefore, we first prove the NP-completeness of verifying \emph{strict} inequality constraints (\Cref{def:strict_net_verify} in \Cref{apx:proofs}) by adapting prior proofs \cite{katz_reluplex_2017,DBLP:conf/rp/SalzerL21}.
Then, we can reduce this problem to Top-1 equivalence verification.
The reduction works by constructing an instance of \textsc{Top-1-Net-Equiv} that has a Top-1 violation iff the violating input also satisfies the constraints of the original NN verification problem.
\end{proof}
}
\begin{proofEnd}
\label{proof:nn_program_existence}
The proof proceeds in a similar manner as previous work on the NP-completeness of finding $\varepsilon$ equivalence violations~\cite{Teuber2021a}, however, we updated the proof on NP-membership based on the new results by Sälzer and Lange~\cite{DBLP:conf/rp/SalzerL21}.
First, observe that the problem is in NP:
Assuming we guessed a \reluSym{} phase configuration (our witness) for all nodes of the \ac{NN}, we can encode the two \acp{NN} as a linear program~\cite{DBLP:conf/rp/SalzerL21}.
For each $k\in\left[1,O\right]$ we can then constrain the program so that $k$ is among the maxima of $f_1$ and perform $(O-1)$ optimizations on the outputs of $f_2$ where we try to find an $x$ which maximizes some output node $j\neq k$ (we do this by maximizing $\left(f_2\left(\mathbf{x}\right)\right)_j-\left(f_2\left(\mathbf{x}\right)\right)_k$ using the appropriate variables from the linear program).
If for any node we obtain a maximum larger 0 we have found a concrete violation.
Since LP optimization is polynomial~\cite{DBLP:conf/stoc/Karmarkar84} and we solve a number of optimization problems linear in the problem size (more specifically in the NN's output dimension), we get that \textsc{Top-1-Net-Equiv} is in NP.

Following up on this, we now propose a reduction from \textsc{StrictNetVerify} (see \Cref{def:strict_net_verify}) to finding Top-1 equivalence violations.
Here, it suffices to construct Top-1 equivalence instances with only 2 output neurons.
Instances for \textsc{StrictNetVerify} are given as a \reluSym{}-\ac{NN} $f$ together with a conjunction of linear constraints over the input (denoted $C_1 \mathbf{x} < \mathbf{b_1}$) and a conjunction of linear constraints over the output (denoted $C_2 \mathbf{y} < \mathbf{b_2}$).
We say the instance is satisfiable iff there exists an $\mathbf{\hat{x}}\in\mathbb{R}^I$ such that $C_1 \mathbf{\hat{x}} < \mathbf{b_1}$ and $C_2 f\left(\mathbf{\hat{x}}\right) < \mathbf{b_2}$.
We now need to construct an instance of \textsc{Top-1-Net-Equiv} based of this \textsc{StrictNetVerify} instance.
We construct a first \ac{NN} $f_1$ with the following outputs $y_{11}, y_{12}$:
\begin{align*}
    y_{1 1} &= \left(f\left(x\right)\right)_1 + 1.0
    & y_{1 2} &= \left(f\left(x\right)\right)_1 - 1.0
\end{align*}
Note, that this \reluSym{}-\ac{NN} can be constructed in polynomial time and by construction $y_{1 1}$ will always be larger than $y_{1 2}$, i.e. $y_{1 1}$ is the maximum.
As an intermediate step, we construct a \reluSym{}-\ac{NN} $\tilde{f}_2\left(x\right)$ which has the following outputs:
\begin{align*}
    C_1 \mathbf{x} - b_1\\
    C_2 f\left(\mathbf{x}\right) - b_2
\end{align*}
Note, that this $\tilde{f}_2$ can also be constructed in polynomial time and furthermore that $\tilde{f}_2\left(\mathbf{x}\right)$ has a strictly negative output \emph{in all components} iff the input $\mathbf{x}$ satisfies the given constraints of the \textsc{NetVerify} instance.
Given two values $a,b\in\mathbb{R}$ we can compute the maximum of the two using the following function which is expressible with $\reluSym{}$:
\begin{align*}
\mathrm{remax}\left(a,b\right) &= a - \relu{b-a}\\ &= \relu{a} - \relu{-a} + \relu{b-a}    
\end{align*}
\noindent
By using a logarithmic number of layers and using a pyramid of $\mathrm{remax}$ functions, we can thus compute the maximum over all outputs of $\tilde{f}_2$.
We call this maximum $d^*$.
The output of the final second \ac{NN} $f_2$ is then:
\begin{align*}
    y_{2 1} &= d^*
    & y_{2 2} &= -d^*
\end{align*}
In this case $y_{2 1} < y_{2 2}$ iff $d^* < -d^*$ which is true iff $d^* < 0$.
We find a $\mathbf{\hat{x}}$ with this property iff it satisfies all strict inequalities $C_1 \mathbf{\hat{x}} < \mathbf{b_1}$ and all strict inequalities $C_2 f\left(\mathbf{\hat{x}}\right) < \mathbf{b_2}$.
Thus, $\mathbf{\hat{x}}$ satisfies the constraints of \textsc{StrictNetVerify} iff it is a violation of Top-1 equivalence as the maximal output of $f_1$ is by construction the first node.
Note, that the constructed \textsc{Top-1-Net-Equiv} instance is polynomial in the size of the original \textsc{StrictNetVerify} instance.
Thus, this reduces the NP-complete \textsc{StrictNetVerify} problem to the problem of finding Top-1 equivalence violations.
\end{proofEnd}
\begin{textAtEnd}
\subsection{Proofs for Differential Verification Methodology}
\end{textAtEnd}

\section{Equivalence Analysis via Zonotopes}
\label{sec:diff_zono}
\looseness=-1
Given two \acp{NN} $f_1,f_2:\mathbb{R}^I \to \mathbb{R}^O$ with $L$ layers each,
we follow the basic principle of differential verification, i.e. we bound the difference $f_1^{(i)}-f_2^{(i)}$ at every layer $1 \leq i \leq L$.
For our presentation, we assume that all layers of the \ac{NN} have the same width, i.e. the same number of nodes.
In practice, this limitation can be lifted by enriching the thinner layer with zero rows~\cite{WANG2024127347}.
Our approach to differential verification is summarized in \Cref{algorithm:diff_zono:overall}:
First and foremost, we perform a classic reachability analysis via Zonotopes for the two \acp{NN} $f_1,f_2$.
To this end, we propagate a given input Zonotope $\zonotope{}_{\text{in}}$ through both \acp{NN} resulting in output Zonotopes $\zonotope{}',\zonotope{}''$.
This part of the analysis uses the well-known Zonotope transformers described in \Cref{prop:background:zono_affine,prop:background:zono_relu}.
The individual reachability analysis is complemented with the computation of the \emph{Differential Zonotope} $\zonotope{}^\Delta$ which is initialized with $0$ (meaning the NN's inputs are initially equal) and computed in lock-step to the computation of $\zonotope{}'$ and $\zonotope{}''$:
At every layer, we overapproximate the maximal deviation between the two \acp{NN}.
Using the transformers described in the remainder of this Section, we can prove that the Differential Zonotope always overapproximates the difference between the two \acp{NN} (see proof on \cpageref{proof:verydiff_verysound}):
\begin{theoremE}[Soundness][end,restate,text link=]
    \label{thm:verydiff_verysound}
    Let $f_1,f_2$ be two feed-forward \textsc{ReLU}-\acp{NN}, $\zonotope{}_{\text{in}}$ some Zonotope mapping $n$ generators to $I$ dimensions and $\zonotope{}',\zonotope{}'',\zonotope{}^\Delta$ the output of $\textsc{Reach}_\Delta\left(f_1,f_2,\zonotope{}_{\text{in}}\right)$.
    The following statements hold for all $\mathbf{v} \in \left[-1,1\right]^n$:
    \begin{enumerate}
        \item $f_1\left(\zonotope{}_{\text{in}}\left(\mathbf{v}\right) \right) \in \langle \zonotope{}'\left(\mathbf{v}\right) \rangle$ and $f_2\left(\zonotope{}_{\text{in}}\left(\mathbf{v}\right) \right) \in \langle \zonotope{}''\left(\mathbf{v}\right) \rangle$
        \item $\big(~f_1\left(\zonotope{}_{\text{in}}\left(\mathbf{v}\right)\right) - f_2\left(\zonotope{}_{\text{in}}\left(\mathbf{v}\right)\right)~\big) \in \langle \zonotope{}^\Delta\left(\mathbf{v}\right) \rangle$
    \end{enumerate}
\end{theoremE}
\begin{proofEnd}
\label{proof:verydiff_verysound}
This proof presumes \Cref{def:diff_affine_form_containment} and the results from \Cref{lemma:diff_zono:zono_affine,lemma:diff_zono:zono_relu}.
The first statement immediately follows from the soundness results on Zonotope-based \ac{NN} verification~\cite{Singh18,bak_improved_2020}.
The second statement can be shown inductively over the number of hidden layers by proving that for all $\mathbf{v} \in [-1,1]^n$
it holds that $\left(\mathbf{x},\mathbf{y}\right) \in \langle \left(\zonotope{}'\left(\mathbf{v}\right),\zonotope{}''\left(\mathbf{v}\right),\zonotope{}^\Delta\left(\mathbf{v}\right)\right) \rangle$ where resp. $\mathbf{x}$ and $\mathbf{y}$ are any value reachable after propagating $\zonotope{}_{\text{in}}(\mathbf{v})$ through the first $i$ hidden layers of resp. $f_1$ and $f_2$.
Initially (i.e. before propagation through the \acp{NN}), the values are equal.
Since we initialized the difference with $0$ and initialized $Z_1$ and $Z_2$ with the same Zonotope this yields $(\zonotope{}_{\text{in}}\left(\mathbf{v}\right),\zonotope{}_{\text{in}}\left(\mathbf{v}\right)) \in \langle \left(\zonotope{}'\left(\mathbf{v}\right),\zonotope{}''\left(\mathbf{v}\right),\zonotope{}^\Delta\left(\mathbf{v}\right)\right) \rangle$.
For induction, we can assume $\left(\mathbf{x},\mathbf{y}\right) \in \langle \left(\zonotope{}'\left(\mathbf{v}\right),\zonotope{}''\left(\mathbf{v}\right),\zonotope{}^\Delta\left(\mathbf{v}\right)\right) \rangle$ for any outputs $\mathbf{x},\mathbf{y}$ of the previous layer reachable from $\zonotope{}_{\text{in}}$.
For both the affine and the ReLU case  \Cref{lemma:diff_zono:zono_affine,lemma:diff_zono:zono_relu} (see remainder of \Cref{sec:diff_zono}) have exactly these assumptions and immediately prove that the newly generated $\hat{\zonotope{}}^\Delta$ satisfies our inductive hypothesis.
Thus, after propagating the Zonotopes through all layers, our induction yields that all values $\left(\mathbf{x},\mathbf{y}\right)$ reachable from $\zonotope{}_{\text{in}}\left(\mathbf{v}\right)$ via $f_1$ and $f_2$ are represented $\langle \left(\zonotope{}'\left(\mathbf{v}\right),\zonotope{}''\left(\mathbf{v}\right),\zonotope{}^\Delta\left(\mathbf{v}\right)\right) \rangle$.
By construction, this implies that that $\big(~f_1\left(\zonotope{}_{\text{in}}\left(\mathbf{v}\right)\right) - f_2\left(\zonotope{}_{\text{in}}\left(\mathbf{v}\right)\right)~\big) \in \langle \zonotope{}^\Delta\left(\mathbf{v}\right) \rangle$
\end{proofEnd}

\begin{corollaryE}[Bound on output difference][all end]
\label{cor:verydiff_tight_bounds}
Let $f_1,f_2$ be two feed-forward \textsc{ReLU}-\acp{NN}, $\zonotope{}_{\text{in}}$ some Zonotope with input dimension $n$ and output dimension $I$ and $\zonotope{}',\zonotope{}'',\zonotope{}^\Delta$ the output of $\textsc{Reach}_\Delta\left(f_1,f_2,\zonotope{}_{\text{in}}\right)$.
For all $\mathbf{v} \in \left[-1,1\right]^n$:
\[
\big(f_1\left(\zonotope{}_{\text{in}}\left(\mathbf{v}\right)\right), f_2\left(\zonotope{}_{\text{in}}\left(\mathbf{v}\right)\right)\big) \in \langle \left(\zonotope{}'\left(\mathbf{v}\right),\zonotope{}''\left(\mathbf{v}\right),\zonotope{}^\Delta\left(\mathbf{v}\right)\right)\rangle.
\]
\end{corollaryE}
\begin{proofEnd}
    This result follows by the same proof as as for \Cref{thm:verydiff_verysound} (note the inductive hypothesis of \Cref{thm:verydiff_verysound}'s proof).
\end{proofEnd}

\begin{algorithm}[t]
    \caption{Verification with Differential Zonotopes}
    \label{algorithm:diff_zono:overall}
    \begin{algorithmic}
    \Require \acp{NN} $g_1,g_2:\mathbb{R}^I \to \mathbb{R}^O$ with $L$ layers, Input-Zonotope $Z_{\text{in}}=\left(G_{\text{in}},b_{\text{in}}\right)$
    \Ensure Reachable Zonotopes $\zonotope{}',\zonotope{}'',\zonotope{}^\Delta$
    \Procedure{$\textsc{Reach}_\Delta$}{$g_1,g_2,Z_{\text{in}}$}
    \State $\zonotope{}' \leftarrow Z_{\text{in}}$
    \State $\zonotope{}'' \leftarrow \text{copy}\left(Z_{\text{in}}\right)$
    \State $\zonotope{}^\Delta \leftarrow \left(0,0\right) \in \mathbb{R}^{I\times I} \times \mathbb{R}^I$
    \Comment{Initialize Differential Zonotope with 0}
    \For{$l \in \left[1,L\right]$}
        \If{layer $l$ is affine}
            \State $\zonotope{}^\Delta \leftarrow \textsc{Affine}_\Delta\left(\zonotope{}^\Delta,\zonotope{}',\zonotope{}'',W^{(l)}_{1/2},\mathbf{b}^{(l)}_{1/2}\right)$\\
            \Comment{See \Cref{lemma:diff_zono:zono_affine}}
            \State $\zonotope{}' \leftarrow \textsc{Affine}\left(\zonotope{}',W^{(l)}_1,\mathbf{b}^{(l)}_1\right)$\\
            \Comment{See Transformation in \Cref{prop:background:zono_affine}}
            \State $\zonotope{}'' \leftarrow \textsc{Affine}\left(\zonotope{}'',W^{(l)}_2,\mathbf{b}^{(l)}_2\right)$\\
            \Comment{See Transformation in \Cref{prop:background:zono_affine}}
        \Else \Comment{Transformation for \reluSym{} Layer}
            \State $\hat{\zonotope{}}',\hat{\zonotope{}}'' \leftarrow \textsc{ReLU}\left(\zonotope{}'\right), \textsc{ReLU}\left(\zonotope{}''\right)$\\
            \Comment{See Transformation in \Cref{prop:background:zono_relu}}
            \State $\zonotope{}^\Delta \leftarrow \textsc{ReLU}_\Delta\left(\zonotope{}^\Delta,\zonotope{}',\zonotope{}'',\hat{\zonotope{}}',\hat{\zonotope{}}''\right)$\\
            \Comment{See \Cref{lemma:diff_zono:zono_relu}}
            \State $\zonotope{}',\zonotope{}'' \leftarrow \hat{\zonotope{}}',\hat{\zonotope{}}''$
        \EndIf
    \EndFor
    \Return $\zonotope{}',\zonotope{}'',\zonotope{}^\Delta$
    \EndProcedure
    
    \end{algorithmic}
\end{algorithm}

\noindent
\looseness=-1
For the descriptions of transformations, we again focus on a single Affine Form.
In this section, we denote the representation of an Affine Form as $\affineForm{}=\left(\mathbf{e},\mathbf{a},c\right)$ where $\mathbf{e}$ are the $n$ original (\emph{e}xact) generators present in an Affine Form of $\left(\zonotope{}_{\text{in}}\right)_i$ and $\mathbf{a}$ are the $p$ (\emph{a}pproximate) generators added via \reluSym{} transformations.
For Differential Affine Forms we split the approximate generators into $\mathbf{a}'^\Delta,\mathbf{a}''^\Delta$ and $\mathbf{a}^\Delta$, 
distinguishing their origin ($\affineForm{}{}'$,$\affineForm{}''$ or generators added to $\affineForm{}^\Delta$ directly).
This yields the Affine Form $\affineForm{}^\Delta = \left(\mathbf{e}^\Delta,\mathbf{a}'^\Delta,\mathbf{a}''^\Delta,\mathbf{a}^\Delta,c^\Delta\right)$.
We assume that the corresponding vectors have equal dimensions (i.e. $\mathbf{e}'$,$\mathbf{e}''$ and $\mathbf{e}^\Delta$ all have dimension $n_1$; $\mathbf{a}'$ and $\mathbf{a}'^\Delta$ have dimension $n_2$ and $\mathbf{a}''$ as well as $\mathbf{a}''^\Delta$ have dimension $n_3$).
The points described by two Affine Forms and a Differential Affine Form are then described via common generator values
$\mathbf{\epsilon_1},
\mathbf{\epsilon_2},
\mathbf{\epsilon_3},
\mathbf{\epsilon_4}$ across all three Affine Forms.
This definition naturally generalizes to the multidimensional (Zonotope) case by fixing all $\mathbf{\epsilon}$ values across dimensions and Zonotopes. 
We also denote the points contained in a tuple $\left(\affineForm{}',\affineForm{}'',\affineForm{}^\Delta\right)$ using $\langle \cdot \rangle$:
\begin{definition}[Points Contained by Differential Affine Form]
\label{def:diff_affine_form_containment}
Given two Affine Forms $\affineForm{}'=(\mathbf{e}',\mathbf{a}',c'), \affineForm{}''=\left(\mathbf{e}'',\mathbf{a}'',c''\right)$ and a Differential Affine Form $\affineForm{}^\Delta = \left(\mathbf{e}^\Delta,\mathbf{a}'^\Delta,\mathbf{a}''^\Delta,\mathbf{a}^\Delta,c^\Delta\right)$ with resp. a matching number of columns, we define the set of points described by $(\affineForm{}',\affineForm{}'',\affineForm{}^\Delta)$ where $n_1,n_2,n_3,n_4$ are resp. the number of columns in $\mathbf{e}^\Delta,\mathbf{a}'^\Delta,\mathbf{a}''^\Delta,\mathbf{a}^\Delta$ and $\overline{n}=n_1+n_2+n_3+n_4$:
\begin{align*}
\langle
(\affineForm{}',\affineForm{}'',\affineForm{}^\Delta)
\rangle &= \\
\Big\{
\left(x,y\right) ~\Big|&~
\exists~
\mathbf{\epsilon_1} \in \left[-1,1\right]^{n_1},~
\mathbf{\epsilon_2} \in \left[-1,1\right]^{n_2},~
\mathbf{\epsilon_3} \in \left[-1,1\right]^{n_3},~
\mathbf{\epsilon_4} \in \left[-1,1\right]^{n_4}~
\\
& x = 
\mathbf{e}' \mathbf{\epsilon_1} + \mathbf{a}' \mathbf{\epsilon_2} + c'~\land~\\
&y = \mathbf{e}''\mathbf{\epsilon_1} + \mathbf{a}''\mathbf{\epsilon_3} + c''~\land\\
&
(x-y) = 
\mathbf{e}' \mathbf{\epsilon_1} + 
\mathbf{a}'^\Delta \mathbf{\epsilon_2} +
\mathbf{a}''^\Delta \mathbf{\epsilon_3} + 
\mathbf{a}^\Delta \mathbf{\epsilon_4} + c^\Delta
\Big\}
\end{align*}
\end{definition}%

\begin{wrapfigure}[20]{r}{0.25\textwidth}
  \centering
  \begin{subfigure}[t]{\linewidth}
  \begin{subfigure}{\linewidth}
  \includegraphics[width=\linewidth]{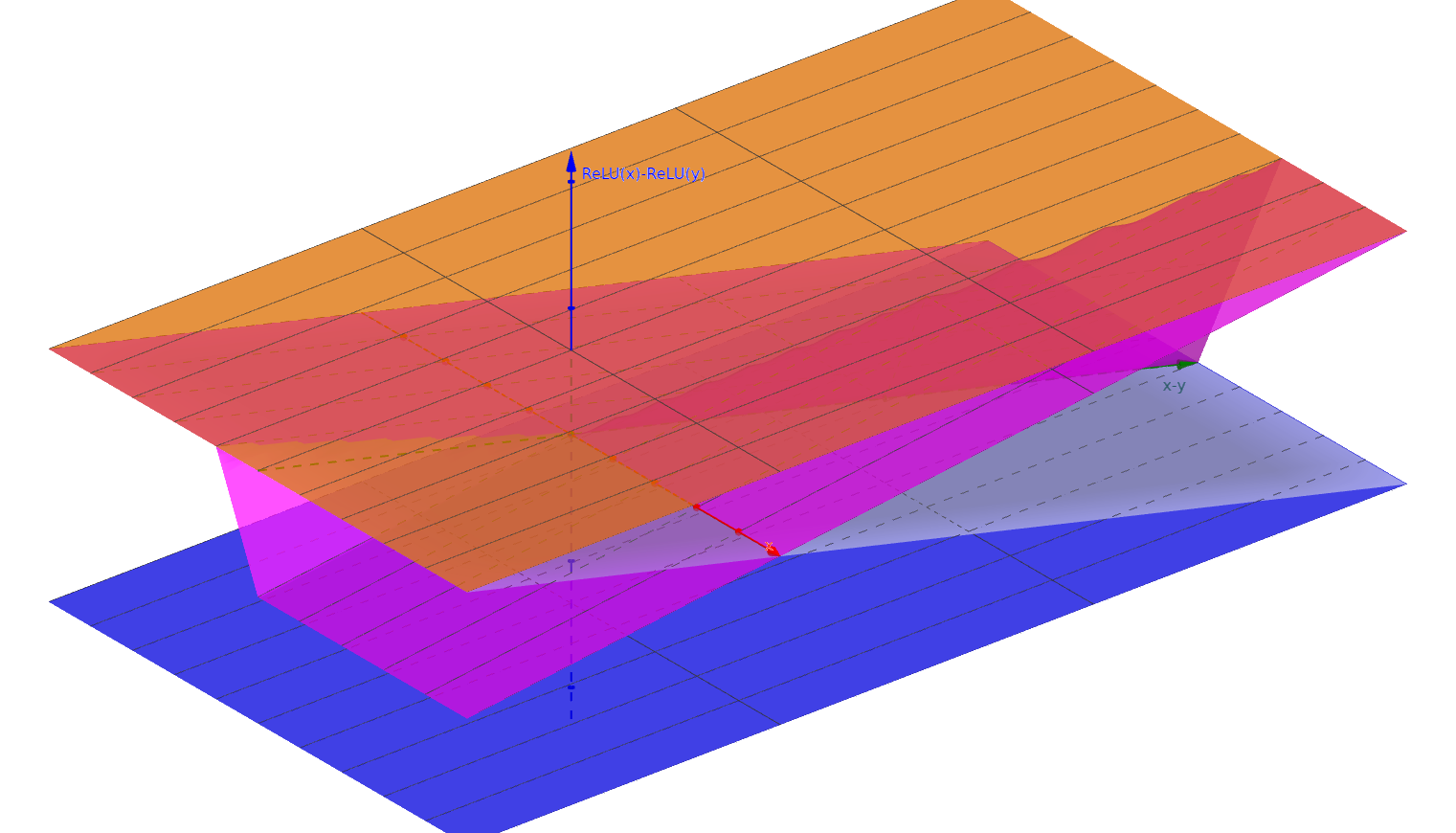}
  \caption{Bounds on difference of two instable neurons.}
  \label{subfig:3d_instable}
  \end{subfigure}
  
  \begin{subfigure}{\linewidth}
  \includegraphics[width=\linewidth]{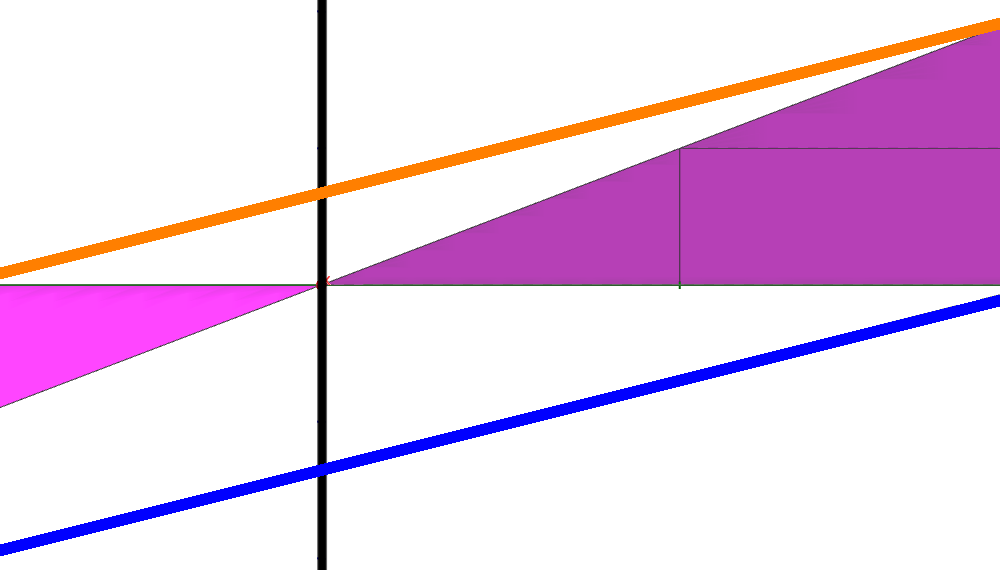}
    \caption{Projection of \Cref{subfig:3d_instable} w.r.t. neuron difference.}
  \label{fig:instable_neurons}
  \end{subfigure}
  \end{subfigure}
  \caption{Visualization for the construction of a new Affine Form via $\textsc{ReLU}_\Delta$}
  \label{fig:relu_delta_visualization}
\end{wrapfigure}
\noindent
Omitting the computation of $\zonotope{}^\Delta$, $\textsc{Reach}_\Delta$ corresponds to computing reachable outputs of $f_1,f_2$ w.r.t. a common input Zonotope $\zonotope{}_{\text{in}}$.
This algorithm implicitly computes the \emph{naive} Differential Affine Form $\affineForm{}' - \affineForm{}''$, i.e. the affine form $\affineForm{}^\Delta = \left(\mathbf{e}'-\mathbf{e}'',\mathbf{a}',-\mathbf{a}'',0,c'-c''\right)$.

\paragraph{Affine Transformations.}
For affine transformations, we construct a transformation based on the insights from Differential Symbolic Bounds~\cite{paulsen_reludiff_2020} that exactly two quantities determine the output difference:
First, the difference accumulated so far (represented by the current Differential Affine Form $\affineForm{}_\Delta$ and scaled by the layer's affine transformation); second, the difference between the affine transformations in the current layer.
Given two affine transformations $w^{(l)}_1 x + b^{(l)}_1$ and $w^{(l)}_2 x + b^{(l)}_2$ and Affine Forms $\affineForm{}'=\left(\mathbf{e}',\mathbf{a}',c'\right),\affineForm{}''=\left(\mathbf{e}'',\mathbf{a}'',c''\right),\affineForm{}^\Delta=\left(\mathbf{e}^\Delta,\mathbf{a}'^\Delta,\mathbf{a}''^\Delta,\mathbf{a}^\Delta,c^\Delta\right)$ we 
propose $\textsc{Affine}_\Delta$ which returns a new Affine Form $\hat{\affineForm{}}^\Delta=\left(\hat{\mathbf{e}}^\Delta,\hat{\mathbf{a}}'^\Delta,\hat{\mathbf{a}}''^\Delta,\hat{\mathbf{a}}^\Delta,\hat{c}^\Delta\right)$ with:
\begin{align}
\label{eq:affine_diff_zono}
\begin{aligned}
    \hat{\mathbf{e}}^\Delta &= w^{(l)}_1 \mathbf{e}^\Delta + \left(w^{(l)}_1-w^{(l)}_2\right) \mathbf{e}''
    &\hat{\mathbf{a}}'^\Delta &= w^{(l)}_1 \mathbf{a}'^\Delta\\
    \hat{\mathbf{a}}''^\Delta &= w^{(l)}_1 \mathbf{a}''^\Delta + \left(w^{(l)}_1-w^{(l)}_2\right) \mathbf{a}'' & \hat{\mathbf{a}}^\Delta &= w^{(l)}_1  \mathbf{a}^{\Delta}\\
    \rlap{$\hat{c}^\Delta = w^{(l)}_1 c^\Delta + \left(w^{(l)}_1-w^{(l)}_2\right) c'' + \left(b^{(l)}_1-b^{(l)}_2\right)$}
\end{aligned}
\end{align}
\looseness=-1
We scale prior differences by $w^{(l)}_1$ and add the new difference $(w^{(l)}_1-w^{(l)}_2)$ to $\hat{\mathbf{e}}^\Delta$, $\hat{\mathbf{a}}''$ (scaled by the reachable values of $f_2$, i.e. by $\affineForm{}''$) and $\hat{c}^\Delta$ (and resp. $(b^{(l)}_1-b^{(l)}_2)$).
The transformation is sound (proof on \cpageref{proof:diff_zono:zono_affine} generalizes to Zonotopes):
\begin{lemmaE}[Soundness of $\textsc{Affine}_\Delta$][end,restate,text link=]
\label{lemma:diff_zono:zono_affine}
    For affine transformations $\alpha_1\left(x\right)=w^{(l)}_1 x + b^{(l)}_1$ and $\alpha_2\left(x\right)=w^{(l)}_2 x + b^{(l)}_2$ and Affine Forms
    $\affineForm{}'=\left(\mathbf{e},\mathbf{a}',c'\right),\affineForm{}''=\left(\mathbf{e}'',\mathbf{a}'',\mathbf{c}''\right),\affineForm{}^\Delta=\left(\mathbf{e}^\Delta,\mathbf{a}^\Delta,\mathbf{a}^\Delta,\mathbf{a}^\Delta,c^\Delta\right)$ the transformation $\text{Affine}_\Delta$ is sound, i.e. if $\hat{\affineForm{}}^\Delta$ is the result of $\textsc{Affine}_\Delta$, %
    then for all $d \leq n$ and $\mathbf{v} \in \left[-1,1\right]^d$
    with $(x,y) \in \langle (\affineForm{}'\left(\mathbf{v}\right),\affineForm{}''\left(\mathbf{v}\right),\affineForm{}^\Delta\left(\mathbf{v}\right)) \rangle$ 
    and $\hat{\affineForm{}}',\hat{\affineForm{}}''$ the outputs of $\textsc{Affine}$ (\Cref{prop:background:zono_affine}) we get:
    \[
    \left(\alpha_1\left(x\right),\alpha_2\left(y\right)\right)
    \in 
    \langle (\tilde{z}_1\left(v\right),\tilde{z}_2\left(v\right),\tilde{z}_\Delta\left(v\right)) \rangle
    \]
\end{lemmaE}
\begin{proofEnd}
\label{proof:diff_zono:zono_affine}
    We perform the proof w.r.t. the Matrix representation of Zonotopes ($\zonotope{}=\left(G,\mathbf{c}\right)$) and weights ($W$) -- this obviously implies the one dimensional case.
    Let $\left(\mathbf{x},\mathbf{y}\right) \in \langle (\zonotope{}'\left(\mathbf{v}\right),\zonotope{}''\left(\mathbf{v}\right),\zonotope{}^\Delta\left(\mathbf{v}\right)) \rangle$, then we have $\mathbf{\epsilon}_1,\mathbf{\epsilon}_2,\mathbf{\epsilon}_3,\mathbf{\epsilon}_4$ representing $\mathbf{x},\mathbf{y}$ and $\mathbf{x}-\mathbf{y}$ w.r.t. $\zonotope{}'\left(\mathbf{v}\right),\zonotope{}''\left(\mathbf{v}\right)$ and $\zonotope{}^\Delta\left(\mathbf{v}\right)$.
    The value of a vector $\mathbf{v}$ can be integrated by moving the the columns multiplied by $v$ into the Zonotope's bias.
    \begin{align*}
        & \left(W^{(l)}_1 \mathbf{x} + \mathbf{b}^{(l)}_1\right) - \left(W^{(l)}_2 \mathbf{y} + \mathbf{b}^{(l)}_2\right)\\
        =& \left(W^{(l)}_1 \mathbf{x} - W^{(l)}_2 \mathbf{y}\right) + \left(\mathbf{b}^{(l)}_1 - \mathbf{b}^{(l)}_2\right)\\
        =& \left(W^{(l)}_1 (\mathbf{y}+(\mathbf{x}-\mathbf{y})) - W^{(l)}_2 \mathbf{y})\right) + \left(\mathbf{b}^{(l)}_1 - \mathbf{b}^{(l)}_2\right)\\
        =& \left(\left(W^{(l)}_1 - W^{(l)}_2\right) \mathbf{y} + W^{(l)}_1 (\mathbf{x}-\mathbf{y}))\right) + \left(\mathbf{b}^{(l)}_1 - \mathbf{b}^{(l)}_2\right)\\
        =& \left(W^{(l)}_1 - W^{(l)}_2\right) \underbrace{\mathbf{y}}_{\in \langle \zonotope{}''\left(\mathbf{v}\right) \rangle} + \left(\mathbf{b}^{(l)}_1 - \mathbf{b}^{(l)}_2\right) + W^{(l)}_1 \underbrace{(\mathbf{x}-\mathbf{y})}_{\in \zonotope{}^\Delta\left(\mathbf{v}\right)}\\
        \intertext{Thus, we can replace $\mathbf{y}$ and $\mathbf{x}-\mathbf{y}$ by their representation w.r.t. generators. We will now use this fact to derive that the difference after application of an affine transformation is contained in the newly constructed $\hat{\zonotope{}}^\Delta$. To this end, first note that we can rewrite the formula as follows:}\\
        =&
        \left(W^{(l)}_1 - W^{(l)}_2\right) \left(E''\mathbf{\epsilon}_1 + A''\mathbf{\epsilon}_3 + \mathbf{c}''\right) +
        \left(\mathbf{b}^{(l)}_1 - \mathbf{b}^{(l)}_2\right)~+\\
        &W^{(l)}_1 
        \left(E^\Delta\mathbf{\epsilon}_1 + A'^\Delta\mathbf{\epsilon}_2 + A''^\Delta\mathbf{\epsilon}_3 + A^\Delta\mathbf{\epsilon}_4 + \mathbf{c}^\Delta\right)\\
        \intertext{Sorting the components by $\mathbf{\epsilon}_i$, this yields:}
        =&
        \left(\left(W^{(l)}_1 - W^{(l)}_2\right)E'' + W^{(l)}_1 E^\Delta\right) \mathbf{\epsilon}_1~+ \\
        &  W^{(l)}_1 A'^\Delta \mathbf{\epsilon}_2~+\\
        & \left(\left(W^{(l)}_1 - W^{(l)}_2\right)A'' + W^{(l)}_1 A''^\Delta\right) \mathbf{\epsilon}_3~+ \\
        &  W^{(l)}_1 A^\Delta \mathbf{\epsilon}_4~+\\
        & \left(W^{(l)}_1 - W^{(l)}_2\right) \mathbf{c}'' + 
        \left(\mathbf{b}^{(l)}_1 - \mathbf{b}^{(l)}_2\right) + 
        W^{(l)}_1 \mathbf{c}^\Delta\\
    \end{align*}
    The latter formula is obviously included in $\hat{\zonotope{}}^\Delta$ constructed according to the rules in \Cref{eq:affine_diff_zono}.
    The inclusion of outputs w.r.t. $\hat{\zonotope{}}_1$ and $\hat{\zonotope{}}_2$ is ensured via \Cref{prop:background:zono_affine}, i.e. for example $W_1^{(l)} \mathbf{x} + b_1^{(l)} = W_1^{(l)} (E' \mathbf{\epsilon}_1 + A' \mathbf{\epsilon}_2 + \mathbf{c}') + \mathbf{b}_1^{(l)} = 
     W_1^{(l)} E' \mathbf{\epsilon}_1 + W_1^{(l)} A' \mathbf{\epsilon}_2 + (\mathbf{c}' + \mathbf{b}_1^{(l)})$.
\end{proofEnd}

\paragraph{\reluSym{} Transformations.}
\looseness=-1
Since \reluSym{} is piece-wise linear, we cannot hope to construct an exact transformation for this case.
However, by carefully distinguishing the possible cases (e.g. both nodes are solely negative, one is negative one is instable, etc.) we provide linear representations for 6 out of the 9 cases (both nodes stable or a negative and an instable node) while overapproximating the other three.
Given $\affineForm{}'=\left(\mathbf{e}',\mathbf{a}',c'\right),\affineForm{}''=\left(\mathbf{e}'',\mathbf{a}'',c''\right)$ and a Differential Affine Form $\affineForm{}^\Delta=\left(\mathbf{e}^\Delta,\mathbf{a}'^\Delta,\mathbf{a}''^\Delta,\mathbf{a}^\Delta,c^\Delta\right)$
as well as $\hat{\affineForm{}}',\hat{\affineForm{}}''$ (the result of applying $\reluSym{}$ in the individual \acp{NN}), we compute the result of the $\reluSym{}_\Delta$ transformation, using the case distinctions from \Cref{tab:relu_transform}.
The transformation is sound (see \cpageref{proof:diff_zono:zono_relu}):
\begin{lemmaE}[
Soundness of $\reluSym{}_\Delta$][end,restate,text link=]
\label{lemma:diff_zono:zono_relu}
Consider Affine Form $\affineForm{}'=\left(\mathbf{e},\mathbf{a}',c'\right),\affineForm{}''=\left(\mathbf{e}'',\mathbf{a}'',\mathbf{c}''\right)$ and a Differential Affine Form $\affineForm{}^\Delta=\left(\mathbf{e}^\Delta,\mathbf{a}'^\Delta,\mathbf{a}''^\Delta,\mathbf{a}^\Delta,c^\Delta\right)$.
Let $\hat{\affineForm{}}'$/$\hat{\affineForm{}}''$ be the result of the $\textsc{ReLU}$ transformation on $\affineForm{}'$/$\affineForm{}''$ (see \Cref{prop:background:zono_relu}).
Define $\hat{\affineForm{}}^\Delta=\left(\hat{\mathbf{e}}^\Delta,\hat{\mathbf{a}}^\Delta,\hat{\mathbf{a}}^\Delta,\hat{\mathbf{a}}^\Delta,c^\Delta\right)$ such that it matches the case distinctions in \Cref{tab:relu_transform}.
Then $\hat{\affineForm{}}^\Delta$  overapproximates the behavior of $\reluSym{}$, i.e. for all $d \leq n$ and $\mathbf{v} \in \left[-1,1\right]^d$ with  $(x,y) \in \langle (\affineForm{}'\left(\mathbf{v}\right),\affineForm{}''\left(\mathbf{v}\right),\affineForm{}^\Delta\left(\mathbf{v}\right)) \rangle$ we get that:
\[
\left(\relu{x},\relu{y}\right) \in \langle (\hat{\affineForm{}}'\left(\mathbf{v}\right),\hat{\affineForm{}}''\left(\mathbf{v}\right),\hat{\affineForm{}}^\Delta\left(\mathbf{v}\right)) \rangle
\]
\end{lemmaE}%
\begin{proofEnd}
\label{proof:diff_zono:zono_relu}
Given arbitrary $x,y$ as constrained above, we fix $\mathbf{\epsilon}_1,\mathbf{\epsilon}_2,\mathbf{\epsilon}_3,\mathbf{\epsilon}_4$ such that 
$
x = \mathbf{e}' \mathbf{\epsilon}_1 + \mathbf{a}' \mathbf{\epsilon}_2 + c'
$,
$y=\mathbf{e}'' \mathbf{\epsilon}_1 + \mathbf{a}'' \mathbf{\epsilon}_3 + c''$,
and
$x-y =  \mathbf{e}^\Delta \mathbf{\epsilon}_1 + \mathbf{a}'^\Delta \mathbf{\epsilon}_2 + \mathbf{a}''^\Delta \mathbf{\epsilon}_3 + \mathbf{a}^\Delta \mathbf{\epsilon}_4 + \mathbf{c}^\Delta$.
Then, we know via \Cref{prop:background:zono_relu} (and its proof) for the Zonotopes $\hat{\affineForm{}}',\hat{\affineForm{}}''$ that there are $\mathbf{\epsilon}_2^+,\mathbf{\epsilon}_3^+\in\left[-1,1\right]$ such that for $\hat{\mathbf{\epsilon}}_i = \left(\mathbf{\epsilon}_i^T~\mid~\mathbf{\epsilon}_{i}^+\right)^T$ with $i \in\left\{2,3\right\}$ it holds that
$
\relu{x}=
\hat{\mathbf{e}}' \mathbf{\epsilon}_1 + \hat{\mathbf{a}}' \hat{\mathbf{\epsilon}}_2 + \hat{c}'$
and
$\relu{y}=
\hat{\mathbf{e}}'' \mathbf{\epsilon}_1
+ \hat{\mathbf{a}}'' \mathbf{\epsilon}_3 + \hat{c}''$
(via $\affineForm{}'\left(\mathbf{\epsilon}_1\right)$/$\affineForm{}''\left(\mathbf{\epsilon}_1\right)$).
Note, that $\hat{\mathbf{\epsilon}}_i$ and $\mathbf{\epsilon}$ are equal on its common prefix of dimensions
It remains to show that $\relu{x}-\relu{y}$ is inside $\hat{\affineForm{}}^\Delta$ for the fixed values of $\hat{\mathbf{\epsilon}}_1=\mathbf{\epsilon}_1$ and $\hat{\mathbf{\epsilon}}_2,\hat{\mathbf{\epsilon}}_3$ as well as $\hat{\mathbf{\epsilon}}_4$ which will be case dependent:
For the first six cases, we set $\hat{\mathbf{\epsilon}}_4=\mathbf{\epsilon}_4$. For the remaining three cases, $\hat{\mathbf{\epsilon}}_4$ and $\mathbf{\epsilon}_4$ are equal on the common prefix of dimensions, but $\hat{\mathbf{\epsilon}}_4$ is extended by an additional dimension as explained below.
We perform a case distinction over the 9 cases of \Cref{tab:relu_transform}.
Where the dimension does not match for $\mathbf{a}'$/$\mathbf{a}''$ in comparison to $\mathbf{a}'^\Delta$/$\mathbf{a}''^\Delta$ we can silently append zeros.
\begin{itemize}
    \item[--,--]
    All $x,y$ are negative. This yields:
    \[
    \relu{x}-\relu{y}=0-0=0\in\langle\left(0,0,0,0,0\right)\rangle
    \]
    \item[--,+]
    All $x$ are negative while all $y$ (with $y \in \langle \left(\mathbf{e}'',\mathbf{a}'',c''\right) \rangle$) are positive. This yields:
    \begin{align*}
    &\relu{x}-\relu{y}=0-y~=\\
    =~&-\mathbf{e}'' \mathbf{\epsilon}_1 - \mathbf{a}'' \mathbf{\epsilon}_3 - c''
    \in\langle\left(-\mathbf{e}'',0,-\mathbf{a}'',0,-c''\right)\rangle)
    \end{align*}
    \item[+,--]
    All $x$ are positive (with $x \in \langle \left(\mathbf{e}',\mathbf{a}',c'\right) \rangle$) while all $y$ are negative. This yields:
    \begin{align*}
    &\relu{x}-\relu{y}=x-0~=\\
    =~&\mathbf{e}' \mathbf{\epsilon}_1 + \mathbf{a}' \mathbf{\epsilon}_2 + c'
    \in\langle\left(\mathbf{e}',-\mathbf{a}',0,0,+c'\right)\rangle)
    \end{align*}
    \item[+,+]
    All $x$ and $y$ are positive with $\left(x-y\right) \in \langle\left( \mathbf{e}^\Delta,\mathbf{a}^\Delta,\mathbf{a}^\Delta,\mathbf{a}^\Delta,c^\Delta \right)\rangle$. This yields by assumption:
    \[
    \relu{x}-\relu{y}=x-y\in\langle\left( \mathbf{e}^\Delta,\mathbf{a}^\Delta,\mathbf{a}^\Delta,\mathbf{a}^\Delta,c^\Delta \right)\rangle
    \]
    \item[$\sim$,--]
    $x$ is instable while all $y$ are negative with $\relu{x}\in\langle\left(\hat{\mathbf{e}}',\hat{\mathbf{a}}',\hat{c}'\right)\rangle$. This yields:
    \begin{align*}
    &\relu{x}-\relu{y}=\relu{x}-0~=\\
    =~&
    \hat{\mathbf{e}}' \hat{\mathbf{\epsilon}}_1 + \hat{\mathbf{a}}' \hat{\mathbf{\epsilon}}_2 + \hat{c}'%
    \in\langle\left(
    \hat{\mathbf{e}}',\hat{\mathbf{a}}',0,0,\hat{c}'
    \right)\rangle
    \end{align*}
    \item[--,$\sim$]
    $x$ is negative while all $y$ are instable with $\relu{y}\in\langle\left(\hat{\mathbf{e}}'',\hat{\mathbf{a}}'',\hat{c}''\right)\rangle$. This yields
    \begin{align*}
    &\relu{x}-\relu{y}=0-\relu{y}~=\\
    =~&
    -\hat{\mathbf{e}}'' \hat{\mathbf{\epsilon}}_1 - \hat{\mathbf{a}}'' \hat{\mathbf{\epsilon}}_3 - \hat{c}'%
    \in\langle\left(
    -\hat{\mathbf{e}}'',
    0,
    -\hat{\mathbf{a}}'',0,-\hat{c}''
    \right)\rangle
    \end{align*}
    \item[$\sim$,+]
    $x$ is instable while all $y$ are positive with $\left(x-y\right) \in \langle\left(
    \mathbf{e}^\Delta,\mathbf{a}'^\Delta,\mathbf{a}''^\Delta,\mathbf{a}^\Delta,c^\Delta
    \right)\rangle$ and $\relu{x}\in\langle\left(\hat{\mathbf{e}}',\hat{\mathbf{a}}',\hat{c}'\right)\rangle$. This yields:
    \begin{align*}
    &\relu{x}-\relu{y}~=\\
    =~&\relu{x}-(x-(x-y))~=\\
    =~&(x-y)+\max\left(0,-x\right)=\\
    =~&
    \mathbf{e}^\Delta \mathbf{\epsilon}_1
    + \mathbf{a}'^\Delta \mathbf{\epsilon}_2
    + \mathbf{a}''^\Delta \mathbf{\epsilon}_3
    + \mathbf{a}^\Delta \mathbf{\epsilon}_4
    + c^\Delta + \max\left(0,-x\right)
    \end{align*}
    By using the classical Zonotope construction for $\max$, we get for $\lambda' = \frac{-\zonoLower{\affineForm{}'}}{\zonoUpper{\affineForm{}'}-\zonoLower{\affineForm{}'}}$ that $\max\left(0,-x\right)$ is contained in the following Affine Form:
    \[
    \langle \left( -\lambda' \mathbf{e}', \left(-\lambda' (\mathbf{a})'^T \mid 0.5*\lambda'\zonoUpper{\affineForm{}'}\right)^T, -\lambda' c' +0.5*\lambda'\zonoUpper{\affineForm{}'}\right)\rangle.
    \]
    Note, that for this formulation we can then reuse $\hat{\mathbf{\epsilon}}_1,\hat{\epsilon}_2$ and can turn it into a correct representation of $\max\left(0,-x\right)$ by choosing an appropriate value for the newly introduced generator.
    Since the new generator is \emph{only} relevant for the Differential Affine Form, we append it to $\hat{\mathbf{a}}^\Delta$ and choose an appropriate ${\mathbf{\epsilon}}_4^+\in\left[-1,1\right]$.
    Since the Affine Form above bounds all values of $\max\left(0,-x\right)$ for $x$ inside $\langle \affineForm{}' \rangle$ and fixed $\hat{\mathbf{\epsilon}}_1,\hat{\mathbf{\epsilon}}_2$, we know that such a fresh value ${\mathbf{\epsilon}}_4^+$ exists based on the intermediate value theorem.
    $(x-y)-\min\left(0,x\right)$ then is contained in the Affine Form
    \begin{align*}
    \rlap{$
    \langle \left(
    \hat{\mathbf{e}}^\Delta,
    \hat{\mathbf{a}}'^\Delta,
    \hat{\mathbf{a}}''^\Delta,
    \hat{\mathbf{a}}^\Delta,
    \hat{c}^\Delta
    \right)\rangle
    $}\\
    \intertext{with:}
    \hat{\mathbf{e}}^\Delta &=
    \mathbf{e}^\Delta - \lambda' \mathbf{e}',\\
    \hat{\mathbf{a}}'^\Delta &=
    \mathbf{a}'^\Delta - \lambda' \mathbf{a}',\\
    \hat{\mathbf{a}}''^\Delta &=
    \mathbf{a}''^\Delta,\\
    \hat{\mathbf{a}}^\Delta &=
    ( (\mathbf{a}^{\Delta})^T \mid 0.5*\lambda'\zonoUpper{\affineForm{}'} )^T,\\
    \hat{c} &=
    c^\Delta - \lambda' c' + 0.5*\lambda'\zonoUpper{\affineForm{}'}.
    \end{align*}

    \item[+,$\sim$]
    $y$ is instable while all $x$ are positive with $\left(x-y\right) \in \langle\left(
    \mathbf{e}^\Delta,\mathbf{a}'^\Delta,\mathbf{a}''^\Delta,\mathbf{a}^\Delta,c^\Delta
    \right)\rangle$ and $\relu{y}\in\langle\left(\hat{\mathbf{e}}'',\hat{\mathbf{a}}'',\hat{c}''\right)\rangle$. This yields:
    \begin{align*}
    &\relu{x}-\relu{y}~=\\
    =~&((x-y)+y)-\relu{y}~=\\
    =~&(x-y)-\max\left(0,-y\right)
    \end{align*}
    By the same argument as above (switching $\hat{\mathbf{\epsilon}}_2$ with $\hat{\mathbf{\epsilon}}_3$ and considering the switched sign of $\max\left(0,-y\right)$) this yields that $(x-y)+\max\left(0,-y\right)$ is in the following set:
    \begin{align*}
    \rlap{$
    \langle \left(
    \hat{\mathbf{e}}^\Delta,
    \hat{\mathbf{a}}'^\Delta,
    \hat{\mathbf{a}}''^\Delta,
    \hat{\mathbf{a}}^\Delta,
    \hat{c}^\Delta
    \right)\rangle
    $}\\
    \intertext{with:}
    \hat{\mathbf{e}}^\Delta &=
    \mathbf{e}^\Delta + \lambda'' \mathbf{e}''\\
    \hat{\mathbf{a}}'^\Delta &=
    \mathbf{a}'^\Delta\\
    \hat{\mathbf{a}}''^\Delta &=
    \mathbf{a}''^\Delta + \lambda'' \mathbf{a}''\\
    \hat{\mathbf{a}}^\Delta &=
    ( (\mathbf{a}^{\Delta})^T \mid 0.5*\lambda''\zonoUpper{\affineForm{}''} )^T\\
    \hat{c} &=
    c^\Delta + \lambda'' c'' - 0.5*\lambda''\zonoUpper{\affineForm{}''}
    \end{align*}
    With $\lambda'' = \frac{-\zonoLower{\affineForm{}''}}{\zonoUpper{\affineForm{}''}-\zonoLower{\affineForm{}''}}$
    \item[$\sim$,$\sim$]
    Both $x$ and $y$ are instable with $\left(x-y\right) \in \langle\left(
    \mathbf{e}^\Delta,\mathbf{a}'^\Delta,\mathbf{a}''^\Delta,\mathbf{a}^\Delta,c^\Delta
    \right)\rangle$.
    This yields:
    \begin{align}
    &\relu{x}-\relu{y}~= \nonumber\\
    =~&
    \relu{x}-\relu{x-(x-y)}~= \nonumber\\
    =~&\begin{cases}
        (x-y) & x \geq 0 \land x \geq (x-y)\\
        x & x \geq 0 \land x < (x-y)\\
        (x-y)-x & x< 0 \land x \geq (x-y)\\
        0 & x < 0 \land x < (x-y)
    \end{cases}\label{eq:diff_relu:relu_cases}
    \end{align}
    For this case, we need to construct an Affine Form with an additional generator to take account of the outputs' nonlinearities.
    We prove that the lower/upper bound of the constructed Affine Form is always below/above the function values.
    Since the Affine Form contains all points between the linear lower/upper bound considered here, this ensures that our output always lies within the output Affine Form.
    Moreover, we note that, since we consider lower and upper bounds, according to the intermediate value theorem we can choose a suitable assignment for the fresh generator $\mathbf{\epsilon}_4^+$ for given assignments of $\hat{\mathbf{\epsilon}}_1,\hat{\mathbf{\epsilon}}_2,\hat{\mathbf{\epsilon}}_3,\hat{\mathbf{\epsilon}}_4$ as the lower/upper bound properties shown below imply that for a given assignment there must exist a concrete $\mathbf{\epsilon}_4^+$ value such that the Affine Form's equation \emph{exactly} represents the difference $(x-y)$.

    The Affine Form's tuple in this case reads as follows:
    \begin{align*}
    \rlap{$
    \langle \left(
    \hat{\mathbf{e}}^\Delta,
    \hat{\mathbf{a}}'^\Delta,
    \hat{\mathbf{a}}''^\Delta,
    \hat{\mathbf{a}}^\Delta,
    \hat{c}^\Delta
    \right)\rangle
    $}\\
    \intertext{with:}
    \hat{\mathbf{e}}^\Delta &=
    \lambda^\Delta \mathbf{e}^\Delta\\
    \hat{\mathbf{a}}'^\Delta &=
    \lambda^\Delta \mathbf{a}'^\Delta\\
    \hat{\mathbf{a}}''^\Delta &=
    \lambda^\Delta \mathbf{a}''^\Delta\\
    \hat{\mathbf{a}}^\Delta &=
    ( (\lambda^\Delta \mathbf{a}^{\Delta})^T \mid \mu^\Delta )^T\\
    \hat{c} &=
    \lambda^\Delta c^\Delta + \nu^\Delta - \mu^\Delta\\
    \lambda^\Delta &= \mathrm{clamp}\left(\frac{\zonoUpper{\affineForm{}^\Delta}}{\zonoUpper{\affineForm{}^\Delta}-\zonoLower{\affineForm{}^\Delta}},0,1\right)\\
    \mu^\Delta &= 0.5*\max\left(-\zonoLower{\affineForm{}^\Delta},\zonoUpper{\affineForm{}^\Delta}\right)\\
    \nu^\Delta &= \lambda^\Delta * \max\left(0,-\zonoLower{\affineForm{}^\Delta}\right)
    \end{align*}
    And this can be expressed as a bound on $\delta = \relu{x}-\relu{x-(x-y)}$:
    \begin{equation*}
    \lambda^\Delta (x-y) + \nu^\Delta - 2\mu^\Delta \leq \delta \leq  \lambda^\Delta (x-y) + \nu^\Delta
    \end{equation*}

    Our proof procedes as follows:
    First, we prove that the new Zonotope bounds the equation $(x-y)$ everywhere.
    Then we prove that the new Zonotope also bounds $0$ everywhere.
    Finally, we show that the bounds also apply to the cases 2 and 3 of \Cref{eq:diff_relu:relu_cases}.

    \paragraph{Lower Bound for $\mathbf{\left(x-y\right)}$.}
    We now show that the Affine Form's lower bound is indeed a bound for $\left(x-y\right)$.
    The Affine Form's lower bound reads
    \[
    \lambda^\Delta \underbrace{(x-y)}_{\in \left[\zonoLower{\affineForm{}^\Delta},\zonoUpper{\affineForm{}^\Delta}\right]} + \lambda^\Delta \max\left(0,-\zonoLower{\affineForm{}^\Delta}\right) - \max\left(-\zonoLower{\affineForm{}^\Delta},\zonoUpper{\affineForm{}^\Delta}\right)
    \]
    Case 1: $\zonoLower{\affineForm{}^\Delta} \geq 0$\\
    In this case, the formula reduces to
    \[
    \underbrace{\lambda^\Delta}_{\in\left[0,1\right]} \underbrace{(x-y)}_{\in \left[\zonoLower{\affineForm{}^\Delta},\zonoUpper{\affineForm{}^\Delta}\right]} - \zonoUpper{\affineForm{}^\Delta}.
    \]
    Since we know that $\zonoUpper{\affineForm{}^\Delta}$ must also be positive, this guarantees that the bound is smaller than $\left(x-y\right)$.\\[1em]
    Case 2: $\zonoLower{\affineForm{}^\Delta} < 0$\\
    Case 2.1.: $-\zonoLower{\affineForm{}^\Delta} \geq \zonoUpper{\affineForm{}^\Delta}$\\
    In this case, the formula reduces to
    \[
    \underbrace{\lambda^\Delta}_{\in\left[0,1\right]}
    (
        \underbrace{(x-y)}_{\in \left[\zonoLower{\affineForm{}^\Delta},\zonoUpper{\affineForm{}^\Delta}\right]}
        -\zonoLower{\affineForm{}^\Delta}
    )
    +\zonoLower{\affineForm{}^\Delta}.
    \]
    This equation is linear in $\left(x-y\right)$ and we thus consider the endpoints of its interval bounds:
    For $\left(x-y\right)=\zonoLower{\affineForm{}^\Delta}$ the formula reduces to $\zonoLower{\affineForm{}^\Delta}$ which is a valid bound.
    For $\left(x-y\right)=\zonoUpper{\affineForm{}^\Delta}$ the formula reduces to $\lambda^\Delta \left(\zonoUpper{\affineForm{}^\Delta} - \zonoLower{\affineForm{}^\Delta}\right)+\zonoLower{\affineForm{}^\Delta}$.
    Note, that for $\lambda^\Delta\in\left\{0,1\right\}$ this is a valid bound for $\left(x-y\right)=\zonoUpper{\affineForm{}^\Delta}$.
    For the remaining case ($\lambda^\Delta = \frac{\zonoUpper{\affineForm{}^\Delta}}{\zonoUpper{\affineForm{}^\Delta}-\zonoLower{\affineForm{}^\Delta}}$), the formula reduces to $\zonoUpper{\affineForm{}^\Delta}+\zonoLower{\affineForm{}^\Delta}$ which is a valid bound because $\zonoLower{\affineForm{}^\Delta}<0$ (Case 2).\\[1em]
    Case 2.2: $-\zonoLower{\affineForm{}^\Delta} < \zonoUpper{\affineForm{}^\Delta}$\\
    In this case, we know that $\zonoUpper{\affineForm{}^\Delta}>-\zonoLower{\affineForm{}^\Delta}>0$ (Case 2.2 and Case 2) and the formula reduces to
    \[
    \underbrace{\lambda^\Delta}_{\in\left[0,1\right]}
    (
        \underbrace{(x-y)}_{\in \left[\zonoLower{\affineForm{}^\Delta},\zonoUpper{\affineForm{}^\Delta}\right]}
        -\zonoLower{\affineForm{}^\Delta}
    )
    -\zonoUpper{\affineForm{}^\Delta}.
    \]
    We again consider the maxima w.r.t. $\left(x-y\right)$:
    For $\left(x-y\right)=\zonoLower{\affineForm{}^\Delta}$ the formula reduces to $-\zonoUpper{\affineForm{}^\Delta}$ which  (due to Case 2.2) is smaller than $\zonoUpper{\affineForm{}^\Delta} = \left(x-y\right)$.
    For $\left(x-y\right)=\zonoUpper{\affineForm{}^\Delta}$ the bound becomes $\lambda^\Delta \left(\zonoUpper{\affineForm{}^\Delta} - \zonoLower{\affineForm{}^\Delta}\right) - \zonoUpper{\affineForm{}^\Delta}$ which can be rewritten as
    $\left(\lambda^\Delta - 1\right) \zonoUpper{\affineForm{}^\Delta} - \lambda^\Delta \zonoLower{\affineForm{}^\Delta}$ which is a valid bound for $\lambda^\Delta\in\left\{0,1\right\}$.
    For the remaining case ($\lambda^\Delta = \frac{\zonoUpper{\affineForm{}^\Delta}}{\zonoUpper{\affineForm{}^\Delta}-\zonoLower{\affineForm{}^\Delta}}$), the formula reduces to $0$ which is a valid bound since we can assume that $\zonoUpper{\affineForm{}^\Delta}>0$ (Case 2.2).

    \paragraph{Upper Bound for $\mathbf{\left(x-y\right)}$.}
    The Affine Form's upper bound reads:
    \[
    \lambda^\Delta \underbrace{(x-y)}_{\in \left[\zonoLower{\affineForm{}^\Delta},\zonoUpper{\affineForm{}^\Delta}\right]} + \lambda^\Delta \max\left(0,-\zonoLower{\affineForm{}^\Delta}\right)
    \]
    Case 1: $\zonoLower{\affineForm{}^\Delta}<0$\\
    In this case, the bound reduces to $\lambda^\Delta \left( (x-y) - \zonoLower{\affineForm{}^\Delta} \right)$.
    Again, this formula is linear in $\left(x-y\right)$ and we look at its maximal values:
    For $\left(x-y\right) = \zonoLower{\affineForm{}^\Delta}$ the formula reduces to $0$ which is a valid upper bound (due to Case 1).
    For $\left(x-y\right) = \zonoUpper{\affineForm{}^\Delta}$ the formula reduces to $\lambda^\Delta \left( \zonoUpper{\affineForm{}^\Delta} - \zonoLower{\affineForm{}^\Delta} \right)$.
    For $\lambda=0$ the formula reduces to $0$ but for this to happen it must be the case that $\zonoUpper{\affineForm{}^\Delta} \leq 0$, i.e. we get a valid upper bound.
    For $\lambda=1$, it must be the case that $\zonoLower{\affineForm{}^\Delta}\geq 0$ which we excluded (Case 1).
    For the remaining case ($\lambda^\Delta = \frac{\zonoUpper{\affineForm{}^\Delta}}{\zonoUpper{\affineForm{}^\Delta}-\zonoLower{\affineForm{}^\Delta}}$), the formula reduces to $\zonoUpper{\affineForm{}^\Delta}$ which is a valid upper bound.\\[1em]
    Case 2: $\zonoLower{\affineForm{}^\Delta} \geq 0$\\
    In this case, we get that $\lambda^\Delta=1$.
    Thus, the formula reduces to $\left(x-y\right)$ which is obviously a valid upper bound.\\[1em]

    We now proceed to show that our Affine Form always bounds $0$.

    \paragraph{Lower Bound for 0.}
    Due to the prior proofs, we can assume that $\zonoUpper{\affineForm{}^\Delta}>0$ (otherwise, the Affine Form's lower bound is trivially a lower bound for $0$ due to the lower bound property for $(x-y)$).\\
    Case 1: $\zonoLower{\affineForm{}^\Delta}<0$\\
    In this case $\lambda^\Delta \notin \left\{0,1\right\}$ (i.e. $\lambda^\Delta = \frac{\zonoUpper{\affineForm{}^\Delta}}{\zonoUpper{\affineForm{}^\Delta}-\zonoLower{\affineForm{}^\Delta}}$) and the formula reduces to
    \[
    \underbrace{\lambda^\Delta \left(((x-y)-\zonoLower{\affineForm{}^\Delta}\right))}_{\leq \zonoUpper{\affineForm{}^\Delta}} - \underbrace{\max\left(-\zonoLower{\affineForm{}^\Delta},\zonoUpper{\affineForm{}^\Delta}\right)}_{\geq \zonoUpper{\affineForm{\Delta}}},
    \]
    which trivially bounds $0$ from below.\\[1em]

    Case 2: $\zonoLower{\affineForm{}^\Delta}\geq 0$\\
    In this case $\lambda=1$ and the formula reduces to
    \[
    \underbrace{(x-y)}_{\leq \zonoUpper{\affineForm{}^\Delta}} - \zonoUpper{\affineForm{}^\Delta} \leq 0.
    \]

    \paragraph{Upper Bound for 0.}
    Due to the prior proofs we can assume that $\zonoLower{\affineForm{}^\Delta}<0$ (otherwise, the Affine Form's upper bound is trivially an upper bound for 0 due to the upper bound property for $(x-y)$).
    In this case, the formula reduces to
    \[
    \underbrace{\lambda^\Delta}_{\in\left[0,1\right]} (
    \underbrace{(x-y)}_{\geq \zonoLower{\affineForm{}^\Delta}}-\zonoLower{\affineForm{}^\Delta}) \geq 0.
    \]
    
    \paragraph{Remaining cases.}
    We have now shown that our Affine Form always bounds $(x-y)$ and  $0$ from above and below.
    This obviously covers cases 1 and 4 in \Cref{eq:diff_relu:relu_cases}.
    Moreover, it covers case 2 as $0\leq x <(x-y)$.
    Concerning case 3, observe that $(x-y)-x \leq 0$ as $(x-y)\leq x$ and furthermore $(x-y)-x\geq (x-y)$ as $x < 0$.
    Thus, case 3 is also covered by the inequalities above.
    
\end{itemize}
\end{proofEnd}%
\noindent
\looseness=-1
Unfortunately, reasoning about the difference of two \reluSym{}s is not particularly intuitive.
The first 6 cases in \Cref{tab:relu_transform} follow from substituting Node 1 and Node 2 values.
In the other cases (\emph{Positive + Instable} and \emph{All Instable}), the difference is \emph{not} linear in the input or output Zonotopes.
Thus, we append an additional generator to the Differential Zonotope, i.e. to $\mathbf{a}^\Delta$.
The approximation for the case of two instable neurons is plotted in \Cref{fig:relu_delta_visualization}:
The bounding planes ensure that $0$ is always reachable and that the prior difference is within the bound.

\begin{table}[t]
    \centering
    \caption{Case Distinction for the construction of a new Affine Form $\textsc{ReLU}_\Delta\left(
    \affineForm{}^\Delta,
    \affineForm{}',
    \affineForm{}'',
    \hat{\affineForm{}}',
    \hat{\affineForm{}}''\right) = \hat{\affineForm{}}^\Delta=
    \left(\hat{\mathbf{e}}^\Delta,\hat{\mathbf{a}}'^\Delta,\hat{\mathbf{a}}''^\Delta,\hat{\mathbf{a}}^\Delta,\hat{c}^\Delta\right)$. Additions of vectors with different lengths signify addition for the components in the common-length prefix. We distinguish the nodes' different phases by positive ($+$), negative (--), and instable ($\sim$). The variables $\lambda$, $\mu$ and $\nu$ are initialized as follows:
    $\lambda' = \frac{-\zonoLower{\affineForm{}'}}{\zonoUpper{\affineForm{}'} - \zonoLower{\affineForm{}'}}$,
    $\mu' = 0.5*\lambda'\zonoUpper{\affineForm{}'}$ (accordingly for $\lambda'',\mu''$) and
    $\lambda^\Delta = \mathrm{clamp}\left(\frac{\zonoUpper{\affineForm{}^\Delta}}{\zonoUpper{\affineForm{}^\Delta}-\zonoLower{\affineForm{}^\Delta}},0,1\right)$,
    $\mu^\Delta = 0.5*\max\left(-\zonoLower{\affineForm{}^\Delta},\zonoUpper{\affineForm{}^\Delta}\right)$, $\nu^\Delta = \lambda^\Delta * \max\left(0,-\zonoLower{\affineForm{}^\Delta}\right)$}
    \begin{tabular}{c|c||c|c|c|c|c}
        \multicolumn{2}{c||}{\textbf{Neurons}} & \multicolumn{5}{c}{\textbf{New Affine Form} $\hat{\affineForm{}}^\Delta$}\\\hline
        $f_1$ & $f_2$& $\hat{\mathbf{e}}$ & $\hat{\mathbf{a}}'^\Delta$ & $\hat{\mathbf{a}}''^\Delta$ & $\hat{\mathbf{a}}^\Delta$ & $\hat{c}$\\\hline\hline
    
        \multicolumn{6}{l}{\textbf{Both nodes stable}}\\\hline
    
        -- & -- & 
        $0$ & $0$ & $0$ & $0$ & $0$\\\hline
    
        -- & + &
        $-\mathbf{e}''$ & $0$ & $-\mathbf{a}''$ & $0$ & $-c''$\\\hline
    
        + & -- &
        $\mathbf{e}'$ & $\mathbf{a}'$ & $0$ & $0$ & $c'$\\\hline
    
        + & + &
        $\mathbf{e}^\Delta$ & $\mathbf{a}'^\Delta$ & $\mathbf{a}''^\Delta$ & $\mathbf{a}^{\Delta}$ & $c^\Delta$\\\hline
    
        \multicolumn{6}{l}{\textbf{Negative + Instable}}\\\hline
    
        $\sim$ & -- &
        $\hat{\mathbf{e}}'$ & $\hat{\mathbf{a}}'$ & $0$ & $0$ & $\hat{c}'$\\\hline
    
        -- & $\sim$ &
        $-\hat{\mathbf{e}}''$ & $0$ & $-\hat{\mathbf{a}}''$ & $0$ & $-\hat{c}''$\\\hline
    
        \multicolumn{6}{l}{\textbf{Positive + Instable}}\\\hline
    
        $\sim$ & + &
        $\mathbf{e}^\Delta - \lambda' \mathbf{e}'$ & $\mathbf{a}'^\Delta - \lambda' \mathbf{a}'$ & $\mathbf{a}''^\Delta$ & $( (\mathbf{a}^{\Delta})^T \mid \mu' )^T$ & $c^\Delta - \lambda' c' + \mu'$\\\hline
    
        + & $\sim$ &
        $\mathbf{e}^\Delta + \lambda'' \mathbf{e}''$ & $\mathbf{a}'^\Delta$ & $\mathbf{a}''^\Delta + \lambda'' \mathbf{a}''$ & $( (\mathbf{a}^\Delta)^T \mid \mu'' )^T$ & $c^\Delta + \lambda'' c'' - \mu''$\\\hline
    
        \multicolumn{6}{l}{\textbf{All Instable}}\\\hline
        $\sim$ & $\sim$ &
        $\lambda^\Delta \mathbf{e}^\Delta$ & $\lambda^\Delta \mathbf{a}'^\Delta$ & $\lambda^\Delta \mathbf{a}''^\Delta$ & $\left(\lambda^\Delta (\mathbf{a}^{\Delta})^T \mid \mu^\Delta \right)^T$ & $\lambda^\Delta c^\Delta +\nu^\Delta - \mu^\Delta$
    \end{tabular}
    \label{tab:relu_transform}
\end{table}%

\section{Verification of Equivalence Properties}
\label{sec:top1}
\looseness=-1
To verify equivalence with $\textsc{Reach}_\Delta$, we proceed as follows:
We propagate a Zonotope $\zonotope{}_{\text{in}}$ through the two \acp{NN} $f_1,f_2$ as explained above.
Subsequently, we check a condition on the outputs $\left(\zonotope{}',\zonotope{}'',\zonotope{}^\Delta\right)$ that implies the desired equivalence property.
If this check fails, we refine the Zonotope by splitting the input space.
For the naive case, we compute $\zonotope{}^\Delta$ as explained in \Cref{sec:diff_zono}.

\paragraph{$\varepsilon$ equivalence.}
Checking $\varepsilon$ equivalence w.r.t. $\zonotope{}^\Delta$ works similarly to the approach for symbolic interval-based differential verification~\cite{paulsen_reludiff_2020,paulsen_neurodiff_2020}:
We compute  $\zonotope{}^\Delta$'s interval bounds and check for an absolute bound $> \varepsilon$.
If all bounds are smaller, we have proven $\varepsilon$ equivalence for $\langle \zonotope{}_{\text{in}} \rangle$ (see also \Cref{lem:soundness_epsilon_equiv} in \Cref{apx:proofs}).

\begin{lemmaE}[Soundness of $\varepsilon$ equivalence check][all end]
    \label{lem:soundness_epsilon_equiv}
    Let $\zonotope{}',\zonotope{}'',\zonotope{}^\Delta$ be $\textsc{Reach}_\Delta$'s result for two \acp{NN} $f_1,f_2$ and an input space described by $\zonotope{}_{\text{in}}$.
    If for all $1 \leq i \leq O$ it holds that $\max\left(\left|\zonoLower{\left(\zonotope{}^\Delta\right)_i}\right|,\left|\zonoUpper{\left(\zonotope{}^\Delta\right)_i}\right|\right) \leq \varepsilon$ then ${\forall \mathbf{x} \in \langle Z_{\text{in}}\rangle ~\left(f_1\left(\mathbf{x}\right)-f_2\left(\mathbf{x}\right)\right) \leq \varepsilon}$.
\end{lemmaE}
\begin{proofEnd}
This result follows directly from \Cref{thm:verydiff_verysound}:
$\zonotope{}^\Delta$ bounds the difference between $f_1$ and $f_2$.
Consequently, if $\zonotope{}^\Delta$'s bounds are smaller than $\varepsilon$, then so is the difference between $f_1$ and $f_2$.
\end{proofEnd}

\paragraph{Top-1 Equivalence.}
\looseness=-1
Since Top-1 equivalence considers the order of outputs, it cannot be read off from $\zonotope{}^\Delta$'s bounds directly.
We frame the property as an LP optimization problem in \Cref{def:top1_lp}.
The LP has an optimal solution $\leq0$ if no $\mathbf{x} \in \langle\zonotope{}_{\text{in}}\rangle$ classified as $k$ in $f_1$ is classified as $j$ in $f_2$.
Formally, the LP computes an upper bound for $\left(f_2\left(\mathbf{x}\right)\right)_j-\left(f_2\left(\mathbf{x}\right)\right)_k$ (maximized expression) under the condition that
$\left(f_1\left(\mathbf{x}\right)\right)_k$ is the maximum of $f_1\left(\mathbf{x}\right)$ (first constraint) and under the additional condition that $\zonotope{}^\Delta$ bounds the difference between $f_1$ and $f_2$, resp. the difference between the reachable points in $\zonotope{}'$ and $\zonotope{}''$ (second constraint; see \Cref{def:zono_lp_condition} in \Cref{apx:proofs}).
The first constraint ensures a gap $t$ between the largest and second largest output of $f_1$.
In this section, we only consider $t=0$.
\begin{definition}[Top-1 Violation LP]
\label{def:top1_lp}
Given $\zonotope{}'=\left(G',\mathbf{c}'\right)=\left(E',A',\mathbf{c}'\right),\zonotope{}''=\left(E'',A'',\mathbf{c}''\right),\zonotope{}^\Delta=\left(G^\Delta,\mathbf{c}^\Delta\right)$, a constant $t \geq 0$ and $k,j\in\left[1,O\right]$ with $k \neq j$ the \emph{Top-1 Violation LP} is defined below.
$E'$ and $E''$ have $n_1$ columns, $A'$ has $n_2$ columns, $A''$ has $n_3$ columns and $A^\Delta$ has $n_4$ columns. $\mathbf{x}$ contains generators for these matrices in order and has dimension $\bar{n}=n_1+n_2+n_3+n_4$.
\begin{align*}
    \max_{\mathbf{x}\in\left[-1,1\right]^{\overline{n}}} &
    \left(E'' \mathbf{x}_{1:n_1} + A'' \mathbf{x}_{(n_1+n_2+1):(n_1+n_2+n_3)} + \mathbf{c}\right)_j\\
    & - \left(E'' \mathbf{x}_{1:n_1} + A'' \mathbf{x}_{(n_1+n_2+1):(n_1+n_2+n_3)} + \mathbf{c}\right)_k\\
    \text{s.t. }&
    \left(G' \mathbf{x}_{1:(n_1+n_2)} + \mathbf{c}'\right)_l + t \leq \sum_{i=1}^{n_1+n_2} \left(G'\right)_{k,i} x_i + \left(\mathbf{c}'\right)_k \text{ for }l \neq k\\
    &\zonotope{}' = \zonotope{}'' + \zonotope{}^\Delta
\end{align*}
\end{definition}
\begin{definitionE}[LP condition][all end]
\label{def:zono_lp_condition}
Using the notation from \Cref{def:top1_lp}, we define $\zonotope{}' = \zonotope{}'' + \zonotope{}^\Delta$ as follows:
\begin{align*}
&G' \mathbf{x}_{1:(n_1+n_2)} + \mathbf{c}' \\
=~&
E'' \mathbf{x}_{1:n_1} + A'' \mathbf{x}_{(n_1+n_2+1):(n_1+n_2+n_3)} + \mathbf{c}''
+
G^\Delta \mathbf{x} + \mathbf{c}^\Delta
\end{align*}
\end{definitionE}
\noindent
\looseness=-1
If the LP's maximum is positive, we check whether the generated counterexample is spurious.
Verification requires $\mathcal{O}\left(O^2\right)$ LP optimizations.
In practice, we reuse the same constraint formulation for each $k$ and optimize over all possible $j \neq k$ admitting warm starts.
Our approach is sound (see proof on \cpageref{proof:lem_top1_sound}):
\begin{lemmaE}[Soundness for Top-1][end,restate,text link={}]
\label{lem:soundness_top1}
\looseness=-1
Consider $\zonotope{}',\zonotope{}'',\zonotope{}^\Delta$ provided by $\textsc{Reach}_\Delta$ w.r.t. $Z_{\text{in}}$:
If for all $k,j \in \left[1,O\right]$ ($k \neq j$) the
Top-1 Violation LP with $t=0$ has a maximum $\leq 0$, then $f_1,f_2$ satisfy Top-1 equivalence w.r.t. inputs in $\langle Z_{\text{in}}\rangle$.
\end{lemmaE}
\begin{proofEnd}
\label{proof:lem_top1_sound}
First, observe that via \Cref{thm:verydiff_verysound} and \Cref{cor:verydiff_tight_bounds} the Zonotopes $\zonotope{}'$, $\zonotope{}''$ overapproximate the behavior of $f_1$ and $f_2$.
I.e., for any $\mathbf{v} \in \mathbb{R}^n$ we know that $f_1\left(\zonotope{}_{\text{in}}\left(\mathbf{v}\right)\right) \in \langle \zonotope{}'\left(\mathbf{v}\right) \rangle$, $f_2\left(\zonotope{}_{\text{in}}\left(\mathbf{v}\right)\right) \in \langle \zonotope{}''\left(\mathbf{v}\right) \rangle$.
and
$f_1\left(\zonotope{}_{\text{in}}\left(\mathbf{v}\right)\right)-f_2\left(\zonotope{}_{\text{in}}\left(\mathbf{v}\right)\right) \in \langle \zonotope{}^\Delta\left(\mathbf{v}\right) \rangle$ \emph{with the same generators values} $\mathbf{\epsilon}_1,\mathbf{\epsilon}_2,\mathbf{\epsilon}_3,\mathbf{\epsilon}_4$.
Given 
$\zonotope{}'=\left(G',\mathbf{c}'\right)$ and $\zonotope{}^\Delta = \left(G^\Delta,\mathbf{c}^\Delta\right)$, the points reachable via $f_2$ are \emph{also} described by the Zonotope $\tilde{\zonotope{}}'' = \left( G' - G^\Delta, \mathbf{c}' - \mathbf{c}^\Delta \right)$ (where $G'$ is appropriately extended with $0$ columns).
Thus, any value actually reachable via $f_1$ and $f_2$ satisfies the constraint $\zonotope{}' = \zonotope{}'' + \zonotope{}^\Delta$.
The constraints $G' \mathbf{v} + \mathbf{c} + t \leq \sum_{i=1}^n \left(G'\right)_{k,i} \mathbf{v}_i + \left(\mathbf{c}'\right)_k$ overapproximate the set of $\mathbf{v}\in \left[-1,1\right]^n$ for which $\left(f_1\left(\zonotope{}_{\text{in}}\left(\mathbf{v}\right)\right)\right)_k$ is the maximum with distance of $t$ to the second largest value (here $t=0$).
We now consider these $\mathbf{v}$ for an arbitrary, but fixed $k$:
If for every $j \neq k$ it holds that $\left(f_2\left(Z_{\text{in}}\left(\mathbf{v}\right)\right)\right)_k \geq \left(f_2\left(Z_{\text{in}}\left(\mathbf{v}\right)\right)\right)_j$ then these $\mathbf{x}=\zonotope{}_{\text{in}}\left(\mathbf{v}\right)$ satisfy Top-1 equivalence.
This is equivalent to the statement that 
$0 \geq \left(f_2\left(\zonotope{}_{\text{in}}\left(\mathbf{v}\right)\right)\right)_j - \left(f_2\left(Z_{\text{in}}\left(\mathbf{v}\right)\right)\right)_k$.
As $\zonotope{}''$ overapproximates the behavior of $f_2$ it is thus sufficient to bound its upper bound for all $j \neq k$ (w.r.t. to the other constraints explained above).
Thus, if our optimization problems all return a value less than or equal to zero, this implies that all $x \in \langle \zonotope{}_{\text{in}}\left(\mathbf{v}\right)\rangle$ satisfy Top-1 equivalence.

Given that every $\mathbf{v}$ must have some $k$ such that $\left(f_1\left(\mathcal{Z}_{\text{in}}\left(\mathbf{v}\right)\right)\right)_k$ is the maximum and we consider all such $k$, all $\mathbf{v}$ satisfy Top-1 equivalence.
Therefore, all values from $\zonotope{}_{\text{in}}$ satisfy Top-1 equivalence.
\end{proofEnd}
\looseness=-1

\paragraph{Input Space Refinement.}
\looseness=-1
If verification fails, we split the input space in half and solve the verification problems separately.
To this end, we use a heuristic to estimate the influence of splits along different input dimensions
Splitting can improve the bounds in two ways:
Either the reduced input range directly reduces the computed output bounds, or the reduced range reduces the number of instable neurons and hence reduces the over-approximation error w.r.t. output bounds.
Our heuristic works similar to forward-mode gradient computation estimating the influence of input dimensions on output bounds.
For an analysis with $n$ generators in $\zonotope{}_{\text{in}}$ and $m$ generators in $\zonotope{}'$/$\zonotope{}''$
our heuristic requires two matrices ($n \times m$) and two additional matrix multiplication per layer.
For details see \Cref{apx:input_refinement}; we leave a fine-grained analysis of refinement strategies to future work.

\paragraph{Generator Compression}
To increase performance we analyze $Z_{\text{in}}$.
In case an input dimension has range $0$, we eliminate the generator in $Z_{\text{in}}$.
This optimization speeds up equivalence verification for, e.g., targeted pixel perturbations~\cite{paulsen_reludiff_2020}. %

\paragraph{Completeness.}
\looseness=-1
As we only employ axis-aligned input-splitting, we cannot provide a completeness guarantee.
However, our evaluation (\Cref{sec:evaluation}) demonstrates, that this approach outperforms complete State-of-the-Art solvers.

\section{Equivalence Verification for Classification \acp{NN}}
\label{sec:classification}
\looseness=-1
Top-1 equivalence is particularly useful when verifying the equivalence of classification \acp{NN}.
Indeed, there are examples of classification \acp{NN} which are $\varepsilon$-equivalent, but not Top-1 equivalent w.r.t. some input region (see also \Cref{apx:subsec:eval:epsilon}).
This underlines the importance of choosing the right equivalence property.
Unfortunately, as we empirically show in \Cref{apx:subsec:eval:abelation}, classic Top-1 equivalence does not benefit from Differential Verification.
Moreover, prior work on equivalence verification only provides guarantees for small parts of the input space, e.g. by proving Top-1 equivalence for $\epsilon$-balls around given data points.
As pointed out in orthogonal work~\cite{DBLP:conf/icml/GengLXWGS23}, $\epsilon$-balls around data points are not necessarily a semantically useful specification.
Moreover, proving Top-1 equivalence on large parts of the input space is typically impossible, because pruning \acp{NN} invariably \emph{will} change their behavior.
This raises two questions:
\begin{enumerate*}
    \item Why does Top-1 equivalence not benefit from Differential Verification?
    \item What equivalence property for classification \acp{NN} is verifiable on large parts of the input space while it can benefit from Differential Verification?
\end{enumerate*}
We will answer these questions in order.

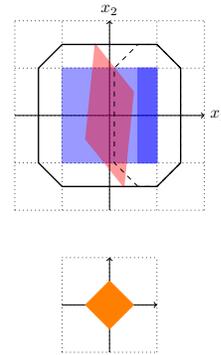
\begin{wrapfigure}[15]{r}{0.25\textwidth}
    \centering
    \vspace*{-2.5em}
    \resizebox{\linewidth}{!}{
    \begin{tikzpicture}
      \draw[thin,dotted] (-2,-2) grid (2,2);
      \draw[->] (-2,0) -- node[pos=1.0,anchor=west] {$x_1$} (2,0);
      \draw[->] (0,-2) -- node[pos=1.0,anchor=south] {$x_2$} (0,2);
      \draw[thick, blue, fill=blue, opacity=0.4] (-1,-1) -- (-1,1)
        -- (1,1) -- (1,-1) -- cycle;
      \draw[thick, blue, fill=blue, opacity=0.4] (0.6,-1) -- (0.6,1)
        -- (1,1) -- (1,-1) -- cycle;
      \draw[thick, red, fill=red, opacity=0.4] (-0.5,-0.5) -- (-0.3,1.5)
        -- (0.5,0.5) -- (0.3,-1.5) -- cycle;

      \draw[thick, black] (-1,-1.5) -- (-1.5,-1) -- (-1.5,1) -- (-1,1.5)
        -- (1,1.5) -- (1.5,1) -- (1.5,-1) -- (1,-1.5) -- cycle;

    \draw[dashed, black] (0.6,-1.5) -- (0.1,-1) -- (0.1,1) -- (0.6,1.5)
        -- (1,1.5) -- (1.5,1) -- (1.5,-1) -- (1,-1.5) -- cycle;

      \draw[thin,dotted] (-1,-5) grid (1,-3);
      \draw[->] (-1,-4) -- (1,-4);
      \draw[->] (0,-5) -- (0,-3);
      \draw[thick, orange, fill=orange, opacity=1.0] (-0.5,-4) -- (0,-3.5)
        -- (0.5,-4) -- (0,-4.5) -- cycle;
    \end{tikzpicture}
    }
    \caption{Differential Zonotopes for Top-1 equivalence}
    \label{fig:top_1_diff_zono}
\end{wrapfigure}
\paragraph{Ineffectiveness for Top-1 equivalence.}
\label{apx:top_1_tight_diff}
\looseness=-1
Our initial intuition would have been that the tighter bounds in $\zonotope{}^\Delta$ should also aid the verification of Top-1 equivalence.
To refute this intuition, consider the sketch in \Cref{fig:top_1_diff_zono}:
The light blue area represents a Zonotope for the output space reachable via $f_1$ (i.e. $\zonotope{}'$) and the red area describes a Zonotope for the output space reachable via $f_2$ (i.e. $\zonotope{}''$).
The orange Zonotope describes $\zonotope{}^\Delta$, i.e. it is a bound for the difference between the two outputs.
Depending on the nature of $\zonotope{}'$'s and $\zonotope{}^\Delta$'s generators (i.e. if they are shared or not), the output region for $\zonotope{}'-\zonotope{}^\Delta$ could reach as far as the solid black line.
Thus, while $\zonotope{}^\Delta$ limits the difference for individual input points, it does not necessarily provide effective bounds for the reachable values of $f_2$ resp. $\zonotope{}''$.
However, if via an LP formulation, the reachable values from $\zonotope{}'$ were restrained to $x_1 \geq 0.6$ (dark blue area on the right), then adding $\zonotope{}^\Delta$ would yield the black dashed region meaningfully constrains the behavior of $f_2$ beyond the constraints by $\zonotope{}''$ (red area).
We could then prove that $f_2$'s first dimension ($x_1$) is positive.
Notably, our constraint on $f_1$ ($\geq 0.6$) is \emph{stricter} than our constraint on $f_2$ ($\geq 0$) making the Differential Zonotope useful even if it has reachable values in negative $x_1$ direction.
Top-1 equivalence imposes \emph{equal} constraints on the difference between two dimensions in both \acp{NN}.
Thus, if the difference has a negative bound in $\zonotope{}^\Delta$ (a likely outcome), Differential Verification, independently of the considered abstract domain, \emph{cannot} help for Top-1 equivalence: While the concrete regions would look different for other abstract domains, the outcome would be the same given a negative differential bound for $x_1$.

\subsubsection{Confidence-Based Equivalence}
\looseness=-1
We propose the notion of $\delta$-Top-1 equivalence.
Our key idea is to integrate the confidence values (computed by $\mathrm{softmax}$) into the verified property.
In contrast to the arbitrary threshold $x_1 \geq 0.6$, constraints based on confidence values have intuitive meaning:
\begin{textAtEnd}%
\subsection{Proofs for Confidence Based Verification}
\end{textAtEnd}%
\begin{definition}[$\delta$-Top-1 equivalence]%
\label{def:delta_top_1_equivalence}
Given two \acp{NN} $f_1,f_2 : \mathbb{R}^I \to \mathbb{R}^O$, $\delta \in \left[\frac{1}{2},1\right]$ and an input region $Y \subseteq \mathbb{R}^I$,
$f_2$ is \emph{$\delta$-Top-1 equivalent} w.r.t. $f_1$ iff $f_1,f_2$ are Top-1 equivalent w.r.t. 
$X_{f_1,Y}\left(\delta\right) = \left\{ \mathbf{x} \in Y~\middle|~\exists i \in \left[1,O\right]\enspace \mathrm{sofmax}_i\left(f_1\left(\mathbf{x}\right)\right) \geq \delta\right\}$.
\end{definition}%
\looseness=-1

\begin{wrapfigure}[13]{r}{0.35\textwidth}
\centering
\includegraphics[width=\linewidth]{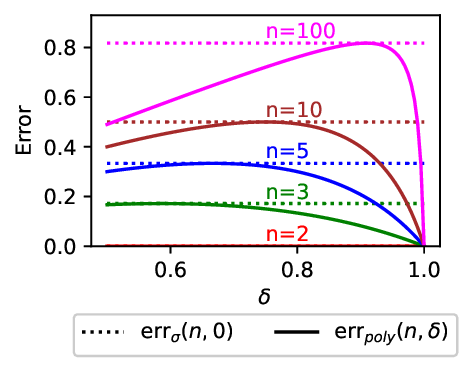}
\caption{Approximation Errors w.r.t. confidence $\delta$: Ours (solid) and Athavale \emph{et al.}~\cite{DBLP:conf/cav/AthavaleBCMNW24} (dashed).}
\label{fig:maximal_error}
\end{wrapfigure}
\noindent
\looseness=-1
We assume that $f_1,f_2$ are $\reluSym{}$ \acp{NN} with one $\mathrm{softmax}$ function after the last layer.
Verifying $\delta$-Top-1 equivalence achieves multiple objectives:
First, even for larger $Y$ (e.g. intervals over standard deviations for normalized inputs), we may provide meaningful guarantees for a suitable $\delta$.
Secondly, we rely on confidence estimates of a component we already trust: The reference \ac{NN}.
Finally, from a technical perspective, constraining the confidence level of $f_1$ to $\geq \delta$ while ``only'' requiring the same classification in $f_2$ achieves the asymmetry
necessary for exploiting Differential Verification.
Unfortunately, deciding $\delta$-Top-1 equivalence is a coNP-hard decision problem (proof on \cpageref{proof:delta_top_1_np}):
\begin{corollaryE}[Complexity of $\delta$-Top-1][end,restate,text link=]
Let $Y \subseteq \mathbb{R}^I$ be a polytope, $f_1,f_2$ be two $\reluSym{}$-$\mathrm{softmax}$-\acp{NN}, $\frac{1}{2} < \delta \leq 1$.
Deciding whether there exists a $\mathbf{y} \in Y$ and $k \in \left[1,O\right]$ s.t. $\left(f_1\left(x\right)\right)_k \geq \delta$ but $\exists j\in\left[1,O\right]~\left(f_2\left(\mathbf{x}\right)\right)_k < \left(f_2\left(\mathbf{x}\right)\right)_j$ is NP-hard.
\end{corollaryE}
\begin{proofEnd}
\label{proof:delta_top_1_np}
To prove NP-hardness we show a reduction from another NP-complete problem, specifically once again \textsc{StrictNetVerify}.
The reduction can be performed in a similar manner as for \Cref{thm:top_1_comp} with a minor adjustment:
The second \ac{NN} is constructed identically as before.
For the first \ac{NN} we choose an arbitrary $\frac{1}{2} < \delta < 1$ and set the outputs $y_{11},y_{22}$ of the first \ac{NN} to:
\begin{align*}
y_{11} &= \left(f\left(x\right)\right)_1 + \ln\left( \frac{\delta}{1-\delta}\right)
&y_{12} &=
\left(f\left(x\right)\right)_1\\
\intertext{Applying $\mathrm{softmax}$ to the outputs of $f_1$ then yields:}\\
\mathrm{softmax}\left(y_{11},y_{12}\right)_1 &=
\rlap{$
\frac{e^{y_{11}}}{e^{y_{11}} + e^{y_{1_2}}} =
\frac{
e^{\left(f\left(x\right)\right)_1}
e^{\ln\left( \frac{\delta}{1-\delta}\right)}
}{
e^{\left(f\left(x\right)\right)_1}
\left(
e^{\ln\left( \frac{\delta}{1-\delta}\right)}
+
1
\right)
}
$}\\
&= \frac{
\frac{\delta}{1-\delta}
}{
\frac{\delta}{1-\delta}
+1
} = 
\frac{\delta}{
\delta + \left(1-\delta\right)
} = \delta
\end{align*}
Consequently, we have ensured that for all inputs the $\mathrm{softmax}$ output of $f_1$ returns a probability $\geq \delta$ for output 1.
With this setup, there is a violation to $\delta$-Top-1 equivalence to $f_1$ iff the corresponding input satisfies the constraints of the \textsc{StrictNetVerify} instance.
Constraints on number representation may hinder the usage of the exact $\mathrm{softmax}$ function and the exact value of $\ln\left(\frac{\delta}{1-\delta}\right)$, however choosing an appropriate $\delta$ and choosing a larger (representable) number than the result of $\ln$ can mitigate this issue.

Unfortunately, the prior NP membership argument for Top-1 equivalence is not applicable to $\delta$-Top-1 equivalence as the $\mathrm{softmax}$ function cannot be represented in the LP problem.
\end{proofEnd}
\noindent
Due to the usage of $\mathrm{softmax}$, the previous NP-membership argument does not apply in this case.
To verify $\delta$-Top-1 equivalence, we part from prior work on confidence-based verification~\cite{DBLP:conf/cav/AthavaleBCMNW24} and propose the following approximation of all vectors $\mathbf{z}$ for which output $i$ has a confidence $\geq \delta$ (see proof on \cpageref{proof:lin_approx_softmax}):
\begin{textAtEnd}
\subsubsection[Approximation of softmax]{Approximation of $\mathrm{softmax}$.}
Prior work on confidence-based verification~\cite{DBLP:conf/cav/AthavaleBCMNW24} used a complex approximation procedure for the $\mathrm{softmax}$ function which relied on a reformulation of $\mathrm{softmax}$ using $\mathrm{sigmoid}$ and subsequent approximation of $\mathrm{sigmoid}$ using 35 linear segments.
Importantly, the prior approximation was independent of the considered probability threshold $\delta$.
\begin{definition}[Maximal Error for $\mathrm{softmax}$ approximation from Athavale \emph{et al.}~\cite{DBLP:conf/cav/AthavaleBCMNW24}]
\label{def:max_error_athavale}
Let $\widehat{\mathrm{softmax}}$ be the $\mathrm{softmax}$ approximation by Athavale \emph{et al.}~\cite{DBLP:conf/cav/AthavaleBCMNW24}.
For any vector $\mathbf{z}$ of dimension $n \geq 2$ and $\mathbf{z}_i = \max_{j=0}^{n} \mathbf{z}_j$ and approximation precision $\upsilon$ for $\mathrm{sigmoid}$ we get that:
\[
\mathrm{softmax}\left(\mathbf{z}\right)_i - \widehat{\mathrm{softmax}}\left(\mathbf{z}\right)_i
\leq
\errorSoftmaxSigmoid{n}{\upsilon} \coloneqq
\frac{n-2}{\left(\sqrt{n-1}+1\right)^2} + 2\upsilon
\]
\end{definition}
\noindent
Note, that while $\widehat{\mathrm{softmax}}$ is an approximation with bounded error, it is not guaranteed to be a strict lower/upper bound.
To consider all input vectors $\mathbf{z}$ s.t. $\mathrm{softmax}\left(\mathbf{z}\right)_i \geq \delta$, Athavale \emph{et al.}~\cite{DBLP:conf/cav/AthavaleBCMNW24} then used the (piece-wise linear) constraint $\widehat{\mathrm{softmax}}\left(\mathbf{z}\right)_i \geq \delta-\errorSoftmaxSigmoid{n}{\upsilon}$.
While the first part of the approximation by Athavale \emph{et al.}~\cite{DBLP:conf/cav/AthavaleBCMNW24} provides an exact lower-bound, this lower bound is then approximated using piece-wise linear constraints in order to remove a sigmoid function.
Hence, to recover soundness, their approach shifts the confidence threshold down.
It is our understanding that this downshift mixes two kinds of errors that could be treated separately for a more precise analysis (maximal error of the sound lower-bound and maximal error of the unsound piece-wise linear approximation).
However, even if we were to account for this distinction, our approach would yield a more precise approximation as it is parameterized in $\delta$ and does not rely on a sigmoid linearization.
\end{textAtEnd}
\begin{lemmaE}[Linear approximation of $\mathrm{softmax}$][end,restate,text link={}]
\label{lem:lin_approx_softmax}
For $\delta \in \left[1/2,1\right)$ the following set relationship holds:
\[
\left\{
\mathbf{z} \in \mathbb{R}^n ~\middle|~
\mathrm{softmax}\left(\mathbf{z}\right)_i \geq \delta
\right\}
\subseteq
\left\{
\mathbf{z} \in \mathbb{R}^n ~\middle|~
\bigwedge_{\substack{j=1\\j\neq i}}^n \mathbf{z}_i-\mathbf{z}_j \geq \ln\left(\frac{\delta}{1-\delta}\right)
\right\} \eqqcolon P_n\left(\delta\right)
\]
\end{lemmaE}
\begin{proofEnd}
\label{proof:lin_approx_softmax}
Consider some $\mathbf{z}$ such that $\mathrm{softmax}\left(\mathbf{z}\right)_i \geq \delta$.
This implies $\mathbf{z}$ is contained in the set on the left-hand-side.
Then we know in particular that $\frac{e^{\mathbf{z}_i}}{\sum_{j=1}^n e^{\mathbf{z}_j}} \geq \delta$.
Moving all $e^{\mathbf{z}_i}$ to the left this is equivalent to 
\[
\left(1-\delta\right) e^{\mathbf{z}_i} \geq \delta \sum_{\substack{j=1\\j\neq i}}^n \underbrace{e^{\mathbf{z}_j}}_{\geq 0}.
\]
This in turn implies that for all $j \neq i$ we have
$
\left(1-\delta\right) e^{\mathbf{z}_i} \geq
\delta e^{\mathbf{z}_j}
$
which is equivalent to
$
\mathbf{z}_i-\mathbf{z}_j \geq \ln\left(\frac{\delta}{1-\delta}\right).
$
Consequently, $\mathbf{z}$ is contained in the set on the right-hand-side.
\end{proofEnd}
\noindent
We can then prove $\delta$-Top-1 equivalence for $\delta \geq \frac{1}{2}$ by reusing the Top-1 Violation LP with $t$ chosen appropriately (see proof on \cpageref{proof:soundness_delta_top1}):
\begin{corollaryE}[Soundness for $\delta$-Top-1][end,restate,text link={}]
\looseness=-1
Given $\zonotope{}',\zonotope{}'',\zonotope{}^\Delta$ from $\textsc{Reach}_\Delta$ w.r.t. $Z_{\text{in}}$, $\frac{1}{2} \leq \delta < 1$.
If for all $k,j \in \left[1,O\right]$ ($k \neq j$) the
Top-1 Violation LPs with $t=\ln\left(\frac{\delta}{1-\delta}\right)$ have maxima $\leq 0$, then $f_2$ is $\delta$-Top-1 equivalence w.r.t. $f_1$ on $\langle Z_{\text{in}}\rangle$.
\end{corollaryE}
\begin{proofEnd}
\label{proof:soundness_delta_top1}
The implication follows from the soundness argument of \Cref{lem:soundness_top1}.
The key insight for this proof is the observation that it suffices to prove Top-1 equivalence for an overapproximation of the set of inputs for which $f_1$ has confidence $\geq \delta$.
Contrary to \Cref{lem:soundness_top1}, $t$ is now larger than $0$.
Considering that $\zonotope{}'$ overapproximates the reachable values of $f_1$, the chosen $t=\ln\left(\frac{\delta}{1-\delta}\right)$ constrains the set of $f_1$'s outputs exactly to the right-hand-set of \Cref{lem:lin_approx_softmax}.
Thus, we in particular prove a property for all inputs where the $k$-th $\mathrm{softmax}$ component of $f_1$'s output has a value larger than $\delta$ (left-hand-side of \Cref{lem:lin_approx_softmax}).
By proving Top-1 equivalence for all inputs such that $f_1$'s output has confidence $\geq \delta$, we thus prove $\delta$-Top-1 equivalence.
\end{proofEnd}

\noindent
\looseness=-1
In comparison to the $\mathrm{softmax}$ approximation by Athavale \emph{et al.}~\cite{DBLP:conf/cav/AthavaleBCMNW24}, we approximate via \emph{one} polytope in the output space (in contrast to a 35-segment piece-wise linear approximation).
Additionally, our approximation is parametrized in the confidence threshold $\delta$ while their approximation is uniform across confidence values.
We now analyze the precision of the two approximations.
Given a desired confidence level $\delta$, we consider its error the maximal deviation below $\delta$ still encompassed by the approximation.
For Athavale \emph{et al.}~\cite{DBLP:conf/cav/AthavaleBCMNW24} this maximal error is given as a function $\errorSoftmaxSigmoid{n}{\upsilon}$ in the input dimension $n$ and the sigmoid approximation error $\upsilon$ (see \Cref{def:max_error_athavale} or \cite[Thm. 1]{DBLP:conf/cav/AthavaleBCMNW24}).
For our approximation, we want to derive the error via the minimal confidence value that is still part of $P_n\left(\delta\right)$, i.e. as $\errorSoftmaxPolytope{n}{\delta} = \delta - \min \left\{
\max_i \mathrm{softmax}\left(\mathbf{z}\right)_i ~\middle|~
\mathbf{z}\in P_n\left(\delta\right)
\right\}$.
We can derive the following properties for our approximation error $\errorSoftmaxPolytope{n}{\delta}$ in relation to the approximation error $\errorSoftmaxSigmoid{n}{\upsilon}$ incurred by Athavale \emph{et al.}~\cite{DBLP:conf/cav/AthavaleBCMNW24} (proof on \cpageref{proof:approx_error}):
\begin{lemmaE}[Maximal Error for our $\mathrm{softmax}$ approximation][end,restate,text link=]
\label{lem:maximal_error_softmax}
Consider $n\geq 2$ and $\delta \geq \frac{1}{2}$, then:
\begin{enumerate}
    \item $\errorSoftmaxPolytope{n}{\delta} = \delta\enspace -\enspace\delta / \left(\delta\left(2-n\right)+n-1\right)$ and $\errorSoftmaxPolytope{2}{\delta} =0$
    \item For all $\upsilon>0$ we get $\errorSoftmaxSigmoid{n}{\upsilon} > \errorSoftmaxSigmoid{n}{0} = \max_{\delta\in\left[\frac{1}{2},1\right]} \errorSoftmaxPolytope{n}{\delta}$
    \item  $\lim_{\delta \to 1} \errorSoftmaxPolytope{n}{\delta} = 0$
\end{enumerate}
\end{lemmaE}
\begin{proofEnd}
\label{proof:approx_error}
We prove the statements in order:
\paragraph{Proof of 1.}
W.l.o.g. assume that $\mathbf{z}_1$ is the maximum over all components of $\mathbf{z}$ (otherwise, observing that the order of components is irrelevant to $\mathrm{softmax}$, reorder).
The objective is now to find the $\mathbf{z}$ with minimal value $\mathbf{z}_1$ contained in $P_n\left(\delta\right)$.
This corresponds to the following minimization problem:
\begin{align*}
    \min_{\mathbf{z}}~~ & \frac{e^{\mathbf{z}_1}}{\sum_{i=1}^n e^{\mathbf{z}_i}}
    & \text{s.t. }
    & \mathbf{z}_1 \geq \ln\left(\frac{\delta}{1-\delta}\right) + \mathbf{z}_i \enspace\left(\text{for all } i\in\left[2,O\right]\right)
\end{align*}
Note that for any value $\mathbf{z}_1$ the $\mathrm{softmax}$ expression on the left becomes minimal for large values in the faction's denominator.
Since the exponential function is positive and strictly increasing, our objective is thus to maximize all values $\mathbf{z}_j$ for $j \neq 1$.
Considering the bounds on the left, we therefore set their value to $\mathbf{z}_j = \mathbf{z}_1 - \ln\left(\frac{\delta}{1-\delta}\right)$.
Placing these values into the $\mathrm{softmax}$ function reduces the expression on the left-hand-side to:
\begin{align*}
\frac{
\exp{\mathbf{z}_1}
}{
\exp{\mathbf{z}_1} + 
\exp\left(\mathbf{z}_1 - \ln\left(\frac{\delta}{1-\delta}\right)\right)\left(n-1\right)
}
&= 
\frac{\exp{\mathbf{z}_1}}{
\exp{\mathbf{z}_1}\left(
1 + \frac{\left(n-1\right)\left(1-\delta\right)}{\delta}
\right)
}\\
&=
\frac{\delta}{\delta\left(2-n\right)+n-1}
\end{align*}
Subtracting this value from $\delta$ then yields the maximal error reachable via using our approximation.
Substituting $n$ by $2$ shows that our approximation is exact for $n=2$.

\paragraph{Proof of 2.}
The strict inequality $\errorSoftmaxSigmoid{n}{\upsilon} > \errorSoftmaxSigmoid{n}{0}$ follows directly from the definition because $\errorSoftmaxSigmoid{n}{\upsilon} = \errorSoftmaxSigmoid{n}{0} + 2\delta$.
For the identity with our maximal achievable error, i.e. $\max_{\delta\in\left[\frac{1}{2},1\right]} \errorSoftmaxPolytope{n}{\delta}$, we compute the partial derivative of $\errorSoftmaxPolytope{n}{\delta}$ w.r.t. $\delta$ which yields $ 1 +\frac{1-n}{\left(\delta\left(2-n\right)+n-1\right)^2}$.
For $n>2$ we find maximal values at $\frac{n \pm \sqrt{n-1}-1}{n-2}$ of which only $\delta^*=\frac{n - \sqrt{n-1}-1}{n-2}$ is smaller than 1.
Computing the value of $\errorSoftmaxPolytope{n}{\delta^*}$ yields 
$
\frac{\left(\sqrt{n-1}-2\right)n+2}{\left(n-2\right)\sqrt{n-1}}
$
as the maximal error in dependence of $n$.
Additional reformulations yield that this formula is equal to $\errorSoftmaxSigmoid{n}{0}=\frac{n-2}{\left(\sqrt{n-1}+1\right)^2}$.

\paragraph{Proof of 3.}
To prove that $\lim_{\delta \to 1} \errorSoftmaxPolytope{n}{\delta} = 0$, it suffices to show that $\lim_{\delta \to 1} \frac{\delta}{\delta\left(2-n\right)+n-1} = 1$.
Since $\lim_{\delta \to 1} \delta\left(2-n\right)+n-1 = 1$ and $\lim_{\delta \to 1} \delta = 1$ we get the limit via the quotient rule.
\end{proofEnd}
\noindent
Note, that while $\errorSoftmaxPolytope{n}{1}$ is well defined for $\delta$ the necessary bound $\ln\left(\frac{\delta}{1-\delta}\right)$ is not, i.e. we can only check for $\delta<1$.
The observations described in \Cref{lem:maximal_error_softmax} are also observable in \Cref{fig:maximal_error}: By parameterizing the approximation in the confidence threshold $\delta$ we achieve significant precision gains over prior work -- independent of output dimensionality ($n$) and in particular as we approach $\delta=1$.

\section{Evaluation}
\label{sec:evaluation}
\looseness=-1
We implemented Differential Zonotope verification in a new tool\footnote{\url{https://figshare.com/s/35fdc787de2872b59e9d}} called \emph{VeryDiff} in Julia~\cite{bezanson2017julia}. %
First, we analyze the efficiency of Differential Zonotopes compared to our naive approach for verifying $\varepsilon$ or ($\delta$-)Top-1 equivalence.
We also compare the performance of VeryDiff to the previous State-of-the-Art and demonstrate significant performance improvements for $\varepsilon$ and $\delta$-Top-1 equivalence across all benchmark families.
\Cref{apx:additional_experimental_results} contains an extended evaluation.

\paragraph{Experimental Setup.}
\begin{wrapfigure}[10]{r}{0.5\textwidth}
    \centering
    \vspace*{-2em}
    \includegraphics[width=\linewidth]{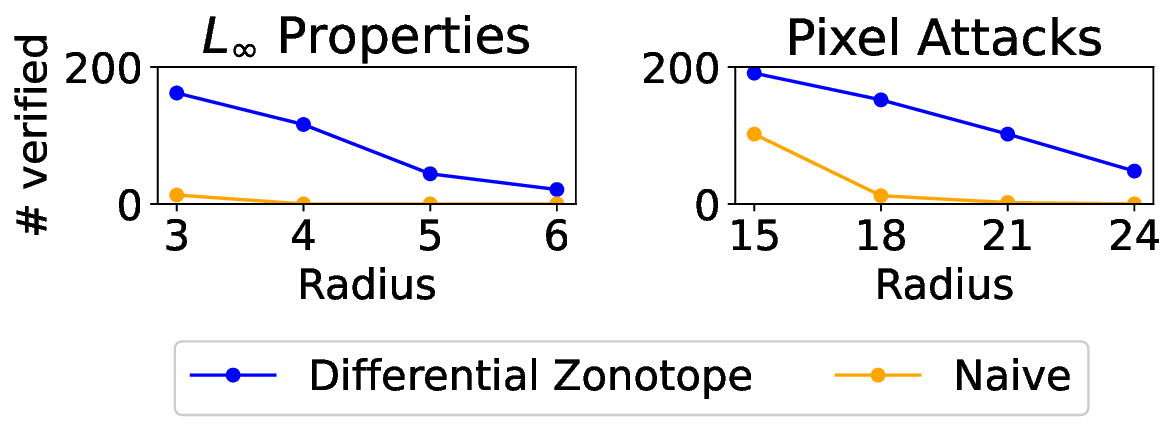}
    \caption{MNIST benchmark queries for which VeryDiff with(out) Differential Zonotopes proves equivalence.}
    \label{fig:abelation_by_radius}
\end{wrapfigure}
\looseness=-1
We compare six different tools or configurations on old and new benchmark families.
A detailed summary of baselines, benchmark families and \ac{NN} architectures can be found in \Cref{apx:subsec:eval:setup}.
For $\varepsilon$ and Top-1 equivalence, we evaluate on preexisting and new ACAS and MNIST \acp{NN} (airborne collision avoidance and handwritten digit recognition) where the second \ac{NN} is generated via pruning (and possibly further training).
For MNIST, we evaluate w.r.t. input regions generated by Paulsen \emph{et al.}~\cite{paulsen_neurodiff_2020} which prove equivalence on $L_\infty$ bounded perturbations of images ($L_\infty$ Properties) or targeted pixel perturbations (Pixel Attacks)
For $\delta$-Top-1 equivalence, we introduce a new \ac{NN} verification benchmark for particle jet classification at CERN's Large Hadron Collider (LHC)~\cite{duarte2018fast}.
We analyze equivalence w.r.t. pruned and further trained \ac{NN}.
\acp{NN} in this context come with strict real-time requirements making pruned \acp{NN} highly desirable~\cite{duarte2018fast}.
We verify equivalence for boxes defined via standard deviations over the normalized input space.

\subsubsection{Where Differential Verification helps.}
\looseness=-1
We compared VeryDiff with Differential Zonotopes activated to the naive computation without Differential Zonotopes.
A summary of all results can be found in \Cref{tab:abelation_diff_zono} (\Cref{apx:subsec:eval:abelation}).
We see significant improvements for $\varepsilon$ equivalence (1430\% more instances certified for the MNIST (VeriPrune) benchmark family).
\Cref{fig:abelation_by_radius} breaks down the verified $\varepsilon$ equivalence queries with/without Differential Zonotopes by input radius across all MNIST benchmark queries and shows that Differential Verification helps to verify larger input regions.
On the other hand, differential analysis slows down the verification of Top-1 equivalence verification (see \Cref{tab:abelation_diff_zono}).
We provide a complementary theoretical analysis to this observation in \Cref{sec:classification} and posit this is a fundamental limitation of Differential Verification and not specific to our implementation.
In contrast, certification of confidence-based equivalence can profit from Differential Verification across two benchmark families.
While we see diminishing speedups as we push the confidence threshold $\delta$ closer to $1$ (implicitly reducing the input space with guarantees), realistic thresholds $\delta > 0.5$ (e.g. $0.9$ instead of $1-10^{-7}$) profit \emph{most} from differential verification (up to speedups of 677 over the naive technique on commonly solved queries for LHC w.r.t. $\delta=0.99$).
We found queries where $\varepsilon$ and Top-1 equivalence results differ -- underlining the importance of choosing equivalence properties w.r.t. the task at hand.

\begin{table}[t]
    \centering
    \caption{Verification results for $\varepsilon$ and $\delta$-Top-1 equivalence: Speedups for commonly solved instances; improvements reported w.r.t. the best other tool.}
    \label{tab:epsilon_equivalence_sota}
\begin{tabular}{l|p{1.7cm}|l|rr|rr|rr}
\multicolumn{2}{c|}{\multirow{2}{*}{\textbf{Benchmark}}} &
\multirow{2}{*}{\textbf{Variant}} &
\multicolumn{2}{c|}{\multirow{2}{*}{\textbf{Equiv.}}}  &
\multicolumn{2}{c|}{\multirow{2}{*}{\textbf{Counterex.}}} &
\multicolumn{2}{c}{\textbf{Speedup}}\\ \cline{8-9}
\multicolumn{2}{c|}{}&&&&&& Median & Max\\\hline\hline

\multirow{9}{*}{\rotatebox[origin=c]{90}{\parbox{1.5cm}{\centering Standard\newline $\varepsilon$ eq.}}} &
\multirow{5}{*}{ACAS}
               &  VeryDiff (ours) &  \textbf{150} & \good{(+24.0\%)}       &  \textbf{153} & \good{(+2.0\%)}       &    --- &     --- \\
      &        &    NNEquiv &   37 &        &  142 &        &   \good{37.3} &  \good{8091.2} \\
      &        &   MILPEquiv &   16 &        &    3 &        & \good{7224.8} & \good{36297.0} \\
      &        &   Marabou   &  110 &        &  109 &        &  \good{141.3} & \good{10070.5} \\
      &        &    $\alpha,\beta$-CROWN &  121 &        &  150 &        &   \good{15.4} &  \good{1954.1} \\\cline{2-9}
      
&\multirow{4}{1.7cm}{MNIST (VeriPrune)} 
               &  VeryDiff  (ours) &  \textbf{352} &  \good{(+101.1\%)}        &   62 &   \bad{(-43.6\%)}      &    --- &     --- \\
      &        &    NNEquiv &    0 &        &  103 &        &   \good{12.6} &   \good{166.2} \\
      &        &   Marabou   &  10 &        &   24 &        &  \good{183.9} & \good{1390.8} \\
      &        &    $\alpha,\beta$-CROWN\tablefootnote{$\alpha,\beta$-CROWN for MNIST spends 80\% of its time on neuron-bound refinement. Speedups may be exaggerated while increases in solved instances are accurate.} &  175 &        &  \textbf{110} &        &  \good{(516.9)} &  \good{(4220.8)} \\ \hline\hline

\multirow{4}{*}{\rotatebox[origin=c]{90}{\parbox{1.5cm}{\centering NeuroDiff \newline $\varepsilon$ eq.}}}
&\multirow{2}{*}{ACAS}
               &  VeryDiff  (ours) &  \textbf{169} &  \good{(+39.7\%)}       &  \textbf{161} &    \good{(+103.8\%)}     &    --- &     --- \\
      &        &   NeuroDiff &  121 &        &   79 &        &   \good{43.1} & \good{16134.7} \\\cline{2-9}
&\multirow{2}{1.7cm}{MNIST (VeriPrune)}
               &  VeryDiff  (ours) &  \textbf{457} &    \good{(+11.5\%)}     &  \textbf{242} &   \good{(+404.1\%)}      &    --- &     --- \\
      &        &   NeuroDiff &  410 &        &   48 &        &    \good{4.5} &  \good{1086.8} \\ \hline\hline

\multirow{2}{*}{$\delta$-Top-1} &
\multirow{2}{*}{LHC} & VeryDiff (ours) &
\textbf{77} & \good{(327.8\%)} & --- & & --- & ---\\
&& $\alpha,\beta$-CROWN &
18 & & --- & & \good{324.5} & \good{11274.3}
    \end{tabular}
\end{table}
\subsubsection{$\varepsilon$ Equivalence.}
\looseness=-1
We compare our tool with other $\varepsilon$ equivalence verification techniques from the literature and summarize the results in \Cref{tab:epsilon_equivalence_sota}.
Across both benchmarks from prior literature (ACAS and MNIST), we significantly outperform equivalence-specific verifiers (NNEquiv~\cite{Teuber2021a}, MILPEquiv~\cite{kleine_buning_verifying_2020}, NeuroDiff~\cite{paulsen_neurodiff_2020}) as well as general NN verification techniques ($\alpha,\beta$-CROWN~\cite{zhang2018efficient,xu2020automatic,xu2021fast,wang2021beta,zhang22babattack,zhang2022general,shi2024genbab,kotha2023provably}, Marabou~\cite{katz_marabou_2019,DBLP:conf/cav/WuIZTDKRAJBHLWZKKB24}) for certification of equivalence.
$\alpha,\beta$-CROWN outperforms VeryDiff in the search for counterexamples.
We suspect this is due to its adversarial attack techniques~\cite{zhang22babattack}.
Due to incompatible differences in the property checked by NeuroDiff's implementation (see \Cref{apx:subsec:eval:setup}), we performed a separate comparison where VeryDiff outperforms NeuroDiff on the same property.

\subsubsection{$\delta$-Top-1 Equivalence.}
\looseness=-1
Differential Verification significantly improves upon the generic State-of-the-Art \ac{NN} verifier $\alpha,\beta$-CROWN (see \Cref{tab:epsilon_equivalence_sota}).
We concede that $\alpha,\beta$-CROWN's attack techniques outperform VeryDiff's counterexample generation (see $\varepsilon$ equivalence).
Hence, our evaluation focuses on equivalence certification.
The objective for $\delta$-Top-1 equivalence is to provide guarantees for low values $\delta$.
$\alpha,\beta$-CROWN was only able to provide guarantees  for 10 of the \acp{NN} in the benchmark set.
In each case, VeryDiff was able to prove equivalence for lower (i.e. better) or equal $\delta$ values.
For the 3 \acp{NN} where the provided guarantees of $\alpha,\beta$-CROWN and VeryDiff matched, both tools only verified equivalence for $\delta=1-10^{-7}$, i.e. they provided an extremely limited guarantee.
This underlines that VeryDiff is a significant step forward in the verification of $\delta$-Top-1 equivalence.

\subsubsection{Limitations}
\looseness=-1
Larger weight differences between \acp{NN} and accumulating $\reluSym{}$ approximations may decrease speedups achievable via Differential Verification (see discussion in \Cref{apx:subsec:eval:confidence}).
Nonetheless, VeryDiff outperforms alternative verifiers -- often by orders of magnitude.
Another limitation for confidence-based \ac{NN} verification is the possibility of satisfied for high confidence thresholds (for mitigation via calibration~\cite{DBLP:conf/icml/GuoPSW17,DBLP:conf/uai/AoRS23} see Athavale \emph{et al.}~\cite{DBLP:conf/cav/AthavaleBCMNW24}).
However, for equivalence verification, we consider the reference \ac{NN} $f_1$ (incl. its confidence) trustworthy.

\section{Conclusion}
\looseness=-1
We introduced
Differential Zonotopes as an abstract domain for \ac{NN} equivalence verification.
Our extensive evaluation shows that we outperform the Differential Verification tool \mbox{NeuroDiff}~\cite{paulsen_reludiff_2020,paulsen_neurodiff_2020}) as well as State-of-the-Art \ac{NN} verifiers.
Moreover, our paper provides insights into the circumstances where differential reasoning does (not) aid verification.
As discussed in \Cref{sec:classification,sec:evaluation}, whether Differential Verification helps is not always straightforward for specifications involving classification.
We believe that confidence-based equivalence is the way forward to scale equivalence verification beyond tiny input regions such as $\epsilon$-balls criticized in the literature~\cite{DBLP:conf/icml/GengLXWGS23}.
Finally, we introduced a simpler approximation for $\mathrm{softmax}$ that is provably tighter than prior work~\cite{DBLP:conf/cav/AthavaleBCMNW24}.

\paragraph{Future Work.}
We see potential in extending generic \ac{NN} verifiers (e.g. Marabou's network level reasoner~\cite{DBLP:conf/cav/WuIZTDKRAJBHLWZKKB24}) with differential abstract domains to improve their reasoning capabilities for relational properties.

\bibliographystyle{splncs04}
\bibliography{content/main}

\begin{thebibliography}{10}
\providecommand{\url}[1]{\texttt{#1}}
\providecommand{\urlprefix}{URL }
\providecommand{\doi}[1]{https://doi.org/#1}

\bibitem{DBLP:conf/uai/AoRS23}
Ao, S., Rueger, S., Siddharthan, A.: Two sides of miscalibration: Identifying
  over and under-confidence prediction for network calibration. In: Evans,
  R.J., Shpitser, I. (eds.) Uncertainty in Artificial Intelligence, {UAI} 2023,
  July 31 - 4 August 2023, Pittsburgh, PA, {USA}. Proceedings of Machine
  Learning Research, vol.~216, pp. 77--87. {PMLR} (2023),
  \url{https://proceedings.mlr.press/v216/ao23a.html}

\bibitem{DBLP:conf/cav/AthavaleBCMNW24}
Athavale, A., Bartocci, E., Christakis, M., Maffei, M., Nickovic, D.,
  Weissenbacher, G.: Verifying global two-safety properties in neural networks
  with confidence. In: Gurfinkel, A., Ganesh, V. (eds.) Computer Aided
  Verification - 36th International Conference, {CAV} 2024, Montreal, QC,
  Canada, July 24-27, 2024, Proceedings, Part {II}. LNCS, vol. 14682, pp.
  329--351. Springer (2024). \doi{10.1007/978-3-031-65630-9_17}

\bibitem{bak_nnenum_2021}
Bak, S.: nnenum: Verification of relu neural networks with optimized
  abstraction refinement. In: Dutle, A., Moscato, M.M., Titolo, L.,
  Mu{\~{n}}oz, C.A., Perez, I. (eds.) {NASA} Formal Methods - 13th
  International Symposium, {NFM} 2021, Virtual Event, May 24-28, 2021,
  Proceedings. LNCS, vol. 12673, pp. 19--36. Springer (2021).
  \doi{10.1007/978-3-030-76384-8_2}

\bibitem{vnncomp2024}
Bak, S., Brix, C., Johnson, T., Liu, C., Wu, H.: {VNN-COMP} 2024 (2024),
  \url{https://docs.google.com/presentation/d/1RvZWeAdTfRC3bNtCqt84O6IIPoJBnF4jnsEvhTTxsPE/edit?usp=sharing},
  accessed: 10/04/2024

\bibitem{bak_improved_2020}
Bak, S., Tran, H., Hobbs, K., Johnson, T.T.: Improved geometric path
  enumeration for verifying {ReLU} neural networks. In: Lahiri, S.K., Wang, C.
  (eds.) Computer Aided Verification - 32nd International Conference, {CAV}
  2020, Los Angeles, CA, USA, July 21-24, 2020, Proceedings, Part {I}. LNCS,
  vol. 12224, pp. 66--96. Springer (2020). \doi{10.1007/978-3-030-53288-8_4}

\bibitem{banerjee2024relational}
Banerjee, D., Singh, G.: Relational {DNN} verification with cross executional
  bound refinement. In: Forty-first International Conference on Machine
  Learning (2024), \url{https://openreview.net/forum?id=HOG80Yk4Gw}

\bibitem{DBLP:journals/pacmpl/BanerjeeXS24}
Banerjee, D., Xu, C., Singh, G.: Input-relational verification of deep neural
  networks. Proc. {ACM} Program. Lang.  \textbf{8}({PLDI}),  1--27 (2024).
  \doi{10.1145/3656377}

\bibitem{DBLP:conf/fm/BartheCK11}
Barthe, G., Crespo, J.M., Kunz, C.: Relational verification using product
  programs. In: Butler, M.J., Schulte, W. (eds.) {FM} 2011: Formal Methods -
  17th International Symposium on Formal Methods, Limerick, Ireland, June
  20-24, 2011. Proceedings. LNCS, vol.~6664, pp. 200--214. Springer (2011).
  \doi{10.1007/978-3-642-21437-0_17}

\bibitem{bezanson2017julia}
Bezanson, J., Edelman, A., Karpinski, S., Shah, V.B.: Julia: {A} fresh approach
  to numerical computing. {SIAM} Rev.  \textbf{59}(1),  65--98 (2017).
  \doi{10.1137/141000671}, \url{https://doi.org/10.1137/141000671}

\bibitem{DBLP:journals/sttt/BrixMBJL23}
Brix, C., M{\"{u}}ller, M.N., Bak, S., Johnson, T.T., Liu, C.: First three
  years of the international verification of neural networks competition
  {(VNN-COMP)}. Int. J. Softw. Tools Technol. Transf.  \textbf{25}(3),
  329--339 (2023). \doi{10.1007/S10009-023-00703-4}

\bibitem{cinar2019classification}
Cinar, I., Koklu, M.: Classification of rice varieties using artificial
  intelligence methods. International Journal of Intelligent Systems and
  Applications in Engineering  \textbf{7}(3),  188--194 (2019)

\bibitem{DBLP:conf/stoc/Cook71}
Cook, S.A.: The complexity of theorem-proving procedures. In: Harrison, M.A.,
  Banerji, R.B., Ullman, J.D. (eds.) Proceedings of the 3rd Annual {ACM}
  Symposium on Theory of Computing, May 3-5, 1971, Shaker Heights, Ohio, {USA}.
  pp. 151--158. {ACM} (1971). \doi{10.1145/800157.805047}

\bibitem{Guidotti}
Demarchi, S., Guidotti, D., Pulina, L., Tacchella, A.: Supporting
  standardization of neural networks verification with vnnlib and coconet. In:
  Narodytska, N., Amir, G., Katz, G., Isac, O. (eds.) Proceedings of the 6th
  Workshop on Formal Methods for ML-Enabled Autonomous Systems. Kalpa
  Publications in Computing, vol.~16, pp. 47--58. EasyChair (2023).
  \doi{10.29007/5pdh}, \url{/publications/paper/Qgdn}

\bibitem{duarte2018fast}
Duarte, J., Han, S., Harris, P., Jindariani, S., Kreinar, E., Kreis, B.,
  Ngadiuba, J., Pierini, M., Rivera, R., Tran, N., Wu, Z.: Fast inference of
  deep neural networks in fpgas for particle physics. Journal of
  Instrumentation  \textbf{13}(07),  P07027 (jul 2018).
  \doi{10.1088/1748-0221/13/07/P07027}

\bibitem{Eleftheriadis2022}
Eleftheriadis, C., Kekatos, N., Katsaros, P., Tripakis, S.: On neural network
  equivalence checking using {SMT} solvers. In: Bogomolov, S., Parker, D.
  (eds.) Formal Modeling and Analysis of Timed Systems - 20th International
  Conference, {FORMATS} 2022, Warsaw, Poland, September 13-15, 2022,
  Proceedings. LNCS, vol. 13465, pp. 237--257. Springer (2022).
  \doi{10.1007/978-3-031-15839-1_14}

\bibitem{GehrZonotope}
Gehr, T., Mirman, M., Drachsler{-}Cohen, D., Tsankov, P., Chaudhuri, S.,
  Vechev, M.T.: {AI2:} safety and robustness certification of neural networks
  with abstract interpretation. In: 2018 {IEEE} Symposium on Security and
  Privacy, {SP} 2018, Proceedings, 21-23 May 2018, San Francisco, California,
  {USA}. pp. 3--18. {IEEE} Computer Society (2018). \doi{10.1109/SP.2018.00058}

\bibitem{DBLP:conf/icml/GengLXWGS23}
Geng, C., Le, N., Xu, X., Wang, Z., Gurfinkel, A., Si, X.: Towards reliable
  neural specifications. In: Krause, A., Brunskill, E., Cho, K., Engelhardt,
  B., Sabato, S., Scarlett, J. (eds.) International Conference on Machine
  Learning, {ICML} 2023, 23-29 July 2023, Honolulu, Hawaii, {USA}. Proceedings
  of Machine Learning Research, vol.~202, pp. 11196--11212. {PMLR} (2023),
  \url{https://proceedings.mlr.press/v202/geng23a.html}

\bibitem{DBLP:conf/icml/GuoPSW17}
Guo, C., Pleiss, G., Sun, Y., Weinberger, K.Q.: On calibration of modern neural
  networks. In: Precup, D., Teh, Y.W. (eds.) Proceedings of the 34th
  International Conference on Machine Learning, {ICML} 2017, Sydney, NSW,
  Australia, 6-11 August 2017. Proceedings of Machine Learning Research,
  vol.~70, pp. 1321--1330. {PMLR} (2017),
  \url{http://proceedings.mlr.press/v70/guo17a.html}

\bibitem{gurobi}
{Gurobi Optimization, LLC}: {Gurobi Optimizer Reference Manual} (2023),
  \url{https://www.gurobi.com}

\bibitem{DBLP:conf/nfm/PP24}
Habeeb, P., Prabhakar, P.: Approximate conformance verification of deep neural
  networks. In: Benz, N., Gopinath, D., Shi, N. (eds.) {NASA} Formal Methods -
  16th International Symposium, {NFM} 2024, Moffett Field, CA, USA, June 4-6,
  2024, Proceedings. LNCS, vol. 14627, pp. 223--238. Springer (2024).
  \doi{10.1007/978-3-031-60698-4_13}

\bibitem{julian_policy_2016}
Julian, K.D., Lopez, J., Brush, J.S., Owen, M.P., Kochenderfer, M.J.: Policy
  compression for aircraft collision avoidance systems. In: {AIAA}/{IEEE}
  {Digital} {Avionics} {Systems} {Conference} - {Proceedings}. vol. 2016-Decem,
  pp. 1--10. IEEE (2016). \doi{10.1109/DASC.2016.7778091}, iSSN: 21557209

\bibitem{DBLP:conf/stoc/Karmarkar84}
Karmarkar, N.: A new polynomial-time algorithm for linear programming. In:
  DeMillo, R.A. (ed.) Proceedings of the 16th Annual {ACM} Symposium on Theory
  of Computing, April 30 - May 2, 1984, Washington, DC, {USA}. pp. 302--311.
  {ACM} (1984). \doi{10.1145/800057.808695}

\bibitem{katz_reluplex_2017}
Katz, G., Barrett, C.W., Dill, D.L., Julian, K., Kochenderfer, M.J.: Reluplex:
  An efficient {SMT} solver for verifying deep neural networks. In: Majumdar,
  R., Kuncak, V. (eds.) Computer Aided Verification - 29th International
  Conference, {CAV} 2017, Heidelberg, Germany, July 24-28, 2017, Proceedings,
  Part {I}. LNCS, vol. 10426, pp. 97--117. Springer (2017).
  \doi{10.1007/978-3-319-63387-9_5}

\bibitem{katz_marabou_2019}
Katz, G., Huang, D.A., Ibeling, D., Julian, K., Lazarus, C., Lim, R., Shah, P.,
  Thakoor, S., Wu, H., Zeljic, A., Dill, D.L., Kochenderfer, M.J., Barrett,
  C.W.: The {Marabou} framework for verification and analysis of deep neural
  networks. In: Dillig, I., Tasiran, S. (eds.) Computer Aided Verification -
  31st International Conference, {CAV} 2019, New York City, NY, USA, July
  15-18, 2019, Proceedings, Part {I}. LNCS, vol. 11561, pp. 443--452. Springer
  (2019). \doi{10.1007/978-3-030-25540-4_26}

\bibitem{kleine_buning_verifying_2020}
{Kleine B{\"{u}}ning}, M., Kern, P., Sinz, C.: Verifying equivalence properties
  of neural networks with {ReLU} activation functions. In: Simonis, H. (ed.)
  Principles and Practice of Constraint Programming - 26th International
  Conference, {CP} 2020, Louvain-la-Neuve, Belgium, September 7-11, 2020,
  Proceedings. LNCS, vol. 12333, pp. 868--884. Springer (2020).
  \doi{10.1007/978-3-030-58475-7_50}

\bibitem{kotha2023provably}
Kotha, S., Brix, C., Kolter, J.Z., Dvijotham, K., Zhang, H.: Provably bounding
  neural network preimages. In: Oh, A., Naumann, T., Globerson, A., Saenko, K.,
  Hardt, M., Levine, S. (eds.) Advances in Neural Information Processing
  Systems 36: Annual Conference on Neural Information Processing Systems 2023,
  NeurIPS 2023, New Orleans, LA, USA, December 10 - 16, 2023 (2023),
  \url{http://papers.nips.cc/paper\_files/paper/2023/hash/fe061ec0ae03c5cf5b5323a2b9121bfd-Abstract-Conference.html}

\bibitem{CEG4N}
Matos, J.B.P., Filho, E.B.d.L., Bessa, I., Manino, E., Song, X., Cordeiro,
  L.C.: Counterexample guided neural network quantization refinement. IEEE
  Transactions on Computer-Aided Design of Integrated Circuits and Systems
  pp.~1--1 (2023). \doi{10.1109/TCAD.2023.3335313}

\bibitem{DBLP:conf/mlsys/MullerS0PV21}
M{\"{u}}ller, C., Serre, F., Singh, G., P{\"{u}}schel, M., Vechev, M.T.:
  Scaling polyhedral neural network verification on gpus. In: Smola, A.,
  Dimakis, A., Stoica, I. (eds.) Proceedings of the Fourth Conference on
  Machine Learning and Systems, MLSys 2021, virtual, April 5-9, 2021. mlsys.org
  (2021),
  \url{https://proceedings.mlsys.org/paper\_files/paper/2021/hash/7c98f9c7ab2df90911da23f9ce72ed6e-Abstract.html}

\bibitem{narodytska_verifying_2018}
Narodytska, N., Kasiviswanathan, S.P., Ryzhyk, L., Sagiv, M., Walsh, T.:
  Verifying properties of binarized deep neural networks. In: McIlraith, S.A.,
  Weinberger, K.Q. (eds.) Proceedings of the Thirty-Second {AAAI} Conference on
  Artificial Intelligence, {AAAI-18}, New Orleans, Louisiana, USA, February
  2-7, 2018. pp. 6615--6624. {AAAI} Press (2018).
  \doi{10.1609/AAAI.V32I1.12206}

\bibitem{paulsen_reludiff_2020}
Paulsen, B., Wang, J., Wang, C.: Reludiff: differential verification of deep
  neural networks. In: Rothermel, G., Bae, D. (eds.) {ICSE} '20: 42nd
  International Conference on Software Engineering, Seoul, South Korea, 27 June
  - 19 July, 2020. pp. 714--726. {ACM} (2020). \doi{10.1145/3377811.3380337}

\bibitem{paulsen_neurodiff_2020}
Paulsen, B., Wang, J., Wang, J., Wang, C.: {NeuroDiff:} scalable differential
  verification of neural networks using fine-grained approximation. In: 35th
  {IEEE/ACM} International Conference on Automated Software Engineering, {ASE}
  2020, Melbourne, Australia, September 21-25, 2020. pp. 784--796. {IEEE}
  (2020). \doi{10.1145/3324884.3416560}

\bibitem{DBLP:conf/rp/SalzerL21}
S{\"{a}}lzer, M., Lange, M.: Reachability is np-complete even for the simplest
  neural networks. In: Bell, P.C., Totzke, P., Potapov, I. (eds.) Reachability
  Problems - 15th International Conference, {RP} 2021, Liverpool, UK, October
  25-27, 2021, Proceedings. LNCS, vol. 13035, pp. 149--164. Springer (2021).
  \doi{10.1007/978-3-030-89716-1_10}

\bibitem{shi2024genbab}
Shi, Z., Jin, Q., Kolter, Z., Jana, S., Hsieh, C., Zhang, H.: Neural network
  verification with branch-and-bound for general nonlinearities. CoRR
  \textbf{abs/2405.21063} (2024). \doi{10.48550/ARXIV.2405.21063},
  \url{https://doi.org/10.48550/arXiv.2405.21063}

\bibitem{shriver_dnnv_2021}
Shriver, D., Elbaum, S.G., Dwyer, M.B.: {DNNV:} {A} framework for deep neural
  network verification. In: Silva, A., Leino, K.R.M. (eds.) Computer Aided
  Verification - 33rd International Conference, {CAV} 2021, Virtual Event, July
  20-23, 2021, Proceedings, Part {I}. LNCS, vol. 12759, pp. 137--150. Springer
  (2021). \doi{10.1007/978-3-030-81685-8_6}

\bibitem{Singh18}
Singh, G., Gehr, T., Mirman, M., P{\"{u}}schel, M., Vechev, M.T.: Fast and
  effective robustness certification. In: Bengio, S., Wallach, H.M.,
  Larochelle, H., Grauman, K., Cesa{-}Bianchi, N., Garnett, R. (eds.) Advances
  in Neural Information Processing Systems 31: Annual Conference on Neural
  Information Processing Systems 2018, NeurIPS 2018, December 3-8, 2018,
  Montr{\'{e}}al, Canada. pp. 10825--10836 (2018),
  \url{https://proceedings.neurips.cc/paper/2018/hash/f2f446980d8e971ef3da97af089481c3-Abstract.html}

\bibitem{Teuber2021a}
Teuber, S., B{\"{u}}ning, M.K., Kern, P., Sinz, C.: Geometric path enumeration
  for equivalence verification of neural networks. In: 33rd {IEEE}
  International Conference on Tools with Artificial Intelligence, {ICTAI} 2021,
  Washington, DC, USA, November 1-3, 2021. pp. 200--208. {IEEE} (2021).
  \doi{10.1109/ICTAI52525.2021.00035}

\bibitem{tran_star-based_2019}
Tran, H., Lopez, D.M., Musau, P., Yang, X., Nguyen, L.V., Xiang, W., Johnson,
  T.T.: Star-based reachability analysis of deep neural networks. In: ter Beek,
  M.H., McIver, A., Oliveira, J.N. (eds.) Formal Methods - The Next 30 Years -
  Third World Congress, {FM} 2019, Porto, Portugal, October 7-11, 2019,
  Proceedings. LNCS, vol. 11800, pp. 670--686. Springer (2019).
  \doi{10.1007/978-3-030-30942-8_39}

\bibitem{wang_efficient_2018}
Wang, S., Pei, K., Whitehouse, J., Yang, J., Jana, S.: Efficient formal safety
  analysis of neural networks. In: Bengio, S., Wallach, H.M., Larochelle, H.,
  Grauman, K., Cesa{-}Bianchi, N., Garnett, R. (eds.) Advances in Neural
  Information Processing Systems 31: Annual Conference on Neural Information
  Processing Systems 2018, NeurIPS 2018, December 3-8, 2018, Montr{\'{e}}al,
  Canada. pp. 6369--6379 (2018),
  \url{https://proceedings.neurips.cc/paper/2018/hash/2ecd2bd94734e5dd392d8678bc64cdab-Abstract.html}

\bibitem{wang_formal_2018}
Wang, S., Pei, K., Whitehouse, J., Yang, J., Jana, S.: Formal security analysis
  of neural networks using symbolic intervals. In: Enck, W., Felt, A.P. (eds.)
  27th {USENIX} Security Symposium, {USENIX} Security 2018, Baltimore, MD, USA,
  August 15-17, 2018. pp. 1599--1614. {USENIX} Association (2018),
  \url{https://www.usenix.org/conference/usenixsecurity18/presentation/wang-shiqi}

\bibitem{wang2021beta}
Wang, S., Zhang, H., Xu, K., Lin, X., Jana, S., Hsieh, C., Kolter, J.Z.:
  Beta-crown: Efficient bound propagation with per-neuron split constraints for
  complete and incomplete neural network verification. CoRR
  \textbf{abs/2103.06624} (2021), \url{https://arxiv.org/abs/2103.06624}

\bibitem{WANG2024127347}
Wang, W., Wang, K., Cheng, Z., Yang, Y.: Veriprune: Equivalence verification of
  node pruned neural network. Neurocomputing  \textbf{577},  127347 (2024).
  \doi{https://doi.org/10.1016/j.neucom.2024.127347},
  \url{https://www.sciencedirect.com/science/article/pii/S0925231224001188}

\bibitem{DBLP:conf/cav/WuIZTDKRAJBHLWZKKB24}
Wu, H., Isac, O., Zeljic, A., Tagomori, T., Daggitt, M.L., Kokke, W., Refaeli,
  I., Amir, G., Julian, K., Bassan, S., Huang, P., Lahav, O., Wu, M., Zhang,
  M., Komendantskaya, E., Katz, G., Barrett, C.W.: Marabou 2.0: {A} versatile
  formal analyzer of neural networks. In: Gurfinkel, A., Ganesh, V. (eds.)
  Computer Aided Verification - 36th International Conference, {CAV} 2024,
  Montreal, QC, Canada, July 24-27, 2024, Proceedings, Part {II}. LNCS, vol.
  14682, pp. 249--264. Springer (2024). \doi{10.1007/978-3-031-65630-9_13}

\bibitem{xu2020automatic}
Xu, K., Shi, Z., Zhang, H., Wang, Y., Chang, K., Huang, M., Kailkhura, B., Lin,
  X., Hsieh, C.: Automatic perturbation analysis for scalable certified
  robustness and beyond. In: Larochelle, H., Ranzato, M., Hadsell, R., Balcan,
  M., Lin, H. (eds.) Advances in Neural Information Processing Systems 33:
  Annual Conference on Neural Information Processing Systems 2020, NeurIPS
  2020, December 6-12, 2020, virtual (2020),
  \url{https://proceedings.neurips.cc/paper/2020/hash/0cbc5671ae26f67871cb914d81ef8fc1-Abstract.html}

\bibitem{xu2021fast}
Xu, K., Zhang, H., Wang, S., Wang, Y., Jana, S., Lin, X., Hsieh, C.: Fast and
  complete: Enabling complete neural network verification with rapid and
  massively parallel incomplete verifiers. In: 9th International Conference on
  Learning Representations, {ICLR} 2021, Virtual Event, Austria, May 3-7, 2021
  (2021)

\bibitem{zhang2022general}
Zhang, H., Wang, S., Xu, K., Li, L., Li, B., Jana, S., Hsieh, C., Kolter, J.Z.:
  General cutting planes for bound-propagation-based neural network
  verification. In: Koyejo, S., Mohamed, S., Agarwal, A., Belgrave, D., Cho,
  K., Oh, A. (eds.) Advances in Neural Information Processing Systems 35:
  Annual Conference on Neural Information Processing Systems 2022, NeurIPS
  2022, New Orleans, LA, USA, November 28 - December 9, 2022 (2022),
  \url{http://papers.nips.cc/paper\_files/paper/2022/hash/0b06c8673ebb453e5e468f7743d8f54e-Abstract-Conference.html}

\bibitem{zhang22babattack}
Zhang, H., Wang, S., Xu, K., Wang, Y., Jana, S., Hsieh, C., Kolter, J.Z.: A
  branch and bound framework for stronger adversarial attacks of {ReLU}
  networks. In: Chaudhuri, K., Jegelka, S., Song, L., Szepesv{\'{a}}ri, C.,
  Niu, G., Sabato, S. (eds.) International Conference on Machine Learning,
  {ICML} 2022, 17-23 July 2022, Baltimore, Maryland, {USA}. Proceedings of
  Machine Learning Research, vol.~162, pp. 26591--26604. {PMLR} (2022),
  \url{https://proceedings.mlr.press/v162/zhang22ae.html}

\bibitem{zhang2018efficient}
Zhang, H., Weng, T., Chen, P., Hsieh, C., Daniel, L.: Efficient neural network
  robustness certification with general activation functions. In: Bengio, S.,
  Wallach, H.M., Larochelle, H., Grauman, K., Cesa{-}Bianchi, N., Garnett, R.
  (eds.) Advances in Neural Information Processing Systems 31: Annual
  Conference on Neural Information Processing Systems 2018, NeurIPS 2018,
  December 3-8, 2018, Montr{\'{e}}al, Canada. pp. 4944--4953 (2018),
  \url{https://proceedings.neurips.cc/paper/2018/hash/d04863f100d59b3eb688a11f95b0ae60-Abstract.html}

\bibitem{DBLP:conf/cav/ZhangSS23}
Zhang, Y., Song, F., Sun, J.: {QEBVerif}: Quantization error bound verification
  of neural networks. In: Enea, C., Lal, A. (eds.) Computer Aided Verification
  - 35th International Conference, {CAV} 2023, Paris, France, July 17-22, 2023,
  Proceedings, Part {II}. LNCS, vol. 13965, pp. 413--437. Springer (2023).
  \doi{10.1007/978-3-031-37703-7_20}

\end{thebibliography}

-%
\clearpage
\appendix
\section{Proofs}
\label{apx:proofs}
\printProofs

\section{Input Space Refinement}
\label{apx:input_refinement}
To split the input space, we require a heuristic that estimates the influence of a split along some dimension on the verification outcome.
Splitting along some dimension can improve the obtained bounds in two ways:
Either the reduced input range directly reduces the computed output bounds (\emph{direct influence}),
or the reduced range reduces the number of instable neurons and hence reduces the overapproximation error w.r.t. output bounds (\emph{indirect influence}).
Consider an affine form $\affineForm{}=\left(\mathbf{e},\mathbf{a},c\right)$ with $n$ dimensions in $\mathbf{e}$ and $p$ dimension in $\mathbf{a}$.
The direct influence of dimension $i \in \left[1,n\right]$ can be estimated via $\left|\mathbf{e}_i\right|$:
Intuitively, splitting the input region into two parts along dimension $i$ ought to approximately reduce $\left|\mathbf{e}_i\right|$ by $\frac{1}{2}$.
However, for larger \acp{NN} $\mathbf{a}$ also has a large influence on the computed bounds.
Unfortunately, while $\mathbf{a}_j$ (for $j \in \left[1,p\right]$) tells us the influence of the $j$-th additional generator on $\affineForm{}$'s bounds, it contains no information on the relationship between generator $j$ and the input dimensions.
In this sense, Zonotopes lack information on the indirect influence of splitting some input dimension.
Thus, for each additional generator $\mathbf{a}_j$ we propose to add an additional vector $\mathbf{d}\left(j\right) \in \mathbb{R}^n$ which stores the indirect influence.
The key insight for computing $\mathbf{d}\left(j\right)$ is the idea that each generator is added due to some instable $\reluSym{}$ node.
Let $\tilde{\affineForm{}} = \left(\tilde{\mathbf{e}},\tilde{\mathbf{a}},\tilde{c}\right)$ be the affine form representing the input of the instable $\reluSym{}$ node, then we compute $\mathbf{d}\left(j\right)$ as follows:
\[
\mathbf{d}\left(j\right) = \left|\tilde{\mathbf{e}}\right| + \sum_{k=1}^{\tilde{p}} \left|\tilde{\mathbf{a}}_k\right| \left|\mathbf{d}\left(k\right)\right|.
\]
When the influence vectors $\mathbf{d}\left(k\right)$ are given as a matrix $D \in \mathbb{R}^{n \times \left(n+p\right)}$ where the first $n$ columns encode an identity matrix (these columns estimate the direct impact of the input dimensions), and the instable $\reluSym{}$ inputs are given via a Zonotope $\zonotope{} = \left(G,\mathbf{c}\right)$, the missing influence vectors can simply be computed via $D \left(G^T\right)$.
Subsequently, to compute the combined direct and indirect influence of splitting dimension $i$ w.r.t. the bounds of $\affineForm{}$, we compute:
\[
\left|\mathbf{e}_i\right| + \sum_{k=1}^p \left|\mathbf{a}_k\right| \left|\mathbf{d}\left(k\right)_i\right|.
\]

For refinement, given an output $\left(\zonotope{}',\zonotope{}'',\zonotope{}^\Delta\right)$ of $\textsc{Reach}_\Delta$ and influence vector matrices $D'$ and $D''$ for $\zonotope{}',\zonotope{}''$ , we compute\footnote{$\left|A\right|$ represents the component-wise application of $\left| \cdot \right|$ to $A$.} $\left|G'\right| \left|D'\right|^T + \left|G''\right| \left|D''\right|^T$
and sum the resulting matrix over the output dimensions of $G'$/$G''$.
This yields an approximation of the direct and indirect influence of all input dimensions on all output dimensions.
These influence values are then scaled by the range of the input dimension.
Subsequently, the input dimension with maximal value is chosen for refinement.
Experiments preceding the final evaluation showed that this heuristic computation improved performance over more naive baselines.
We leave an in-depth evaluation of heuristics for equivalence verification refinement to future work.

\section{Extended Evaluation}
\label{apx:additional_experimental_results}
\FloatBarrier{}
To evaluate the techniques proposed in this paper, we performed a comprehensive evaluation w.r.t. different properties and various baselines.
\Cref{apx:subsec:eval:setup} provides an overview of the experimental setup, including used baselines and evaluated benchmark families.
\Cref{apx:subsec:eval:abelation} contains an ablation study for the usage of Differential Zonotopes.
In particular, it demonstrates that Differential Zonotopes do not help verify Top-1 equivalence (see also \Cref{sec:classification} for a complimentary theoretical analysis).
\Cref{apx:subsec:eval:epsilon} evaluates the performance of our tool for $\varepsilon$-equivalence including a comprehensive comparison to State-of-the-Art tools.
Finally, \Cref{apx:subsec:eval:confidence} evaluates our tool's performance for confidence-based equivalence verification.

\subsection{Experimental Setup}
\label{apx:subsec:eval:setup}
We implemented our approach in a new tool called \emph{VeryDiff} in the programming language Julia~\cite{bezanson2017julia}.
Where necessary, we use Gurobi~\cite{gurobi} as LP optimization backend.
All experiments were performed on a Ubuntu machine with a 4 core, 3.20GHz Intel Core i5-6500 CPU, with 16GB of RAM.
All experiments were performed in single-threaded mode and in sequence.
With the exception of \Cref{apx:subsec:eval:confidence}, all timeouts were fixed to 2 minutes per benchmark query.

\subsubsection{Baselines}
To evaluate the performance of our tool, we compared it to six other tools/configurations that could be used for equivalence verification:

\paragraph{Ablations}
To demonstrate the (dis)advantages of Differential Verification we compare to our own naive implementation where Differential Zonotopes are not computed explicitly, but derived from reachable Zonotopes of the individual \acp{NN}.
This allows us to evaluate the effect of adding Differential Zonotopes while keeping all other implementation details constant.

\paragraph{Verification Tools.}
We compare to the equivalence verification tools NNEquiv~\cite{Teuber2021a} and MILPEquiv~\cite{kleine_buning_verifying_2020} which are resp. based on a MILP encoding via Gurobi~\cite{gurobi} and a reachability analysis via star-based geometric path enumeration~\cite{bak_improved_2020,tran_star-based_2019}.
Both tools support Top-1 and $\varepsilon$ equivalence verification.
For $\varepsilon$-equivalence we also compare to the differential verification tool NeuroDiff~\cite{paulsen_neurodiff_2020} (this tool only handles $\varepsilon$-equivalence properties).
The implementation of NeuroDiff checks a slightly different property than the standard definition of $\varepsilon$-equivalence.
Namely, the tool checks whether one \emph{given} output node (usually the predicted output) does not change by more than $\varepsilon$\footnote{See line 64 of \texttt{DiffNN-Code/split.c} in artifact.}.
Unfortunately, this implicitly also affects the tool's branching heuristic, which optimizes split decisions to show this property.
Since it was unclear how the heuristic could be modified to account for all outputs, we instead decided to compare NeuroDiff and VeryDiff w.r.t. the modified property (from hereon called \emph{NeuroDiff $\varepsilon$ equivalence}).
We want to emphasize, that all other experiments were run w.r.t. the standard definition of $\varepsilon$ equivalence (see \Cref{def:epsilon_equivalence}) and that all provided comparisons are w.r.t. the \emph{same} property.
As a baseline for generic \ac{NN} verification, we use $\alpha,\beta$-CROWN~\cite{zhang2018efficient,xu2020automatic,xu2021fast,wang2021beta,zhang22babattack,zhang2022general,shi2024genbab,kotha2023provably} and Marabou~\cite{katz_marabou_2019,DBLP:conf/cav/WuIZTDKRAJBHLWZKKB24} as the resp. fastest verifier and fastest CPU-only verifier at VNNComp 2024~\cite{DBLP:journals/sttt/BrixMBJL23,vnncomp2024}.
For both tools, we encode the property via a product \ac{NN} and provide specification via VNNLIB~\cite{Guidotti}.
For $\delta$-Top-1 equivalence we require linear equations as output constraints which are not natively supported by $\alpha,\beta$-CROWN.
Hence, we encoded the constraint matrix as an additional affine layer with $O^2$ outputs for \acp{NN} with $O$ outputs.
For Marabou, we used the default configuration of Marabou 2.0.
For $\alpha,\beta$-CROWN, we adapted the recommended ACAS and MNIST configurations to single-threaded CPU computation.
We also used the modified ACAS configuration for the LHC benchmark comparison.
$\alpha,\beta$-CROWN for MNIST is configured to use 80\% of its time for neuron-bound refinement.
Hence, speedups on commonly solved instances are to be taken with a grain of salt while improvements in the number of solved instances signify progress.

\paragraph{Omitted Comparisons.}
We omit comparisons to ReluDiff~\cite{paulsen_reludiff_2020} (outperformed by NeuroDiff~\cite{paulsen_neurodiff_2020}), VeriPrune~\cite{WANG2024127347} (currently available implementation is bit-equivalent to ReluDiff\footnote{\url{https://github.com/RM2PT/VeriPrune}}) and SMT solving~\cite{Eleftheriadis2022} ($\alpha,\beta$-CROWN outperforms SMT solvers for NN verification).
We omit a comparison to Habeeb and Prabhakar~\cite{DBLP:conf/nfm/PP24} as their approach relies on product \acp{NN} and $\alpha,\beta$-CROWN which we compare to.
We do not compare to RaVeN~\cite{DBLP:journals/pacmpl/BanerjeeXS24} as it only supports verification of relational properties w.r.t. a single \ac{NN}.
We also omit a comparison to the relational verification tool RACoon~\cite{banerjee2024relational} which also focuses on properties w.r.t. a single \ac{NN} and furthermore requires tailored relaxations that are not available for equivalence properties.

\begin{table*}[t]
    \centering
    \caption{Overview on \ac{NN} architectures}
\begin{tabular}{c||c|c|c|c}
    Network & Input Dimension & Activation & Hidden Nodes & Output Dimension \\ \hline
    \textsc{ACAS} & 5 &\reluSym{} & 50-50-50-50-50-50 & 5 \\
    \textsc{MNIST\_2\_100} & 784 &\reluSym{} & 100-100 & 10 \\
    \textsc{MNIST\_2\_512} & 784 &\reluSym{} & 512-512 & 10 \\
    \textsc{MNIST\_3\_1024} & 784 &\reluSym{} & 1024-1024-1024 & 10\\
    \textsc{2\_20} & 7 / 16 & \reluSym{} & 20-20 & 2 / 5\\
    \textsc{2\_40} & 7 / 16 & \reluSym{} & 40-40 & 2 / 5\\
    \textsc{2\_80} & 7 / 16 & \reluSym{} & 80-80 & 2 / 5\\
    \textsc{4\_20} & 16 & \reluSym{} & 20-20-20-20 & 5 \\
    \textsc{4\_40} & 7 & \reluSym{} & 40-40-40-40 & 2
\end{tabular}
    \label{tab:net_architectures}
\end{table*}

\subsubsection{Benchmarks}
This subsection provides an overview of the benchmark families that we used for our evaluation.
We use the following terminology:
A \emph{benchmark query} is a tuple of an input specification, an equivalence property, and a neural network.
One neural network may be used w.r.t. different benchmark queries.
All considered neural networks have one of the architectures specified in \Cref{tab:net_architectures}.
A \emph{benchmark family} then is a set of benchmark queries.
We note, that all currently available equivalence verification tools~\cite{paulsen_neurodiff_2020,Teuber2021a,kleine_buning_verifying_2020}, similar to many general NN verification tools~\cite{bak_nnenum_2021,wang2021beta,katz_marabou_2019}, do not come with a principled argument for their floating point soundness\footnotemark.
\footnotetext{In particular, outward rounding of symbolic intervals (e.g.~\cite{paulsen_neurodiff_2020}) is insufficient and ignores rounding errors at inference time (see e.g.~\cite[Sec. 4.1]{DBLP:conf/mlsys/MullerS0PV21})}
Hence, parting from the original evaluation of NeuroDiff, we omit an evaluation on the quantized \acp{NN} by Paulsen \emph{et al.}~\cite{paulsen_reludiff_2020,paulsen_neurodiff_2020}.

\paragraph{ACAS.}
Wang \emph{et al.}~\cite{WANG2024127347} evaluate their technique on the ACAS \acp{NN} from Julian et al.~\cite{julian_policy_2016}.
ACAS NNs provide advisories to a pilot to avoid Near Mid Air Collisions.
The authors prune 5 to 30 percent of the \acp{NN}' nodes (based on the weight norm) and (dis)prove $\varepsilon$ equivalence w.r.t. the regions proposed by Katz \emph{et al.}~\cite{katz_reluplex_2017} for $\varepsilon=0.05$.
We reuse this benchmark family while omitting benchmark queries that are trivially violated, i.e. the input box's middle point violates the property.
This yields 306 queries spanning 45 \textsc{ACAS} \acp{NN} for $\varepsilon$ equivalence, and 338 \textsc{ACAS} queries spanning 45 \acp{NN} for NeuroDiff $\varepsilon$ equivalence.

\paragraph{MNIST (VeriPrune).}
The authors of VeriPrune~\cite{WANG2024127347} also evaluate their technique on three MNIST \acp{NN} from Paulsen et al.~\cite{paulsen_neurodiff_2020}.
MNIST NNs have 784 inputs representing pixel values and predict the digit (0-9) visible on the image.
We consider the benchmark queries where a second \acp{NN}'s nodes have been pruned by 5\%.
The objective is to prove $\varepsilon$-equivalence for $\varepsilon=1.0$ or (in our evaluation) classification equivalence.
The input regions in this case are provided by Paulsen et al.~\cite{paulsen_neurodiff_2020} and based on a set of 100 reference images.
For each image $x_0$ we either prove equivalence on an $L_\infty$-ball ($\left\{~x \in \mathbb{R}^{784} ~\middle|~\left|x_0-x\right|_\infty \leq \nu~\right\}$) of radius $\nu\in\left\{3,4,5,6\right\}$ or we verify equivalence for targeted pixel perturbations chosen in prior work~\cite{paulsen_neurodiff_2020} where the chosen $p\in\left\{15,18,21,24\right\}$ pixels may vary arbitrarily and all other values are fixed.
We again observed some trivially violated properties (middle point violates property) which we omit.
Due to computational constraints, we only evaluate up to $\nu \leq 5$ and $p \leq 21$ for Top-1 equivalence.
This yields 701 benchmark queries across one \textsc{MNIST\_2\_100} \ac{NN} for $\varepsilon$ equivalence, 1089 benchmark queries across three \acp{NN} (\textsc{MNIST\_2\_100}, \textsc{MNIST\_2\_512},\textsc{MNIST\_3\_1024}) for NeuroDiff $\varepsilon$ equivalence, and 1199 benchmark queries across two \acp{NN} (\textsc{MNIST\_2\_100}, \textsc{MNIST\_2\_512}) for Top-1 equivalence.
In the latter case, we omitted the larger \ac{NN} due to excessive computational load, see benchmark family ``MNIST (ours)'' for an evaluation on \textsc{MNIST\_3\_1024} architectures.

\paragraph{MNIST (Ours).}
Reusing the same input space regions as for MNIST (VeriPrune), we pruned resp. 21\% and 51\% of the weights of the original MNIST \acp{NN} by Paulsen et al.~\cite{paulsen_neurodiff_2020} and performed further training to regain accuracy.
We add this benchmark to evaluate whether Differential Verification is still applicable when weight configurations have been modified by gradient descent during additional training.
This resulted in \acp{NN} with architectures and test accuracies as reported in \Cref{tab:mnist_retrained_accuracies}.
We used these to prove $\varepsilon$ equivalence ($\varepsilon$=1.0) and Top-1 equivalence.
This yields 3200 queries for $\varepsilon$ and 2400 queries for Top-1 equivalence across the 4 pruned \acp{NN} from \Cref{tab:mnist_retrained_accuracies}.

\begin{table}[t]
    \centering
    \caption{Architectures \& Test Accuracies for pruned and further trained MNIST \acp{NN}}
    \begin{tabularx}{\textwidth}{X|l|X|X}
        \multirow{2}{*}{\textbf{Architecture}} & 
        \multirow{2}{*}{\textbf{Accuracy} (\%)} & 
        \multicolumn{2}{c}{\textbf{Accuracy after Pruning } (\%)}\\\cline{3-4}
        &&\multicolumn{1}{c}{\textbf{21\%}} & \multicolumn{1}{|c}{\textbf{51\%}}\\\hline\hline
        \textsc{MNIST\_2\_100} & 
        96.54 &
        96.54 &
        96.45 \\
        \textsc{MNIST\_2\_100} &
        98.28 &
        98.21 &
        ---\\
        \textsc{MNIST\_3\_1024} &
        93.85 &
        93.78 &
        ---
    \end{tabularx}
    \label{tab:mnist_retrained_accuracies}
\end{table}

\paragraph{LHC \acp{NN}.}
As a first benchmark instance for confidence-based equivalence verification, we consider a set of \acp{NN} that we trained on a dataset for particle classification from CERN's Large Hadron Collider (LHC)~\cite{duarte2018fast}.
As explained by Duarte \emph{et al.}~\cite{duarte2018fast}, such \acp{NN} could eventually be used in the Level 1 trigger systems at the LHC.
The Level 1 trigger is a system implemented via custom hardware (e.g. FPGAs) to reduce the online data stream of the running LHC experiment to a rate processable via commercial hardware (e.g. CPUs; Level 2 trigger) under strict real-time constraints.
To this end, the system classifies observed events, e.g. based on the observed particle jet.
Particle Jet Classification is a task that could be performed by an \ac{NN}.
To apply \acp{NN} in this setting, it is paramount to generate reliable, very small \acp{NN} capable of satisfying these real-time constraints.
Here, we consider the setting where a previously trained \ac{NN} (classifying observations into one of 5 classes of particle jets) is pruned and trained further to recover accuracy.
To this end, we trained four NNs (\textsc{2\_20},\textsc{2\_40},\textsc{2\_80},\textsc{4\_20}).
We then pruned $p\in\left\{10,20,30,40,50\right\}$ percent of the \reluSym{} nodes and subsequently trained the NN for an additional $e \in \left\{0.1,1\right\}$ epochs.
This results in 40 \acp{NN} for which we aim to verify confidence-based equivalence.
Accuracies of the trained \acp{NN} can be found in \Cref{tab:lhc_accuracies}.
For the original \ac{NN} the accuracies are comparable to Duarte \emph{et al.}~\cite{duarte2018fast}.
Since the \ac{NN}'s input space was normalized, we attempted verification for standard deviatons $\sigma \in \left\{0.1,0.5,1,2,3\right\}$.
We attempted confidence-based verification for confidence thresholds $\delta =0.5$ and $\delta=1-10^{-i}$ for $i\in\left[1,7\right]$ (i.e. $\delta=0.9$, $\delta=0.99$ etc.).
Overall, this yields 1600 benchmark queries.

\paragraph{Rice \acp{NN}.}
For confidence-based equivalence verification, we also consider a set of \acp{NN} that were trained on a dataset for rice classification~\cite{cinar2019classification}.
To this end, 7 engineered features (e.g. pixel count, Perimeter, etc.) are given to an \ac{NN} and the \ac{NN} has the task of predicting which type of rice is pictured (Cammeo and Osmancik).
Note that this task only involves two classes and our $\mathrm{softmax}$ approximation is thus exact.
We consider NNs of the architectures 2\_20,2\_40, 2\_80 and 4\_40.
The \ac{NN} are then either only pruned by 10 to 50 percent (prune only) or pruned by by 10 to 90 percent and retrained for 5 epochs (prune \& retrain).
We omit a table with all accuracies that can be found in the artifact.
In general, the original \acp{NN} had an accuracy of approx. 93\% and even when pruning 90\% of \reluSym{} nodes the resulting \ac{NN} recovered to 92-93\% accuracy after 5 epochs of training.
For pruning, we only saw a drop in accuracy after pruning 50\% of nodes.
We also evaluate w.r.t. a normalized input space and attempted verification for standard deviations $\sigma \in \left\{0.5,1.0,2.0,3.0\right\}$.
Overall, this resulted in 1152 benchmark queries for prune \& retrain and 640 benchmark queries for prune only.

\begin{table}[t]
    \centering
    \caption{Architectures \& Test Accuracies for pruned and further trained LHC \acp{NN} (pruned by 10/20/30/40/50 percent; trained for an additional 0.1 / 1 epoch after pruning).}
    \begin{tabularx}{\textwidth}{l | l | X | X | X | X | X | X | X | X | X | X}
        \multirow{3}{*}{\textbf{Architecture}} &
        \multirow{3}{1.5cm}{\textbf{Accuracy}\newline (\%)} &
        \multicolumn{10}{|c}{\textbf{Accuracy after Prune + Training} (\%)}\\\cline{3-12}
        &&
        \multicolumn{2}{|c}{10\%} &
        \multicolumn{2}{|c}{20\%} &
        \multicolumn{2}{|c}{30\%} &
        \multicolumn{2}{|c}{40\%} &
        \multicolumn{2}{|c}{50\%}\\\cline{3-12}
        &&
        0.1 & 1.0 &
        0.1 & 1.0 &
        0.1 & 1.0 &
        0.1 & 1.0 &
        0.1 & 1.0 \\\hline\hline
        \textsc{2\_20} & 74.78 &
        74.62 & 74.69  &
        74.21 & 56.37 &
        46.42 & 46.30 &
        36.54 & 45.96 &
        36.76 & 36.85  \\
        \textsc{2\_40} & 74.96 &
        74.90 & 74.95  &
        74.88 & 74.97  &
        74.77 & 74.68  &
        74.67 & 74.57  &
        55.59 & 73.87   \\
        \textsc{2\_80} & 75.10 &
        74.94 & 74.91  &
        75.06 & 74.96  &
        74.97 & 74.74  &
        56.35 & 74.40  &
        56.35 & 46.48   \\
        \textsc{4\_20} & 74.64 &
        74.47 & 74.52  &
        74.32 & 74.60  &
        34.99 & 33.97  &
        33.18 & 35.01  &
        35.83 & 35.48
    \end{tabularx}
    \label{tab:lhc_accuracies}
\end{table}

\subsection{Ablation Studies}
\label{apx:subsec:eval:abelation}
We will now address the elephant in the room:
As noted in \Cref{sec:diff_zono}, ``reasoning about the difference of two \reluSym{}s is not particularly intuitive''.
This raises the question of whether this unintuitive reasoning is even necessary, i.e. can we verify more properties with than without Differential Zonotopes?
To this end, we evaluate the performance of our tool when the computation of Differential Zonotopes is activated and deactivated.
In the deactivated case, we use the naive Differential Zonotope (see \Cref{sec:diff_zono}).
An overview of this comparison can be found in \Cref{tab:abelation_diff_zono}.

For $\varepsilon$ equivalence, VeryDiff with Differential Zonotopes can solve more instances than without across all considered benchmark families.
While on commonly solved instances, Differential Zonotopes slow down the median runtime for ACAS (see speedups), this behavior becomes irrelevant when we scale to larger input dimensions, i.e. to MNIST:
Here, we either observe a speedup >1 in the median case or observe substantial improvements (>1400\%) in the number of verified equivalence properties.
This contrast becomes particularly stark when analyzing the number of benchmark instances for which the tools were able to certify equivalence by radius:
As can be seen in \Cref{fig:abelation_by_radius} (\Cref{sec:evaluation}), the naive approach stops working entirely for larger radii.
Thus, we posit that Differential Verification indeed improves the performance of VeryDiff for the verification of $\varepsilon$ equivalence.
This observation corresponds to earlier results by Paulsen \emph{et al.}~\cite{paulsen_reludiff_2020,paulsen_neurodiff_2020}.
We note that even for our new MNIST benchmark family, Differential Verification significantly outperforms the naive approach.
This demonstrates that our technique can be applied even after gradient descent updates.

For Top-1 equivalence, we observe that activating Differential Zonotopes not only makes the approach slower on commonly solved instances but moreover \emph{decreases} the number of instances solved.
While this result may initially come as a surprise, we provide a theoretical intuition for this behavior in \Cref{sec:classification} which explains why differential reasoning rarely helps in verifying Top-1 equivalence.
Given the additional information from differential reasoning is not helpful, the slowdown is a result of the overhead incurred during the computation of the Differential Zonotope:
Without this additional overhead, VeryDiff can explore a more fine-grained partitioning of the input space which may yield an equivalence proof or a counterexample.

The effect of Differential Verification for $\delta$-Top-1 equivalence will be studied in \Cref{apx:subsec:eval:confidence}.

\begin{table}[t]
    \centering
\caption{Comparison of two variants of VeryDiff: Only propagation of Zonotopes in individual NNs (Naive) and additional propagation of a Differential Zonotope (Diff. Zono). We report numbers on verified instances (Equiv.) and found counterexamples (Counterex.) resp. with relative decrease/increase.
We also report relative speedups on commonly solved instances for Differential Zonotopes over the naive approach.}
\label{tab:abelation_diff_zono}
\begin{tabular}{l|p{2cm}|l|rr|rr|rr}
\multicolumn{2}{c|}{\multirow{2}{*}{\textbf{Benchmark}}} &
\multirow{2}{*}{\textbf{Variant}} &
\multicolumn{2}{c|}{\multirow{2}{*}{\textbf{Equiv.}}}  &
\multicolumn{2}{c|}{\multirow{2}{*}{\textbf{Counterex.}}} &
\multicolumn{2}{c}{\textbf{Speedup}}\\ \cline{8-9}
\multicolumn{2}{c|}{}&&&&&& Median & Max\\\hline\hline
\multirow{6}{*}{\rotatebox[origin=c]{90}{$\epsilon$ eq.}} & \multirow{2}{2cm}{ACAS} & Diff. Zono &        \textbf{150} & \good{(+0.6\%)} &            \textbf{153}  &       \good{(+0.6\%)} &       --- &          --- \\
        &                   & Naive &        149 & &           152 &        &  \bad{0.8} &    \good{33.6} \\\cline{2-9}
        & \multirow{2}{2cm}{MNIST (Ours)} & Diff. Zono &        \textbf{484} & \good{(+356.6\%)} &            \textbf{310} &    \good{(+101.3\%)} &       --- &          --- \\
        &                   & Naive &        106 & &           154 &     &  \good{5.1} &  \good{3326.4} \\\cline{2-9}
        & \multirow{2}{2cm}{MNIST (VeriPrune)} & Diff. Zono &        \textbf{352} & \good{(+1430.4\%)}&           \textbf{1761}  &     \good{(+2.4\%)} &       --- &          --- \\
        &                   & Naive &         23 &  &         1720 &      &  \bad{0.9} &  \good{3008.3} \\\hline\hline
\multirow{4}{*}{\rotatebox[origin=c]{90}{Top-1 eq.}} & \multirow{2}{2cm}{MNIST (Ours)} & Diff. Zono &       1401 & \bad{(-2.2\%)}&            131 &     \bad{(-13.8\%)} &       --- &          --- \\
        &                   & Naive &       \textbf{1433} & &            \textbf{152} &      &  \bad{0.8} &     \good{4.2} \\\cline{2-9}
        & \multirow{2}{2cm}{MNIST (VeriPrune)} & Diff. Zono &        885 & \bad{(-1.1\%)} &             53 &     \bad{(-28.4\%)} &       --- &          --- \\
        &                   & Naive &        \textbf{895} & &             \textbf{74} &      &  \bad{0.8} &    \good{17.4} \\ 

\end{tabular}
\end{table}

\subsection{Verifying $\mathbf{\varepsilon}$ equivalence: Comparison to State-of-the-Art}
\label{apx:subsec:eval:epsilon}
\begin{figure}
    \centering
    \includegraphics[width=0.5\linewidth]{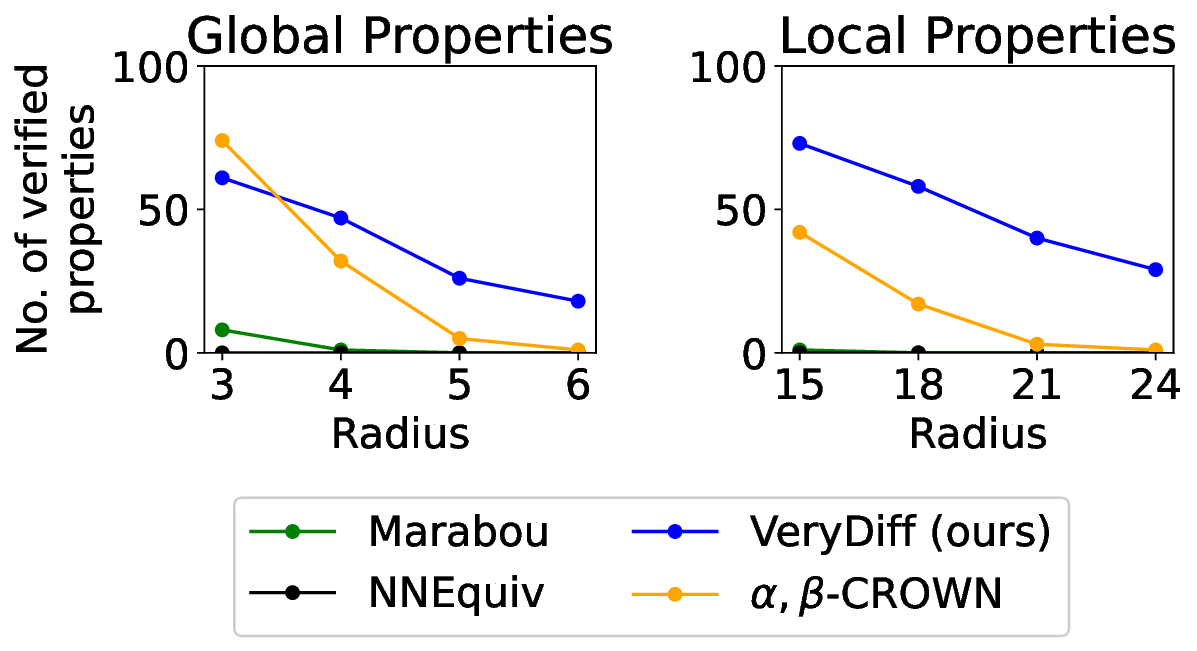}
    \caption{Performance comparison for $\varepsilon$ equivalence certification on MNIST (VeriPrune) benchmark: For larger radii competing tools almost always yield timeouts.}
    \label{fig:epsilon_mnist_radius_comparison.eps}
\end{figure}
The ablation study preceding this subsection has already shown some results on $\varepsilon$ equivalence, however, there remains the question of how these results compare to the wider literature on \ac{NN} verification.
To this end, we compared the performance of VeryDiff with the performance of other (equivalence) verification tools for \acp{NN}.
For standard $\varepsilon$ equivalence and NeuroDiff $\varepsilon$ equivalence (see note on the difference in \Cref{apx:subsec:eval:setup}) results can be found in \Cref{tab:epsilon_equivalence_sota} (see \Cref{sec:evaluation}).
Note that we always compare w.r.t. the verification of the same equivalence property.
We report relative improvements w.r.t. the best other tool and report speedups achieved by VeryDiff on commonly solved instances for all tools.
As indicated in the table, VeryDiff universally beats the State of the Art with the exception of counterexample generation for standard $\varepsilon$ equivalence on the MNIST (VeriPrune) family.
Here, $\alpha,\beta$-CROWN can outpeform VeryDiff.
We suspect that this is due to the counterexample search techniques employed in the early phases of $\alpha,\beta$-CROWN.
Implementing such attack techniques in our tool remains future work.
For the fastest CPU-only verifier at VNNComp 2024, Marabou, we observe significant speedups on commonly solved instances (>140 in the median).
Moreover, we observe that VeryDiff verifies many more equivalence properties than Marabou (+3420\%) for the MNIST (VeriPrune) benchmark family.
For Marabou we observed 13 instances where the tool returned a spurious counterexample while all other tools certified equivalence\footnote{We have reported these instances to the tool authors as they point to a potential bug.}. %
We also ran MILPEquiv on the first 400 benchmark queries of the MNIST (VeriPrune) family:
MILPEquiv verified equivalence 7 times and found 28 counterexamples.
Due to excessive timeouts, we omit a full comparison.
As can be seen in \Cref{fig:epsilon_mnist_radius_comparison.eps}, our approach in particular shines for larger radii: Here, the other tools almost always run into a timeout while the differential bounds provided by $\zonotope{}^\Delta$. suffice to show the given equivalence property.

\paragraph{Choice of Equivalence Property.}
We also cross-referenced our verification results for $\varepsilon$ and Top-1 equivalence.
For the MNIST (VeriPrune) benchmark family, we observe benchmark queries where the results for $\varepsilon$ and Top-1 equivalence mismatch.
This underlines the importance of choosing the correct verification property for the task at hand.
Hence, the subsequent subsection will focus on exploiting differential verification for confidence-based equivalence verification in classification \acp{NN}.

\subsection{Confidence-Based Equivalence Verification}
\label{apx:subsec:eval:confidence}
To analyze the effectiveness of our proposed approach for confidence-based equivalence verification, we derived two new benchmark families (see \Cref{apx:subsec:eval:setup}).
The \acp{NN} of these benchmark families are significantly smaller than the \acp{NN} studied in the prior sections.
As also shown by Athavale \emph{et al.}~\cite{DBLP:conf/cav/AthavaleBCMNW24}, \emph{global} properties can currently only be analyzed w.r.t. smaller-scale \acp{NN} (e.g. prior work only analyzed \acp{NN} with up to 50 \reluSym{} nodes~\cite{DBLP:conf/cav/AthavaleBCMNW24}).
For the tasks at hand, we observe no significant increase in accuracy when scaling from our smallest to our largest \acp{NN} indicating that the chosen \acp{NN} are sufficient for the problem at hand.
We report results w.r.t. a 10 minute timeout.
As shown in \Cref{apx:subsec:eval:epsilon}, falsification of equivalence can be achieved more efficiently using $\alpha,\beta$-CROWN's attack techniques.
Hence, this section focuses on \emph{certifying} equivalence.

\subsubsection{LHC \acp{NN}.}
\begin{wrapfigure}[17]{r}{0.4\textwidth}
    \centering
    \includegraphics[width=\linewidth]{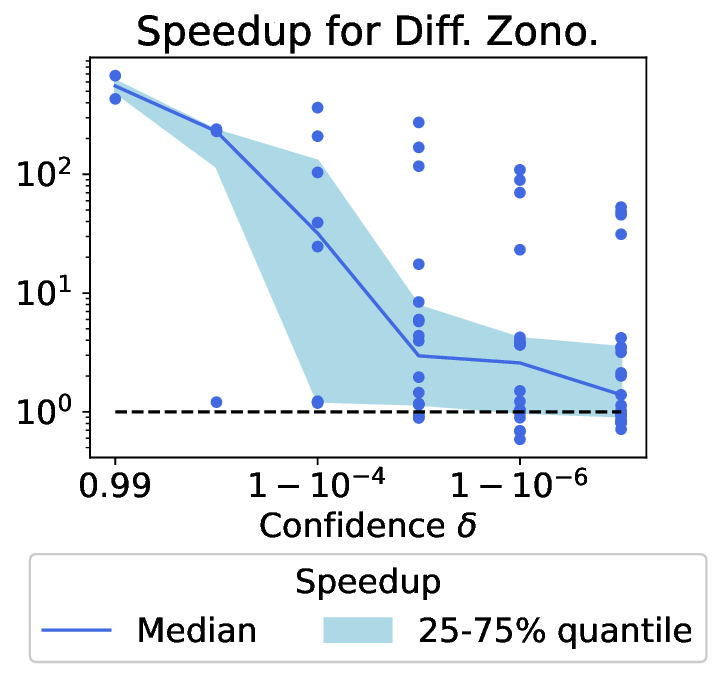}
    \caption{Speedups for certification with Differential Zonotopes over naive technique on LHC \acp{NN}.}
    \label{fig:speedup_lhc}
\end{wrapfigure}
For the LHC benchmark family, we observe 1427 queries with concrete counterexamples.
As the \acp{NN} have 5 output nodes, there may be cases where our $\mathrm{softmax}$ approximation leads to spurious counterexamples, i.e. the output of $f_1$ does not actually pass the confidence threshold $\delta$.
Of the 173 queries without concrete counterexamples, there are 10 cases in which we only find spurious counterexamples.
For 52 queries further search yielded concrete counterexamples after the initial detection of a spurious counterexample (these are included in the 1427 queries above).
We want to underscore that VeryDiff's softmax approximation immediately yielded concrete counterexamples for 1375 out of 1427 queries and found concrete counterexamples for 52 out of 62 cases through continued search.

For the remaining analysis, we focus on instances for which we were able to prove equivalence or observed timeouts.
First, we compared the performance of Differential Zonotopes with the naive VeryDiff implementation.
Here, we observe that both techniques solve a common set of 70 queries while Differential Zonotopes solve an additional 7, and the naive approach solves an additional 3 benchmarks.
For commonly solved instances, we plot the speedups achieved by Differential Zonotopes in \Cref{fig:speedup_lhc}:
As we increase the required confidence threshold towards 1, the speedups achieved through the use of Differential Verification diminish.
This observation is further supported by the benchmark queries not commonly solved:
The queries only solved via Differential Zonotopes have confidence thresholds $\delta=0.9$ (not achieved by the naive technique) or $\delta\in\left\{0.99,0.999,0.9999\right\}$ while the benchmarks only solved by the naive technique have confidence threshold $\delta=1-10^{-7}$.
For example, for a \textsc{2\_20} \ac{NN} trained for an additional epoch after 10\% pruning, Differential Zonotopes can prove $0.9$-Top-1 equivalence in 300 seconds while the naive technique times out after 10 minutes.
Within those 10 minutes, the naive approach had performed 4x as many input space splits while only certifying equivalence on 19\% of the input space.
As we increase $\delta$, we provide guarantees for ever smaller parts of the state space which seems to be easier even without relying on differential reasoning.
Conversely, for $\delta \to 0.5$, the verification procedure converges to regular Top-1 equivalence for which we have shown that Differential Zonotopes do not help (see \Cref{apx:subsec:eval:abelation}).
We conjecture that there exists a sweet spot $\delta^* \in \left(0.5,0.9\right]$ for which Differential Zonotopes provide maximal speedup.
Fortunately, this sweet spot would coincide with the kind of property that we would like to prove ($\delta=0.5$ seems unlikely to hold while $\delta>0.999$ seems overconstrained).
We also compared the capabilities of VeryDiff (with Differential Zonotopes) to $\alpha,\beta$-CROWN (see \Cref{tab:epsilon_equivalence_sota}):
Our technique solves 327\% more queries with a median speedup of 324.5 on commonly solved benchmark queries.

\subsubsection{Rice \acp{NN}.}
\begin{wrapfigure}[15]{r}{0.4\textwidth}
    \centering
    \vspace*{-2.2em}
    \includegraphics[width=\linewidth]{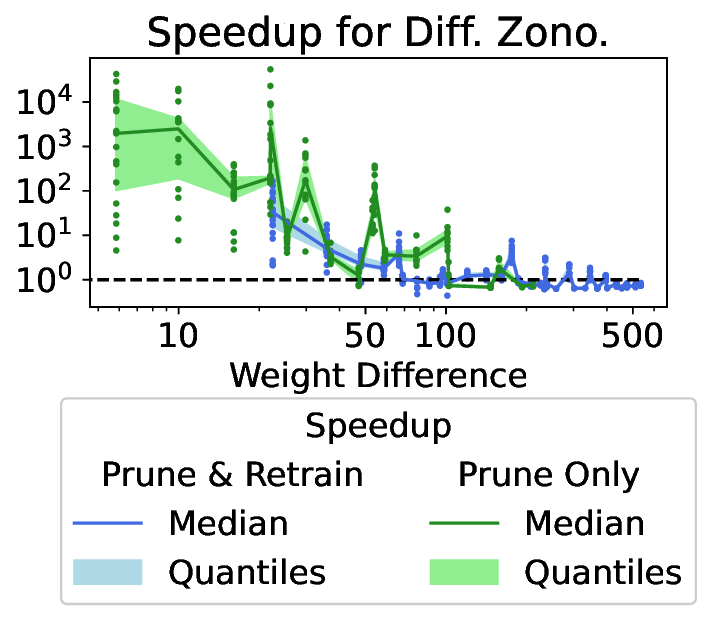}
    \caption{Speedups achieved by Differential Verification on Rice \acp{NN} by weight difference.}
    \label{fig:speedup_rice}
\end{wrapfigure}
We performed additional experiments on a benchmark family of \acp{NN} for rice classification based on 7 engineered features.
Since these \acp{NN} only have two outputs our softmax handling is exact.
For the \acp{NN} that were only pruned, both approaches certified equivalence for 199 common queries (Differential Zonotopes: 37 additional; Naive: 1 additional).
For the \acp{NN} that were pruned and trained for an additional 5 epochs, both approaches certified equivalence for 339 queries (Differential Zonotopes: 8 additional).
We again focus our analysis on the commonly certified instances.
Given the vast range of pruning factors (10 to 90 percent) and the pruned only and pruned \& retrained \acp{NN}, we can analyze the effect of growing weight differences on the efficiency of Differential Verification.
To this end, \Cref{fig:speedup_rice} shows the speedups achieved via Differential Zonotopes over our naive technique w.r.t. increasing (absolute) differences in the weight matrices between the two \acp{NN} (note both axes are in log scale).
Differential Zonotopes perform best for weight differences below 100.
In particular, for weight differences below 20 the median speedups reach beyond 100 and sometimes even beyond 1000.
Absolute weight differences accumulate through increased pruning rates, further training after pruning, and pruning across more layers.
Hence, these influence factors determine the efficiency gains of Differential Verification over naive analyses.

\subsubsection{Limitations.}
Across both benchmark families (LHC \acp{NN} and Rice \acp{NN}) we observed diminishing returns for the usage of Differential Verification when moving from \acp{NN} with 2 hidden \reluSym{} layers to \acp{NN} with 4 hidden \reluSym{} layers.
On the one hand, this observation can be explained by the increasing weight differences between the two \acp{NN} (see also analysis for Rice \acp{NN}).
In addition, we conjecture that this may be an artifact of the \emph{wrapping effect}, i.e. the accumulation of overapproximation errors across layered approximations.
While this limits the applicability at the current time, we can also use this limitation as guidance by scaling \acp{NN} in width rather than depth when we plan to use equivalence verification.

\end{document}